\documentclass{article}

\PassOptionsToPackage{round}{natbib}
\usepackage[main,final]{neurips_2026}

\usepackage[utf8]{inputenc} 
\usepackage[T1]{fontenc}    

\usepackage{times}
\usepackage{mathtools}   %
\usepackage{amssymb}
\usepackage{amsthm}
\usepackage{bm}
\usepackage{dsfont}
\usepackage{mathrsfs}

\usepackage{graphicx}
\usepackage[svgnames]{xcolor}
\usepackage{booktabs}
\usepackage{multirow}
\usepackage{colortbl}
\usepackage{wrapfig}
\usepackage{threeparttable}
\usepackage{longtable} 
\usepackage{etoc}
\usepackage{makecell}

\usepackage{tabularx}
\usepackage{array}

\usepackage{framed}
\usepackage{ascmac}

\usepackage{nicefrac}
\usepackage{physics} 
\usepackage{algorithm}
\usepackage[linesnumbered,ruled,vlined,algo2e]{algorithm2e}
\SetAlgoNlRelativeSize{-1}

\makeatletter
\@ifclassloaded{beamer}{%
}{%
  \usepackage{enumitem}
  \usepackage[hypertexnames=false]{hyperref}
}
\makeatother
\usepackage[capitalize]{cleveref}

\crefname{equation}{}{} 
\crefname{function}{Function}{Functions}
\Crefname{equation}{Eq.}{Eqs.}
\crefname{figure}{Figure}{Figures}
\usepackage{autonum}
\hypersetup{
  colorlinks,
  linkcolor={violet},
  citecolor={blue},
  urlcolor={brown},
}

\usepackage{comment}

\usepackage{soul} %

\theoremstyle{plain}
\newtheorem{thm}{Theorem}[section]
\newtheorem{prop}[thm]{Proposition}
\newtheorem{lem}[thm]{Lemma}

\theoremstyle{definition}
\newtheorem{dfn}[thm]{Definition}

\theoremstyle{remark}

\newcommand{\fixcrefthm}[1]{%
  \AddToHook{env/#1/begin}{\crefalias{thm}{#1}}%
}

\fixcrefthm{thm}
\fixcrefthm{prop}
\fixcrefthm{lem}
\fixcrefthm{cor}
\fixcrefthm{dfn}
\fixcrefthm{ass}
\fixcrefthm{rem}
\fixcrefthm{exa}
\fixcrefthm{prty}
 
\crefname{equation}{}{} 
\crefname{function}{Function}{Functions}
\Crefname{equation}{Eq.}{Eqs.}
\crefname{thm}{Theorem}{Theorems}
\crefname{prop}{Proposition}{Propositions}
\crefname{lem}{Lemma}{Lemmas}
\crefname{cor}{Corollary}{Corollaries}
\crefname{dfn}{Definition}{Definitions}
\crefname{ass}{Assumption}{Assumptions}
\crefname{rem}{Remark}{Remarks}
\crefname{exa}{Example}{Examples}
\crefname{prty}{Property}{Properties}

\newcommand{\calA}{\mathcal{A}}

\newcommand{\calC}{\mathcal{C}}
\newcommand{\calD}{\mathcal{D}}
\newcommand{\calE}{\mathcal{E}}
\newcommand{\calF}{\mathcal{F}}
\newcommand{\calG}{\mathcal{G}}

\newcommand{\calK}{\mathcal{K}}

\newcommand{\calM}{\mathcal{M}}

\newcommand{\calO}{\mathcal{O}}
\newcommand{\calP}{\mathcal{P}}
\newcommand{\calQ}{\mathcal{Q}}

\newcommand{\calS}{\mathcal{S}}
\newcommand{\calT}{\mathcal{T}}

\newcommand{\calV}{\mathcal{V}}

\renewcommand{\hat}{\widehat}
\renewcommand{\tilde}{\widetilde}
\renewcommand{\epsilon}{\varepsilon}

\newcommand{\hatP}{\hat{P}}
\newcommand{\hatQ}{\hat{Q}}

\newcommand{\tilP}{\tilde{P}}

\newcommand{\tilO}{\tilde{\calO}}

\newcommand{\tils}{\tilde{s}}

\newcommand{\overP}{\overline{P}}
\newcommand{\underP}{\underline{P}}
\newcommand{\overq}{\overline{q}}
\newcommand{\underq}{\underline{q}}

\newcommand{\Var}{\mathsf{Var}}

\renewcommand{\ln}{\log}

\def\E{\mathbb{E}}

\def\I{\mathbb{I}}
\def\R{\mathbb{R}}

\def\N{\mathbb{N}}

\newcommand{\ind}{\mathds{1}}

\let\abs\relax
\DeclareMathOperator*{\argmax}{arg\,max}
\DeclareMathOperator*{\argmin}{arg\,min}
\DeclarePairedDelimiter{\brk}{[}{]}
\DeclarePairedDelimiter{\set}{\{}{\}}
\DeclarePairedDelimiter{\prn}{(}{)}
\DeclarePairedDelimiter{\abs}{\lvert}{\rvert}
\DeclarePairedDelimiter{\nrm}{\|}{\|}
\DeclarePairedDelimiter{\inpr}{\langle}{\rangle}  %

\newcommand{\relmiddle}[1]{\mathrel{}\middle#1\mathrel{}}

\newcommand{\one}{\mathbf{1}}
\newcommand{\Reg}{\text{\rm Reg}}

\newcommand{\regterm}{\textnormal{\textbf{reg}}^{\pi}}
\newcommand{\biasterm}{\textnormal{\textbf{bias}}^{\pi}}
\newcommand{\errorterm}{\textnormal{\textbf{error}}^{\pi}}

\newcommand{\term}{\mathrm{\textbf{term}}}

\newcommand{\polylog}{\mathrm{polylog}}

\newcommand{\pio}{\mathring{\pi}}
\newcommand{\pist}{{\pi^\star}}

\newcommand{\Qtrans}{Q_{\mathrm{trans}}}

\newcommand{\cmark}{\ensuremath{\checkmark}}

\allowdisplaybreaks[1]
\usepackage{mdframed}

\usepackage{etoolbox}

\usepackage[normalem]{ulem}
\makeatletter
\newsavebox{\AuditStrikeBox}
\newcommand{\AuditMathStrike}[2]{%
  \begingroup\color{blue}%
  \sbox{\AuditStrikeBox}{$\m@th#1#2$}%
  \rlap{\raisebox{.5ex}{\rule{\wd\AuditStrikeBox}{.45pt}}}%
  \usebox{\AuditStrikeBox}\endgroup}
\DeclareRobustCommand{\DIFdel}[1]{%
  \begingroup\color{blue}\def\@linkcolor{blue}\def\@citecolor{blue}%
  \ifmmode\mathpalette\AuditMathStrike{#1}\else\sout{#1}\fi\endgroup}
\DeclareRobustCommand{\DIFadd}[1]{%
  {\color{red}\def\@linkcolor{red}\def\@citecolor{red}#1}}
\providecommand{\DIFaddbegin}{\color{red}\gdef\@linkcolor{red}\gdef\@citecolor{red}}
\providecommand{\DIFaddend}{\color{black}\gdef\@linkcolor{violet}\gdef\@citecolor{blue}}
\providecommand{\DIFdelbegin}{\color{blue}\gdef\@linkcolor{blue}\gdef\@citecolor{blue}}
\providecommand{\DIFdelend}{\color{black}\gdef\@linkcolor{violet}\gdef\@citecolor{blue}}

\makeatother

\let\AuditOriginalBibitem\bibitem
\RenewDocumentCommand{\bibitem}{o m}{%
  \ifstrequal{#2}{chen2026high}{\color{red}}{\color{black}}%
  \IfNoValueTF{#1}{\AuditOriginalBibitem{#2}}{\AuditOriginalBibitem[#1]{#2}}}

\AtBeginDocument{
\newlabel{DIFold-tab:comparison}{{1}{2}{}{}{}}
\newlabel{DIFold-tab:comparison@cref}{{[table][1][]1}{[1][2][]2}}
\newlabel{DIFold-sec:pre}{{2}{3}{}{}{}}
\newlabel{DIFold-sec:pre@cref}{{[section][2][]2}{[1][3][]3}}
\newlabel{DIFold-def:calP_t}{{1}{4}{}{}{}}
\newlabel{DIFold-def:calP_t@cref}{{[equation][1][]1}{[1][4][]4}}
\newlabel{DIFold-subsec:regimes}{{2.2}{4}{}{}{}}
\newlabel{DIFold-subsec:regimes@cref}{{[subsection][2][2]2.2}{[1][4][]4}}
\newlabel{DIFold-sec:data_dependent_measures}{{2.3}{5}{}{}{}}
\newlabel{DIFold-sec:data_dependent_measures@cref}{{[subsection][3][2]2.3}{[1][5][]5}}
\newlabel{DIFold-def:qtrans_loss}{{2}{5}{}{}{}}
\newlabel{DIFold-def:qtrans_loss@cref}{{[equation][2][]2}{[1][5][]5}}
\newlabel{DIFold-sec:algorithm}{{3}{5}{}{}{}}
\newlabel{DIFold-sec:algorithm@cref}{{[section][3][]3}{[1][5][]5}}
\newlabel{DIFold-eq:opt_pi_OFTRL}{{3}{6}{}{}{}}
\newlabel{DIFold-eq:opt_pi_OFTRL@cref}{{[equation][3][]3}{[1][5][]6}}
\newlabel{DIFold-def:predictor_sequence}{{4}{6}{}{}{}}
\newlabel{DIFold-def:predictor_sequence@cref}{{[equation][4][]4}{[1][6][]6}}
\newlabel{DIFold-def:dilated_bonus}{{5}{6}{}{}{}}
\newlabel{DIFold-def:dilated_bonus@cref}{{[equation][5][]5}{[1][6][]6}}
\newlabel{DIFold-lem:dilated_bonus_main}{{3.1}{6}{}{}{}}
\newlabel{DIFold-lem:dilated_bonus_main@cref}{{[lem][1][3]3.1}{[1][6][]6}}
\newlabel{DIFold-eq:key_lemma_ineq_main}{{6}{6}{}{}{}}
\newlabel{DIFold-eq:key_lemma_ineq_main@cref}{{[equation][6][]6}{[1][6][]6}}
\newlabel{DIFold-alg:OFTRL_unknown}{{1}{7}{}{}{}}
\newlabel{DIFold-alg:OFTRL_unknown@cref}{{[algorithm][1][]1}{[1][6][]7}}
\newlabel{DIFold-line:optimize_policy}{{4}{7}{}{}{}}
\newlabel{DIFold-line:optimize_policy@cref}{{[line][4][]4}{[1][6][]7}}
\newlabel{DIFold-line:virtual_episode}{{5}{7}{}{}{}}
\newlabel{DIFold-line:virtual_episode@cref}{{[line][5][]5}{[1][6][]7}}
\newlabel{DIFold-eq:Q-estimate_full}{{7}{8}{}{}{}}
\newlabel{DIFold-eq:Q-estimate_full@cref}{{[equation][7][]7}{[1][6][]8}}
\newlabel{DIFold-eq:learning_rate_full}{{8}{8}{}{}{}}
\newlabel{DIFold-eq:learning_rate_full@cref}{{[equation][8][]8}{[1][6][]8}}
\newlabel{DIFold-eq:bonus_term_full}{{9}{8}{}{}{}}
\newlabel{DIFold-eq:bonus_term_full@cref}{{[equation][9][]9}{[1][6][]8}}
\newlabel{DIFold-eq:zeta_full}{{10}{8}{}{}{}}
\newlabel{DIFold-eq:zeta_full@cref}{{[equation][10][]10}{[1][6][]8}}
\newlabel{DIFold-eq:Q-estimate_bandit}{{11}{8}{}{}{}}
\newlabel{DIFold-eq:Q-estimate_bandit@cref}{{[equation][11][]11}{[1][6][]8}}
\newlabel{DIFold-eq:def_new_c}{{12}{8}{}{}{}}
\newlabel{DIFold-eq:def_new_c@cref}{{[equation][12][]12}{[1][6][]8}}
\newlabel{DIFold-eq:def_C}{{13}{8}{}{}{}}
\newlabel{DIFold-eq:def_C@cref}{{[equation][13][]13}{[1][6][]8}}
\newlabel{DIFold-eq:eta_update_policy}{{14}{8}{}{}{}}
\newlabel{DIFold-eq:eta_update_policy@cref}{{[equation][14][]14}{[1][6][]8}}
\newlabel{DIFold-eq:bonus_term_bandit}{{15}{8}{}{}{}}
\newlabel{DIFold-eq:bonus_term_bandit@cref}{{[equation][15][]15}{[1][6][]8}}
\newlabel{DIFold-eq:zeta_bandit}{{16}{8}{}{}{}}
\newlabel{DIFold-eq:zeta_bandit@cref}{{[equation][16][]16}{[1][6][]8}}
\newlabel{DIFold-fig:def_hatQ_eta_b}{{1}{8}{}{}{}}
\newlabel{DIFold-fig:def_hatQ_eta_b@cref}{{[figure][1][]1}{[1][6][]8}}
\newlabel{DIFold-sec:main_result}{{4}{8}{}{}{}}
\newlabel{DIFold-sec:main_result@cref}{{[section][4][]4}{[1][8][]8}}
\newlabel{DIFold-thm:main_full}{{4.1}{9}{}{}{}}
\newlabel{DIFold-thm:main_full@cref}{{[thm][1][4]4.1}{[1][9][]9}}
\newlabel{DIFold-thm:main_bandit}{{4.2}{9}{}{}{}}
\newlabel{DIFold-thm:main_bandit@cref}{{[thm][2][4]4.2}{[1][9][]9}}
\newlabel{DIFold-prop:qtrans}{{4.3}{9}{}{}{}}
\newlabel{DIFold-prop:qtrans@cref}{{[prop][3][4]4.3}{[1][9][]9}}
\newlabel{DIFold-sec:proof}{{5}{9}{}{}{}}
\newlabel{DIFold-sec:proof@cref}{{[section][5][]5}{[1][9][]9}}
\newlabel{DIFold-sec:conclusion}{{6}{10}{}{}{}}
\newlabel{DIFold-sec:conclusion@cref}{{[section][6][]6}{[1][10][]10}}
\newlabel{DIFold-tab:notation-summary-main}{{2}{15}{}{}{}}
\newlabel{DIFold-tab:notation-summary-main@cref}{{[table][2][2147483647]2}{[1][14][]15}}
\newlabel{DIFold-app:add_relate}{{B.1}{16}{}{}{}}
\newlabel{DIFold-app:add_relate@cref}{{[subappendix][1][2147483647,2]B.1}{[1][16][]16}}
\newlabel{DIFold-app:qtrans-discussion}{{B.2}{17}{}{}{}}
\newlabel{DIFold-app:qtrans-discussion@cref}{{[subappendix][2][2147483647,2]B.2}{[1][17][]17}}
\newlabel{DIFold-eq:qtrans_example}{{17}{17}{}{}{}}
\newlabel{DIFold-eq:qtrans_example@cref}{{[equation][17][2147483647]17}{[1][17][]17}}
\newlabel{DIFold-eq:qtrans_first_order_relation}{{18}{17}{}{}{}}
\newlabel{DIFold-eq:qtrans_first_order_relation@cref}{{[equation][18][2147483647]18}{[1][17][]17}}
\newlabel{DIFold-lem:tildeP-P}{{C.2}{19}{}{}{}}
\newlabel{DIFold-lem:tildeP-P@cref}{{[lem][2][2147483647,3]C.2}{[1][18][]19}}
\newlabel{DIFold-lem:occup_bound}{{C.3}{19}{}{}{}}
\newlabel{DIFold-lem:occup_bound@cref}{{[lem][3][2147483647,3]C.3}{[1][19][]19}}
\newlabel{DIFold-dfn:good_event}{{C.4}{19}{}{}{}}
\newlabel{DIFold-dfn:good_event@cref}{{[dfn][4][2147483647,3]C.4}{[1][19][]19}}
\newlabel{DIFold-lem:pdl}{{D.1}{19}{}{}{}}
\newlabel{DIFold-lem:pdl@cref}{{[lem][1][2147483647,4]D.1}{[1][19][]19}}
\newlabel{DIFold-lem:pdl_with_L}{{D.2}{19}{}{}{}}
\newlabel{DIFold-lem:pdl_with_L@cref}{{[lem][2][2147483647,4]D.2}{[1][19][]19}}
\newlabel{DIFold-lem:occ_diff_lem}{{D.3}{19}{}{}{}}
\newlabel{DIFold-lem:occ_diff_lem@cref}{{[lem][3][2147483647,4]D.3}{[1][19][]19}}
\newlabel{DIFold-lem:concentration_var}{{D.4}{19}{}{}{}}
\newlabel{DIFold-lem:concentration_var@cref}{{[lem][4][2147483647,4]D.4}{[1][19][]19}}
\newlabel{DIFold-lem:complicated_lemma_var}{{D.5}{20}{}{}{}}
\newlabel{DIFold-lem:complicated_lemma_var@cref}{{[lem][5][2147483647,4]D.5}{[1][20][]20}}
\newlabel{DIFold-eq:var_term1a}{{19}{21}{}{}{}}
\newlabel{DIFold-eq:var_term1a@cref}{{[equation][19][2147483647]19}{[1][21][]21}}
\newlabel{DIFold-eq:var_term1b}{{20}{21}{}{}{}}
\newlabel{DIFold-eq:var_term1b@cref}{{[equation][20][2147483647]20}{[1][21][]21}}
\newlabel{DIFold-eq:var_term2a}{{21}{22}{}{}{}}
\newlabel{DIFold-eq:var_term2a@cref}{{[equation][21][2147483647]21}{[1][22][]22}}
\newlabel{DIFold-eq:var_term2b}{{22}{23}{}{}{}}
\newlabel{DIFold-eq:var_term2b@cref}{{[equation][22][2147483647]22}{[1][22][]23}}
\newlabel{DIFold-eq:var_term2c}{{23}{23}{}{}{}}
\newlabel{DIFold-eq:var_term2c@cref}{{[equation][23][2147483647]23}{[1][23][]23}}
\newlabel{DIFold-lem:Qtrans_bound}{{D.6}{23}{}{}{}}
\newlabel{DIFold-lem:Qtrans_bound@cref}{{[lem][6][2147483647,4]D.6}{[1][23][]23}}
\newlabel{DIFold-eq:Var_leq_M}{{24}{25}{}{}{}}
\newlabel{DIFold-eq:Var_leq_M@cref}{{[equation][24][2147483647]24}{[1][25][]25}}
\newlabel{DIFold-eq:M_leq_phi}{{25}{25}{}{}{}}
\newlabel{DIFold-eq:M_leq_phi@cref}{{[equation][25][2147483647]25}{[1][25][]25}}
\newlabel{DIFold-lem:complicated_lemma}{{D.7}{25}{}{}{}}
\newlabel{DIFold-lem:complicated_lemma@cref}{{[lem][7][2147483647,4]D.7}{[1][25][]25}}
\newlabel{DIFold-lem:complicated_lemma_cond}{{D.8}{25}{}{}{}}
\newlabel{DIFold-lem:complicated_lemma_cond@cref}{{[lem][8][2147483647,4]D.8}{[1][25][]25}}
\newlabel{DIFold-eq:conditional_term1}{{26}{26}{}{}{}}
\newlabel{DIFold-eq:conditional_term1@cref}{{[equation][26][2147483647]26}{[1][26][]26}}
\newlabel{DIFold-eq:conditional_term2}{{27}{27}{}{}{}}
\newlabel{DIFold-eq:conditional_term2@cref}{{[equation][27][2147483647]27}{[1][26][]27}}
\newlabel{DIFold-eq:conditional_term3}{{28}{27}{}{}{}}
\newlabel{DIFold-eq:conditional_term3@cref}{{[equation][28][2147483647]28}{[1][27][]27}}
\newlabel{DIFold-eq:conditional_term4}{{29}{27}{}{}{}}
\newlabel{DIFold-eq:conditional_term4@cref}{{[equation][29][2147483647]29}{[1][27][]27}}
\newlabel{DIFold-eq:cond_occup_diff}{{30}{27}{}{}{}}
\newlabel{DIFold-eq:cond_occup_diff@cref}{{[equation][30][2147483647]30}{[1][27][]27}}
\newlabel{DIFold-eq:conditional_decomp}{{31}{27}{}{}{}}
\newlabel{DIFold-eq:conditional_decomp@cref}{{[equation][31][2147483647]31}{[1][27][]27}}
\newlabel{DIFold-eq:conditional_decomp1}{{32}{28}{}{}{}}
\newlabel{DIFold-eq:conditional_decomp1@cref}{{[equation][32][2147483647]32}{[1][28][]28}}
\newlabel{DIFold-eq:conditional_decomp2}{{33}{28}{}{}{}}
\newlabel{DIFold-eq:conditional_decomp2@cref}{{[equation][33][2147483647]33}{[1][28][]28}}
\newlabel{DIFold-lem:max_to_value}{{D.9}{28}{}{}{}}
\newlabel{DIFold-lem:max_to_value@cref}{{[lem][9][2147483647,4]D.9}{[1][28][]28}}
\newlabel{DIFold-eq:qmax_decomp}{{34}{29}{}{}{}}
\newlabel{DIFold-eq:qmax_decomp@cref}{{[equation][34][2147483647]34}{[1][28][]29}}
\newlabel{DIFold-eq:qmax_first}{{35}{29}{}{}{}}
\newlabel{DIFold-eq:qmax_first@cref}{{[equation][35][2147483647]35}{[1][29][]29}}
\newlabel{DIFold-eq:qmax_second}{{36}{29}{}{}{}}
\newlabel{DIFold-eq:qmax_second@cref}{{[equation][36][2147483647]36}{[1][29][]29}}
\newlabel{DIFold-lem:L-M_to_l-m}{{D.10}{30}{}{}{}}
\newlabel{DIFold-lem:L-M_to_l-m@cref}{{[lem][10][2147483647,4]D.10}{[1][29][]30}}
\newlabel{DIFold-lem:general_predict_result}{{D.11}{30}{}{}{}}
\newlabel{DIFold-lem:general_predict_result@cref}{{[lem][11][2147483647,4]D.11}{[1][30][]30}}
\newlabel{DIFold-lem:qtrans_m_to_ell}{{D.12}{30}{}{}{}}
\newlabel{DIFold-lem:qtrans_m_to_ell@cref}{{[lem][12][2147483647,4]D.12}{[1][30][]30}}
\newlabel{DIFold-lem:policy_diff_bound}{{D.13}{30}{}{}{}}
\newlabel{DIFold-lem:policy_diff_bound@cref}{{[lem][13][2147483647,4]D.13}{[1][30][]30}}
\newlabel{DIFold-lem:policy_diff_layer}{{D.14}{31}{}{}{}}
\newlabel{DIFold-lem:policy_diff_layer@cref}{{[lem][14][2147483647,4]D.14}{[1][30][]31}}
\newlabel{DIFold-eq:occupancy_action_l1}{{37}{31}{}{}{}}
\newlabel{DIFold-eq:occupancy_action_l1@cref}{{[equation][37][2147483647]37}{[1][30][]31}}
\newlabel{DIFold-eq:state_l1_recursion}{{38}{31}{}{}{}}
\newlabel{DIFold-eq:state_l1_recursion@cref}{{[equation][38][2147483647]38}{[1][31][]31}}
\newlabel{DIFold-eq:state_l1_bound}{{39}{31}{}{}{}}
\newlabel{DIFold-eq:state_l1_bound@cref}{{[equation][39][2147483647]39}{[1][31][]31}}
\newlabel{DIFold-lem:corruption_self-bounding}{{D.15}{32}{}{}{}}
\newlabel{DIFold-lem:corruption_self-bounding@cref}{{[lem][15][2147483647,4]D.15}{[1][31][]32}}
\newlabel{DIFold-lem:general_self_bounding}{{D.16}{32}{}{}{}}
\newlabel{DIFold-lem:general_self_bounding@cref}{{[lem][16][2147483647,4]D.16}{[1][32][]32}}
\newlabel{DIFold-lem:self_bound_layer_positive}{{D.17}{32}{}{}{}}
\newlabel{DIFold-lem:self_bound_layer_positive@cref}{{[lem][17][2147483647,4]D.17}{[1][32][]32}}
\newlabel{DIFold-eq:pio_delta_self}{{40}{33}{}{}{}}
\newlabel{DIFold-eq:pio_delta_self@cref}{{[equation][40][2147483647]40}{[1][33][]33}}
\newlabel{DIFold-lem:self_bound_occ_positive}{{D.18}{33}{}{}{}}
\newlabel{DIFold-lem:self_bound_occ_positive@cref}{{[lem][18][2147483647,4]D.18}{[1][33][]33}}
\newlabel{DIFold-lem:OFTRL_log_barrier_policy}{{E.1}{34}{}{}{}}
\newlabel{DIFold-lem:OFTRL_log_barrier_policy@cref}{{[lem][1][2147483647,5]E.1}{[1][34][]34}}
\newlabel{DIFold-app:full_proof}{{F}{35}{}{}{}}
\newlabel{DIFold-app:full_proof@cref}{{[appendix][6][2147483647]F}{[1][35][]35}}
\newlabel{DIFold-lem:dilated_bonus}{{F.1}{35}{}{}{}}
\newlabel{DIFold-lem:dilated_bonus@cref}{{[lem][1][2147483647,6]F.1}{[1][35][]35}}
\newlabel{DIFold-eq:key_lemma_ineq}{{41}{35}{}{}{}}
\newlabel{DIFold-eq:key_lemma_ineq@cref}{{[equation][41][2147483647]41}{[1][35][]35}}
\newlabel{DIFold-eq:decompose_full}{{42}{35}{}{}{}}
\newlabel{DIFold-eq:decompose_full@cref}{{[equation][42][2147483647]42}{[1][35][]35}}
\newlabel{DIFold-lem:bound_B_full}{{F.2}{35}{}{}{}}
\newlabel{DIFold-lem:bound_B_full@cref}{{[lem][2][2147483647,6]F.2}{[1][35][]35}}
\newlabel{DIFold-eq:dilated_bonus_upper}{{43}{35}{}{}{}}
\newlabel{DIFold-eq:dilated_bonus_upper@cref}{{[equation][43][2147483647]43}{[1][35][]35}}
\newlabel{DIFold-lem:eta_lemma_normal}{{F.3}{36}{}{}{}}
\newlabel{DIFold-lem:eta_lemma_normal@cref}{{[lem][3][2147483647,6]F.3}{[1][36][]36}}
\newlabel{DIFold-lem:learning_rate_full}{{F.4}{36}{}{}{}}
\newlabel{DIFold-lem:learning_rate_full@cref}{{[lem][4][2147483647,6]F.4}{[1][36][]36}}
\newlabel{DIFold-lem:regterm_full}{{F.5}{37}{}{}{}}
\newlabel{DIFold-lem:regterm_full@cref}{{[lem][5][2147483647,6]F.5}{[1][36][]37}}
\newlabel{DIFold-eq:regterm_full_mid}{{44}{37}{}{}{}}
\newlabel{DIFold-eq:regterm_full_mid@cref}{{[equation][44][2147483647]44}{[1][37][]37}}
\newlabel{DIFold-lem:biasterm_full}{{F.6}{39}{}{}{}}
\newlabel{DIFold-lem:biasterm_full@cref}{{[lem][6][2147483647,6]F.6}{[1][39][]39}}
\newlabel{DIFold-eq:bound_bias_1}{{45}{39}{}{}{}}
\newlabel{DIFold-eq:bound_bias_1@cref}{{[equation][45][2147483647]45}{[1][39][]39}}
\newlabel{DIFold-eq:bound_bias_2}{{46}{39}{}{}{}}
\newlabel{DIFold-eq:bound_bias_2@cref}{{[equation][46][2147483647]46}{[1][39][]39}}
\newlabel{DIFold-eq:bound_bias_1_mid}{{47}{40}{}{}{}}
\newlabel{DIFold-eq:bound_bias_1_mid@cref}{{[equation][47][2147483647]47}{[1][39][]40}}
\newlabel{DIFold-eq:bound_bias_2_mid}{{48}{40}{}{}{}}
\newlabel{DIFold-eq:bound_bias_2_mid@cref}{{[equation][48][2147483647]48}{[1][40][]40}}
\newlabel{DIFold-lem:Vbt_full}{{F.7}{41}{}{}{}}
\newlabel{DIFold-lem:Vbt_full@cref}{{[lem][7][2147483647,6]F.7}{[1][41][]41}}
\newlabel{DIFold-eq:res_qBzeta}{{49}{42}{}{}{}}
\newlabel{DIFold-eq:res_qBzeta@cref}{{[equation][49][2147483647]49}{[1][41][]42}}
\newlabel{DIFold-eq:regret_overB}{{50}{42}{}{}{}}
\newlabel{DIFold-eq:regret_overB@cref}{{[equation][50][2147483647]50}{[1][42][]42}}
\newlabel{DIFold-eq:qB_to_q}{{51}{42}{}{}{}}
\newlabel{DIFold-eq:qB_to_q@cref}{{[equation][51][2147483647]51}{[1][42][]42}}
\newlabel{DIFold-lem:regret_to_zetat}{{F.8}{43}{}{}{}}
\newlabel{DIFold-lem:regret_to_zetat@cref}{{[lem][8][2147483647,6]F.8}{[1][42][]43}}
\newlabel{DIFold-eq:regret_with_qb}{{52}{43}{}{}{}}
\newlabel{DIFold-eq:regret_with_qb@cref}{{[equation][52][2147483647]52}{[1][43][]43}}
\newlabel{DIFold-lem:zeta_to_max}{{F.9}{44}{}{}{}}
\newlabel{DIFold-lem:zeta_to_max@cref}{{[lem][9][2147483647,6]F.9}{[1][44][]44}}
\newlabel{DIFold-eq:common_bound_adv_full}{{53}{44}{}{}{}}
\newlabel{DIFold-eq:common_bound_adv_full@cref}{{[equation][53][2147483647]53}{[1][44][]44}}
\newlabel{DIFold-eq:max_square_to_H}{{54}{45}{}{}{}}
\newlabel{DIFold-eq:max_square_to_H@cref}{{[equation][54][2147483647]54}{[1][45][]45}}
\newlabel{DIFold-thm:app_full_adv}{{F.10}{45}{}{}{}}
\newlabel{DIFold-thm:app_full_adv@cref}{{[thm][10][2147483647,6]F.10}{[1][45][]45}}
\newlabel{DIFold-eq:adv_bound_mid}{{55}{45}{}{}{}}
\newlabel{DIFold-eq:adv_bound_mid@cref}{{[equation][55][2147483647]55}{[1][45][]45}}
\newlabel{DIFold-eq:first_order_bound1}{{56}{45}{}{}{}}
\newlabel{DIFold-eq:first_order_bound1@cref}{{[equation][56][2147483647]56}{[1][45][]45}}
\newlabel{DIFold-eq:first_order_bound2}{{57}{46}{}{}{}}
\newlabel{DIFold-eq:first_order_bound2@cref}{{[equation][57][2147483647]57}{[1][45][]46}}
\newlabel{DIFold-thm:app_full_corruption}{{F.11}{46}{}{}{}}
\newlabel{DIFold-thm:app_full_corruption@cref}{{[thm][11][2147483647,6]F.11}{[1][46][]46}}
\newlabel{DIFold-eq:coeff_gap2}{{58}{46}{}{}{}}
\newlabel{DIFold-eq:coeff_gap2@cref}{{[equation][58][2147483647]58}{[1][46][]46}}
\newlabel{DIFold-app:bandit_proofs}{{G}{47}{}{}{}}
\newlabel{DIFold-app:bandit_proofs@cref}{{[appendix][7][2147483647]G}{[1][46][]47}}
\newlabel{DIFold-eq:decompose_bandit}{{59}{47}{}{}{}}
\newlabel{DIFold-eq:decompose_bandit@cref}{{[equation][59][2147483647]59}{[1][47][]47}}
\newlabel{DIFold-lem:qC_size_bound}{{G.1}{47}{}{}{}}
\newlabel{DIFold-lem:qC_size_bound@cref}{{[lem][1][2147483647,7]G.1}{[1][47][]47}}
\newlabel{DIFold-eq:qC_size_bound}{{60}{47}{}{}{}}
\newlabel{DIFold-eq:qC_size_bound@cref}{{[equation][60][2147483647]60}{[1][47][]47}}
\newlabel{DIFold-eq:conditional_occupancy_domination}{{61}{47}{}{}{}}
\newlabel{DIFold-eq:conditional_occupancy_domination@cref}{{[equation][61][2147483647]61}{[1][47][]47}}
\newlabel{DIFold-eq:qC_bound_middle}{{62}{48}{}{}{}}
\newlabel{DIFold-eq:qC_bound_middle@cref}{{[equation][62][2147483647]62}{[1][47][]48}}
\newlabel{DIFold-lem:bound_eta_pi_B}{{G.2}{48}{}{}{}}
\newlabel{DIFold-lem:bound_eta_pi_B@cref}{{[lem][2][2147483647,7]G.2}{[1][48][]48}}
\newlabel{DIFold-eq:bonus_term_upper_real1}{{63}{49}{}{}{}}
\newlabel{DIFold-eq:bonus_term_upper_real1@cref}{{[equation][63][2147483647]63}{[1][48][]49}}
\newlabel{DIFold-eq:bonus_term_upper_real2}{{64}{49}{}{}{}}
\newlabel{DIFold-eq:bonus_term_upper_real2@cref}{{[equation][64][2147483647]64}{[1][49][]49}}
\newlabel{DIFold-eq:bonus_term_upper_real}{{65}{49}{}{}{}}
\newlabel{DIFold-eq:bonus_term_upper_real@cref}{{[equation][65][2147483647]65}{[1][49][]49}}
\newlabel{DIFold-eq:bound_B_real}{{66}{49}{}{}{}}
\newlabel{DIFold-eq:bound_B_real@cref}{{[equation][66][2147483647]66}{[1][49][]49}}
\newlabel{DIFold-eq:bonus_term_upper_virtual}{{67}{50}{}{}{}}
\newlabel{DIFold-eq:bonus_term_upper_virtual@cref}{{[equation][67][2147483647]67}{[1][50][]50}}
\newlabel{DIFold-eq:bound_B_virtual}{{68}{50}{}{}{}}
\newlabel{DIFold-eq:bound_B_virtual@cref}{{[equation][68][2147483647]68}{[1][50][]50}}
\newlabel{DIFold-lem:learning_rate_policy}{{G.3}{51}{}{}{}}
\newlabel{DIFold-lem:learning_rate_policy@cref}{{[lem][3][2147483647,7]G.3}{[1][50][]51}}
\newlabel{DIFold-lem:bound_virtual_episode}{{G.4}{51}{}{}{}}
\newlabel{DIFold-lem:bound_virtual_episode@cref}{{[lem][4][2147483647,7]G.4}{[1][51][]51}}
\newlabel{DIFold-lem:bound_rho2}{{G.5}{52}{}{}{}}
\newlabel{DIFold-lem:bound_rho2@cref}{{[lem][5][2147483647,7]G.5}{[1][51][]52}}
\newlabel{DIFold-lem:bound_rho2_upper}{{G.6}{52}{}{}{}}
\newlabel{DIFold-lem:bound_rho2_upper@cref}{{[lem][6][2147483647,7]G.6}{[1][52][]52}}
\newlabel{DIFold-lem:regterm_policy}{{G.7}{53}{}{}{}}
\newlabel{DIFold-lem:regterm_policy@cref}{{[lem][7][2147483647,7]G.7}{[1][53][]53}}
\newlabel{DIFold-eq:regterm_policy_mid}{{69}{54}{}{}{}}
\newlabel{DIFold-eq:regterm_policy_mid@cref}{{[equation][69][2147483647]69}{[1][54][]54}}
\newlabel{DIFold-lem:biasterm_bandit}{{G.8}{55}{}{}{}}
\newlabel{DIFold-lem:biasterm_bandit@cref}{{[lem][8][2147483647,7]G.8}{[1][55][]55}}
\newlabel{DIFold-eq:ct_dominates_bias0}{{70}{55}{}{}{}}
\newlabel{DIFold-eq:ct_dominates_bias0@cref}{{[equation][70][2147483647]70}{[1][55][]55}}
\newlabel{DIFold-eq:ct_dominates_bias}{{71}{55}{}{}{}}
\newlabel{DIFold-eq:ct_dominates_bias@cref}{{[equation][71][2147483647]71}{[1][55][]55}}
\newlabel{DIFold-eq:def_zt_new_c}{{72}{56}{}{}{}}
\newlabel{DIFold-eq:def_zt_new_c@cref}{{[equation][72][2147483647]72}{[1][56][]56}}
\newlabel{DIFold-eq:high_prob_bias}{{73}{57}{}{}{}}
\newlabel{DIFold-eq:high_prob_bias@cref}{{[equation][73][2147483647]73}{[1][56][]57}}
\newlabel{DIFold-eq:ct_upper_for_term1}{{74}{57}{}{}{}}
\newlabel{DIFold-eq:ct_upper_for_term1@cref}{{[equation][74][2147483647]74}{[1][57][]57}}
\newlabel{DIFold-eq:Et_C_lower}{{75}{57}{}{}{}}
\newlabel{DIFold-eq:Et_C_lower@cref}{{[equation][75][2147483647]75}{[1][57][]57}}
\newlabel{DIFold-eq:zt_geq0_new_c}{{76}{57}{}{}{}}
\newlabel{DIFold-eq:zt_geq0_new_c@cref}{{[equation][76][2147483647]76}{[1][57][]57}}
\newlabel{DIFold-eq:zt_upper_before_unroll_new_c}{{77}{57}{}{}{}}
\newlabel{DIFold-eq:zt_upper_before_unroll_new_c@cref}{{[equation][77][2147483647]77}{[1][57][]57}}
\newlabel{DIFold-eq:C_unroll_new_c}{{78}{58}{}{}{}}
\newlabel{DIFold-eq:C_unroll_new_c@cref}{{[equation][78][2147483647]78}{[1][57][]58}}
\newlabel{DIFold-eq:zt_upper_new_c}{{79}{58}{}{}{}}
\newlabel{DIFold-eq:zt_upper_new_c@cref}{{[equation][79][2147483647]79}{[1][57][]58}}
\newlabel{DIFold-def:proof_G}{{80}{58}{}{}{}}
\newlabel{DIFold-def:proof_G@cref}{{[equation][80][2147483647]80}{[1][58][]58}}
\newlabel{DIFold-eq:term1a_final_new_c}{{81}{59}{}{}{}}
\newlabel{DIFold-eq:term1a_final_new_c@cref}{{[equation][81][2147483647]81}{[1][59][]59}}
\newlabel{DIFold-eq:term1b_final_new_c}{{82}{59}{}{}{}}
\newlabel{DIFold-eq:term1b_final_new_c@cref}{{[equation][82][2147483647]82}{[1][59][]59}}
\newlabel{DIFold-eq:term1c_final_new_c}{{83}{59}{}{}{}}
\newlabel{DIFold-eq:term1c_final_new_c@cref}{{[equation][83][2147483647]83}{[1][59][]59}}
\newlabel{DIFold-eq:term1_final_new_c}{{84}{59}{}{}{}}
\newlabel{DIFold-eq:term1_final_new_c@cref}{{[equation][84][2147483647]84}{[1][59][]59}}
\newlabel{DIFold-eq:g2_square_bound}{{85}{60}{}{}{}}
\newlabel{DIFold-eq:g2_square_bound@cref}{{[equation][85][2147483647]85}{[1][60][]60}}
\newlabel{DIFold-eq:g2_square_bound2}{{86}{60}{}{}{}}
\newlabel{DIFold-eq:g2_square_bound2@cref}{{[equation][86][2147483647]86}{[1][60][]60}}
\newlabel{DIFold-eq:term2a_final_new_c}{{87}{60}{}{}{}}
\newlabel{DIFold-eq:term2a_final_new_c@cref}{{[equation][87][2147483647]87}{[1][60][]60}}
\newlabel{DIFold-eq:term2b_final_new_c}{{88}{60}{}{}{}}
\newlabel{DIFold-eq:term2b_final_new_c@cref}{{[equation][88][2147483647]88}{[1][60][]60}}
\newlabel{DIFold-eq:term2c_final_new_c}{{89}{61}{}{}{}}
\newlabel{DIFold-eq:term2c_final_new_c@cref}{{[equation][89][2147483647]89}{[1][60][]61}}
\newlabel{DIFold-eq:term2_final_new_c}{{90}{61}{}{}{}}
\newlabel{DIFold-eq:term2_final_new_c@cref}{{[equation][90][2147483647]90}{[1][61][]61}}
\newlabel{DIFold-eq:term3_splicing_bound}{{91}{61}{}{}{}}
\newlabel{DIFold-eq:term3_splicing_bound@cref}{{[equation][91][2147483647]91}{[1][61][]61}}
\newlabel{DIFold-eq:term3a_maxQ_bound}{{92}{62}{}{}{}}
\newlabel{DIFold-eq:term3a_maxQ_bound@cref}{{[equation][92][2147483647]92}{[1][62][]62}}
\newlabel{DIFold-eq:term3b_count_bound}{{93}{62}{}{}{}}
\newlabel{DIFold-eq:term3b_count_bound@cref}{{[equation][93][2147483647]93}{[1][62][]62}}
\newlabel{DIFold-eq:term3_maxQ_bound}{{94}{62}{}{}{}}
\newlabel{DIFold-eq:term3_maxQ_bound@cref}{{[equation][94][2147483647]94}{[1][62][]62}}
\newlabel{DIFold-eq:term3_positive_occupancy_layer_bound}{{95}{63}{}{}{}}
\newlabel{DIFold-eq:term3_positive_occupancy_layer_bound@cref}{{[equation][95][2147483647]95}{[1][63][]63}}
\newlabel{DIFold-lem:errorterm_bandit}{{G.9}{63}{}{}{}}
\newlabel{DIFold-lem:errorterm_bandit@cref}{{[lem][9][2147483647,7]G.9}{[1][63][]63}}
\newlabel{DIFold-eq:bound_error_1}{{96}{63}{}{}{}}
\newlabel{DIFold-eq:bound_error_1@cref}{{[equation][96][2147483647]96}{[1][63][]63}}
\newlabel{DIFold-eq:bound_error_2}{{97}{63}{}{}{}}
\newlabel{DIFold-eq:bound_error_2@cref}{{[equation][97][2147483647]97}{[1][63][]63}}
\newlabel{DIFold-lem:squared_C_bound}{{G.10}{64}{}{}{}}
\newlabel{DIFold-lem:squared_C_bound@cref}{{[lem][10][2147483647,7]G.10}{[1][64][]64}}
\newlabel{DIFold-eq:C-square-decomp}{{98}{64}{}{}{}}
\newlabel{DIFold-eq:C-square-decomp@cref}{{[equation][98][2147483647]98}{[1][64][]64}}
\newlabel{DIFold-eq:squared_C_to_rho}{{99}{65}{}{}{}}
\newlabel{DIFold-eq:squared_C_to_rho@cref}{{[equation][99][2147483647]99}{[1][65][]65}}
\newlabel{DIFold-eq:squared_C_final_bound}{{100}{65}{}{}{}}
\newlabel{DIFold-eq:squared_C_final_bound@cref}{{[equation][100][2147483647]100}{[1][65][]65}}
\newlabel{DIFold-lem:log_barrier_bt_bandit}{{G.11}{66}{}{}{}}
\newlabel{DIFold-lem:log_barrier_bt_bandit@cref}{{[lem][11][2147483647,7]G.11}{[1][65][]66}}
\newlabel{DIFold-eq:regret_sum_bandit}{{101}{67}{}{}{}}
\newlabel{DIFold-eq:regret_sum_bandit@cref}{{[equation][101][2147483647]101}{[1][67][]67}}
\newlabel{DIFold-thm:app_bandit_adv}{{G.12}{67}{}{}{}}
\newlabel{DIFold-thm:app_bandit_adv@cref}{{[thm][12][2147483647,7]G.12}{[1][67][]67}}
\newlabel{DIFold-eq:bandit_regret_decomp_adv}{{102}{68}{}{}{}}
\newlabel{DIFold-eq:bandit_regret_decomp_adv@cref}{{[equation][102][2147483647]102}{[1][67][]68}}
\newlabel{DIFold-eq:common_bound_adv}{{103}{68}{}{}{}}
\newlabel{DIFold-eq:common_bound_adv@cref}{{[equation][103][2147483647]103}{[1][68][]68}}
\newlabel{DIFold-eq:adv_bandit_term1}{{104}{68}{}{}{}}
\newlabel{DIFold-eq:adv_bandit_term1@cref}{{[equation][104][2147483647]104}{[1][68][]68}}
\newlabel{DIFold-eq:adv_bandit_term2}{{105}{68}{}{}{}}
\newlabel{DIFold-eq:adv_bandit_term2@cref}{{[equation][105][2147483647]105}{[1][68][]68}}
\newlabel{DIFold-eq:adv_bandit_term3}{{106}{68}{}{}{}}
\newlabel{DIFold-eq:adv_bandit_term3@cref}{{[equation][106][2147483647]106}{[1][68][]68}}
\newlabel{DIFold-eq:adv_bandit_mid}{{107}{69}{}{}{}}
\newlabel{DIFold-eq:adv_bandit_mid@cref}{{[equation][107][2147483647]107}{[1][68][]69}}
\newlabel{DIFold-eq:first_order_bandit1}{{108}{69}{}{}{}}
\newlabel{DIFold-eq:first_order_bandit1@cref}{{[equation][108][2147483647]108}{[1][69][]69}}
\newlabel{DIFold-eq:first_order_bandit2}{{109}{69}{}{}{}}
\newlabel{DIFold-eq:first_order_bandit2@cref}{{[equation][109][2147483647]109}{[1][69][]69}}
\newlabel{DIFold-thm:app_bandit_corruption}{{G.13}{69}{}{}{}}
\newlabel{DIFold-thm:app_bandit_corruption@cref}{{[thm][13][2147483647,7]G.13}{[1][69][]69}}
\newlabel{DIFold-eq:bandit_regret_decomp_sto}{{110}{69}{}{}{}}
\newlabel{DIFold-eq:bandit_regret_decomp_sto@cref}{{[equation][110][2147483647]110}{[1][69][]69}}
\newlabel{DIFold-eq:sto_bandit_term1}{{111}{70}{}{}{}}
\newlabel{DIFold-eq:sto_bandit_term1@cref}{{[equation][111][2147483647]111}{[1][69][]70}}
\newlabel{DIFold-eq:sto_bandit_term2}{{112}{70}{}{}{}}
\newlabel{DIFold-eq:sto_bandit_term2@cref}{{[equation][112][2147483647]112}{[1][70][]70}}
\newlabel{DIFold-eq:sto_bandit_term3}{{113}{70}{}{}{}}
\newlabel{DIFold-eq:sto_bandit_term3@cref}{{[equation][113][2147483647]113}{[1][70][]70}}
}

\title{Policy Optimization Achieves Data-Dependent Regret Bounds in MDPs with Unknown Transitions}

\author{%
    Mingyi Li\\
    Department of Mathematical Informatics, The University of Tokyo\\
    \texttt{mingyi-mike@g.ecc.u-tokyo.ac.jp}\\
    \And
    Taira Tsuchiya\\
    Department of Mathematical Informatics, The University of Tokyo\\
    Center for Advanced Intelligence Project, RIKEN\\
    \texttt{tsuchiya@mist.i.u-tokyo.ac.jp}\\
    \And
    Kenji Yamanishi\\
    Department of Mathematical Informatics, The University of Tokyo\\
    \texttt{yamanishi@g.ecc.u-tokyo.ac.jp}\\
}

\begin{document}

\maketitle

\begin{abstract}
We study policy optimization for online episodic tabular Markov decision processes with unknown transition kernels, aiming for best-of-both-worlds guarantees together with data-dependent regret bounds.
Recent work \citep{dann2023best,li2026data} has shown that policy optimization can adapt to both adversarial and stochastic losses with first-order, second-order, and path-length bounds, but only under known transitions, leaving open whether such data-dependent guarantees are achievable by policy optimization when the transition kernel is unknown.
We resolve this by developing a new algorithm based on optimistic follow-the-regularized-leader that attains these guarantees under unknown transitions.
The key ingredient is a new design of optimistic $Q$-function estimators together with a data-dependent transition bonus that controls estimator bias through the loss-prediction error.
Our analysis further identifies an unavoidable transition-dependent complexity term that captures the intrinsic cost of estimating the transition kernel. 
As a result, we obtain first-order, second-order, and path-length bounds with the transition-dependent complexity term while simultaneously achieving gap-dependent $\mathrm{polylog}(T)$ regret in the stochastic regime.
\end{abstract}

\section{Introduction}
We study policy optimization for online finite-horizon episodic tabular Markov decision processes (MDPs) with unknown transitions.
In this setting, the learner interacts with an MDP over $T$ episodes. In each episode, the environment selects a loss function while the learner selects a policy. The learner then executes it and receives feedback on the loss function. The goal is to minimize the regret, the gap between the cumulative loss of the learner and that of the best fixed policy in hindsight.
Policy optimization directly updates the policy and is a central paradigm in reinforcement learning (e.g.,~\citealt{schulman2015trust,schulman2017proximal}).
For tabular MDPs, this approach updates a separate action distribution for each state \citep{shani2020optimistic,luo2021policy}, which is often more computationally efficient and practical than global optimization over the occupancy-measure space.

The difficulty of tabular MDPs depends on how the loss sequence is generated. In the adversarial regime, the worst-case regret scales as $\tilO(\sqrt{T})$~\citep{jin2020learning,luo2021policy}. In contrast, in the stochastic regime, one can often achieve much faster gap-dependent  $\polylog(T)$ regret~\citep{simchowitz2019non,chen2025sharp}. This contrast has motivated best-of-both-worlds algorithms, which aim to achieve near-optimal guarantees in both the adversarial and stochastic regimes with a single algorithm~\citep{jin2020simultaneously,jin2021best,dann2023best}.

Beyond adapting to the underlying regime, another natural goal is to obtain regret bounds that scale with the difficulty of the realized loss sequence in the adversarial regime.
Algorithms only with worst-case regret guarantees ignore the structure of losses and are often overly pessimistic in practical environments.
Motivated by this, data-dependent bounds, such as first-order, second-order, and path-length bounds, have been actively studied in online learning, including prediction with expert advice and multi-armed bandits~\citep{cesa1996worst,allenberg2006hannan,neu2015first,wei2018more,bubeck2019improved}. These depend, respectively, on the comparator's cumulative loss, the magnitude of the loss fluctuations, and the temporal variation of the losses.

Such data-dependent bounds have also been studied in episodic MDPs, but existing data-dependent guarantees in episodic MDPs are limited either in the algorithmic approach or in the transition model.
Under unknown transitions, \citet{lee2020bias} obtained a first-order bound, but only through occupancy-measure-based global optimization rather than policy optimization.
For policy optimization, such data-dependent bounds have so far been obtained only under \emph{known transitions}: \citet{dann2023best} established best-of-both-worlds guarantees with a first-order bound, and \citet{li2026data} provided algorithms achieving first-order, second-order, and path-length bounds simultaneously.
Consequently, whether policy optimization can attain any data-dependent bound under unknown transitions, let alone combine first-order, second-order, and path-length adaptivity with best-of-both-worlds guarantees, was left open by \citet{dann2023best,li2026data}.
This raises the following question.

\begin{center}
\emph{Can policy optimization achieve data-dependent bounds such as first-order, second-order, and path-length bounds in episodic MDPs with unknown transitions?}
\end{center}

\begin{table}[t]
\centering
\caption{Comparison of regret upper bounds based on policy optimization. Here, $\Lambda(\pi)=\min\set*{L(\pi),HT-L(\pi),Q_\infty,V_1}$ is the loss-dependent complexity, $\Qtrans^{\pi_{1:T}}(\ell)$ is the transition-dependent complexity satisfying $\E[\Qtrans^{\pi_{1:T}}(\ell)]\leq H(L(\pi)+\Reg_T(\pi))$ for any policy $\pi$, and \textsc{FO}~indicates that only first-order adaptivity is achieved. For the stochastic regime with adversarial corruption, we report only the leading term $U$ in bounds of the form $\calO\prn{U+\sqrt{U\mathcal C}}$.}
\label{tab:comparison}
\resizebox{\textwidth}{!}{
\begin{tabular}{@{}lllccc@{}}
\toprule
Reference & Adversarial regime & \small\makecell{Stochastic regime with\\ adversarial corruption} & \small\makecell{Unknown\\ transition} & \small\makecell{Data-\\dependent} & \small\makecell{Bandit\\feedback} \\
\midrule
\cite{luo2021policy} & $\tilO\prn{\sqrt{H^4S^2AT}}$ & -- & \cmark &  & \cmark \\
\citet[Theorem~4.2]{dann2023best} & $\tilO\prn{\sqrt{H^4S^2AT}}$ & $\frac{H^4S^2A\log^2(T)}{\Delta_{\min}}$ & \cmark &  & \cmark \\
\citet[Theorem~4.3]{dann2023best} & $\tilO\prn{\sqrt{H^2SAL(\pi)}}$ & $\sum_s\sum_{a\neq\pi^\star(s)}\frac{H^2\log^2(T)}{\Delta(s,a)}$ &  & \textsc{FO} & \cmark \\
\citet{li2026data} & $\tilO\prn{\sqrt{H^2SA\Lambda(\pi)}}$ & $\sum_s\sum_{a\neq\pi^\star(s)}\frac{H^2\log^2(T)}{\Delta(s,a)}$ &  & \cmark & \cmark \\
\midrule
\textbf{This work (\cref{thm:main_full})} & $\tilO\prn{\sqrt{H^2SA\Lambda(\pi)}+\sqrt{S^2A\E[\Qtrans^{\pi_{1:T}}(\ell)]}}$ & $\frac{H^2S^2A\log^2(T)}{\Delta_{\min}}$ & \cmark & \cmark &  \\
\textbf{This work (\cref{thm:main_bandit})} & $\tilO\prn{\sqrt{H^3S^2A\Lambda(\pi)}+\sqrt{S^2A\E[\Qtrans^{\pi_{1:T}}(\ell)]}}$ & $\frac{H^4S^2A\log^2(T)}{\Delta_{\min}}$ & \cmark & \cmark & \cmark \\
\bottomrule
\end{tabular}
}
\vspace{-10pt}
\end{table}

\paragraph{Contributions of this paper.}
We answer this question affirmatively by developing a policy optimization algorithm for episodic tabular MDPs with unknown transitions under \emph{bandit feedback}, the harder setting in which the learner must estimate both the losses and the transition kernel from the realized trajectories.
In the adversarial regime, its regret adapts to $\min\{L(\pi),HT-L(\pi),Q_\infty,V_1\}$, covering first-order, second-order, and path-length complexities, 
together with a transition-dependent complexity term $\Qtrans^{\pi_{1:T}}(\ell)$ in \cref{def:qtrans_loss} that captures the intrinsic cost of estimating the transition kernel and can be connected to the first-order complexity $L(\pi)$ (see \cref{app:qtrans-discussion} for details). 
The same algorithm also enjoys gap-dependent $\polylog(T)$ regret in the stochastic regime, achieving the best-of-both-worlds guarantee (see \cref{thm:main_bandit}).
To our knowledge, this is the first data-dependent guarantee for policy optimization under unknown transitions.
It is worth noting that this is also the first algorithm for tabular MDPs achieving best-of-both-worlds and data-dependent bounds simultaneously under unknown transitions.
A detailed comparison with prior policy-optimization results is given in \cref{tab:comparison}.

As a by-product, our algorithm carries over to the easier full-information setting, where it retains these data-dependent and best-of-both-worlds guarantees and, in the worst case, recovers the $\tilO(\sqrt{H^2S^2AT})$ rate of occupancy-measure-based methods \citep{rosenberg2019online,jin2021best} (see \cref{thm:main_full}).
This implies that the extra $H$ factor often attributed to $Q$-function estimation is not inherent for policy optimization in the full-information setting under unknown transitions.

Technically, our algorithms build on the optimistic follow-the-regularized-leader framework, with new prediction terms and optimistic $Q$-function estimators designed for unknown transitions tailored for each feedback setting (see \cref{sec:algorithm}).
For the full-information setting, the $Q$-function estimator uses the most optimistic transition kernel in the confidence set.
For the bandit setting, we incorporate a new data-dependent transition bonus defined in \cref{eq:def_new_c,eq:def_C} into the $Q$-function estimator.
This bonus keeps the $Q$-function estimator optimistic while reducing the estimator bias to the loss-prediction error, with the remaining worst-case terms contributing only lower-order terms.

Our analysis identifies the new transition-dependent complexity term $\Qtrans^{\pi_{1:T}}(\ell)$, which captures the cost of estimating the unknown transition kernel.
We show that this term is unavoidable under unknown transitions: even for a time-invariant loss sequence with $Q_\infty=V_1=0$, any algorithm still incurs regret from estimating the unknown transition kernel (see \cref{prop:qtrans}).
This explains why data-dependent regret bounds under unknown transitions cannot be characterized solely by the loss-complexity measures used in the known-transition setting of \citet{li2026data}.
Due to space constraints, we defer a more detailed discussion of related work to \cref{app:add_relate}.
\section{Preliminaries}
\label{sec:pre}
\paragraph{Episodic tabular MDPs.}
We study an episodic tabular Markov decision process (MDP) $\mathcal{M}=(\calS,\calA, P, H, s_0)$, where $\calS$ is a finite state space with $S=\abs{\calS}$, $\calA$ is a finite action space with $A=\abs{\calA}$, $H$ is the horizon length, and $s_0$ is the
initial state. The transition kernel $P$ is fixed but unknown to the learner. For each state-action pair $(s,a)$, $P(s'\mid s,a)$ denotes the probability of moving to state $s'$ after taking action $a$ in state $s$.
We adopt the standard layered MDP assumption \citep{neu2010online,jin2020learning,luo2021policy},
where the state space is partitioned into $H+1$ disjoint layers $\calS_0,\calS_1,\ldots,\calS_H$. Here, $\calS_0=\{s_0\}$ is the initial layer and $\calS_H=\{s_H\}$ is a terminal absorbing layer. For simplicity, we exclude $s_H$ from $\calS$, so that $H\leq S$. Transitions occur only between consecutive layers. Specifically, for any $(s,a)\in\calS_h\times\calA$ with $h\in\{0,\ldots,H-1\}$, the distribution $P(\cdot\mid s,a)$ is supported on $\calS_{h+1}$. We write $h(s)$ for the layer index of state $s$.

The learner interacts with the MDP for $T$ episodes, indexed by $t=1,\ldots,T$. At the beginning of episode $t$, the environment selects a loss function $\ell_t:\calS\times\calA\to[0,1]$. 
The learner then selects a policy $\pi_t$ based on past observations, executes it from $s_0$ under $P$, and incurs the losses along the realized trajectory.
Here, a stochastic policy maps each state $s$ to a
distribution $\pi(\cdot\mid s)$ over actions, and for a deterministic policy
$\pi$, we write $\pi(s)\in\calA$ for its selected action at state $s$.
At the end of the episode, the learner receives feedback on $\ell_t$. In the
\emph{full-information} setting, the feedback is the entire loss function $\ell_t$, whereas in the \emph{bandit} setting, it is limited to the realized losses $\{(s_{t,h},a_{t,h},\ell_t(s_{t,h},a_{t,h}))\}_{h=0}^{H-1}$.
We assume $T\geq\max\{2,S,A\}$ for notational convenience.\footnote{The assumptions $T\geq S$ and $T\geq A$ are not essential. If they fail, the analysis remains valid after replacing $\ln(T)$ with $\ln(SAT)$ or $\ln(AT)$.}

For a policy $\pi$ and a loss function $\ell$, the value functions are defined recursively, with terminal condition $V^\pi(s_H;\ell)=0$. 
Specifically, the value function $V^\pi(s;\ell)$ and the $Q$-function $Q^\pi(s,a;\ell)$ satisfy
$V^\pi(s;\ell) = \E_{a\sim \pi(\cdot\mid s)}\brk{Q^\pi(s,a;\ell)}$
and $Q^\pi(s,a;\ell) = \ell(s,a) + \E_{s'\sim P(\cdot\mid s,a)}\brk{V^\pi(s';\ell)}$.
We may overload notation by allowing any function $g:\calS\times\calA\to\R$ to play the role of the loss function, in which case $V^\pi(s;g)$ and $Q^\pi(s,a;g)$ are defined accordingly.
The goal is to minimize the regret 
$\Reg_T(\pi)\coloneqq\E\brk[\big]{\sum_{t=1}^T V^{\pi_t}(s_0;\ell_t)-\sum_{t=1}^T V^{\pi}(s_0;\ell_t)}$
against any fixed comparator policy $\pi$,
and we denote by $\pio \in \argmin_{\pi}\E\brk[\big]{\sum_{t=1}^T V^{\pi}(s_0;\ell_t)}$ an optimal policy in hindsight.

\paragraph{Additional notation.}
For $N\in\N$, let $[N]\coloneqq\{1,\dots,N\}$, given a vector $x$, let $\norm{x}_p$ be its $\ell_p$-norm for $p\in[1,\infty]$, and for a scalar $z$, let $[z]_+\coloneqq \max\{z,0\}$.
For a set $\calK$, let $\Delta(\calK)$ denote the set of probability distributions over $\calK$.
If $f$ and $g$ are functions with $g(x)>0$, we write $f\lesssim g$ or $f=\calO(g)$ when $f(x)\leq c g(x)$ holds on the relevant domain for some universal constant $c>0$, and use $\tilO(\cdot)$ to hide logarithmic factors.
Let $\{\calF_t\}_{t\geq0}$ be the natural filtration generated by all observations up to the end of episode $t - 1$, and define $\E_t[\cdot] = \E[\cdot\mid \calF_{t},\ell_t]$.
We use $\ind[\cdot]$ for indicators, which equal $1$ on the event and $0$ otherwise. Let $\I_t(s,a)=\ind\brk*{(s_{t,h},a_{t,h})=(s,a),\exists h\in\{0,\ldots,H-1\}}$ denote whether a state-action pair $(s,a)$ is visited in episode $t$ under policy $\pi_t$ and transition kernel $P$, and let $\I_t(s)=\sum_{a} \I_t(s,a)$.
For each layer $h$, write $\ell_t(h)\in[0,1]^{\calS_h\times\calA}$ for the restriction of $\ell_t$ to $\calS_h\times\calA$, and use the same convention for other functions on $\calS\times\calA$.

\subsection{Confidence sets of the transition}
Following \citet{jin2020learning}, we construct confidence sets for the unknown transition kernel. Let $n_t(s,a,s')$ be the number of observed transitions from $(s,a)$ to $s'$ up to episode $t - 1$, and let $n_t(s,a)$ be the corresponding visitation count. 
Whenever $n_t(s,a)=0$, we set it to $1$ in any denominator.
The empirical transition estimate is then defined by $\hatP_t(s'\mid s,a)=\frac{n_t(s,a,s')}{n_t(s,a)}$ for all $h$ and $(s,a,s')\in\calS_h\times\calA\times\calS_{h+1}$.
The confidence set of the true transition at episode $t$ is then defined as
\begin{align}
    \calP_t = \set[\Bigg]{\tilP:&\abs*{\tilP(s'\mid s,a) - \hatP_t(s'\mid s, a)} \leq \min\set[\bigg]{2\sqrt{\frac{\hatP_t(s'\mid s, a)\iota}{n_t(s,a)}} + \frac{14\iota}{3n_t(s,a)}, 1},\\[-1mm]
    & \forall (s,a,s') \in \calS_h\times\calA\times\calS_{h+1},\ \forall h = 0, \dots, H - 1,\  \tilP(\cdot\mid s, a) \in \Delta(\calS_{h+1}).}, \label{def:calP_t}
\end{align}
where $\iota=\ln(SAT/\delta)$. Throughout the paper, we set $\delta=1/T^2$.

By \citet[Lemma 2]{jin2020learning}, the true transition kernel $P$ belongs to $\calP_t$ for all episodes $t=1,\ldots,T$ with probability at least $1-4\delta$. We note that the set $\calP_t$ is rectangular, since its
constraints are imposed independently for each state-action pair $(s,a)$.

For a transition kernel $\tilP$ and a policy $\pi$, let $q^{\tilP,\pi}(s,a)$ be the probability of visiting $(s,a)$ within an episode. We also use $q^{\tilP,\pi}(s'\mid s,a)$ and $q^{\tilP,\pi}(s',a'\mid s,a)$ for the corresponding conditional occupancy measures given that $(s,a)$ has been visited and set them to zero whenever $h(s')<h(s)$. For each state $s$, define $q^{\tilP,\pi}(s)\coloneqq \sum_a q^{\tilP,\pi}(s,a)$, so that $q^{\tilP,\pi}(s,a)=q^{\tilP,\pi}(s)\pi(a\mid s)$. When $\tilP=P$, we simply write $q^\pi$.
We also define the upper and lower occupancies $\overq_t^\pi(s)$ and $\underq_t^\pi(s)$ with respect to the confidence set $\calP_t$, and for the learner's policy $\pi_t$, the smoothed occupancy $q_t(s)$ and relative occupancy width $\rho_t(s)$ using an exploration rate $\gamma_t>0$:
\begin{align}
    \overq_t^\pi(s)\!\coloneqq\!\max_{\tilP\in\calP_t} q^{\tilP,\pi}(s),\hspace{2mm}
    \underq_t^\pi(s)\!\coloneqq\!\min_{\tilP\in\calP_t} q^{\tilP,\pi}(s),\hspace{2mm}
    q_t(s)\!\coloneqq\! \overq^{\pi_t}_t(s) + \gamma_t,\hspace{2mm}
    \rho_t(s) \!\coloneqq\! \frac{q_t(s) - \underq_t^{\pi_t}(s)}{q_t(s)}.
\end{align}
These quantities can be computed efficiently by the dynamic-programming procedure of \citet[Algorithm~3]{jin2020learning}. 
Similarly, for any transition kernel $\tilP$, policy $\pi$, and function $g:\calS\times\calA\to\mathbb{R}$, we write $V^{\tilP,\pi}(s;g)$ and $Q^{\tilP,\pi}(s,a;g)$ for the value and $Q$-functions computed under $\tilP$ and $\pi$. When $\tilP=P$, we again omit the transition kernel and write $V^\pi(s;g)$ and $Q^\pi(s,a;g)$. 

\subsection{Regime of environments}
\label{subsec:regimes}

We consider three regimes for the loss sequence $\ell_1,\ldots,\ell_T$.
In the adversarial regime, we impose no assumption on the losses. At the beginning of each episode $t$, the environment selects an arbitrary loss function $\ell_t\in[0,1]^{S\times A}$, which may depend on the learner's algorithm, as well as on past trajectories and observed losses, but not on the learner's internal randomness.

In the stochastic regime, the loss functions $\ell_1,\ldots,\ell_T$ are sampled i.i.d.~from a fixed but unknown distribution $\calD$.
The stochastic regime with adversarial corruption interpolates between these two settings. In this regime, there is an underlying (non-corrupted) loss sequence $\ell'_1,\ldots,\ell'_T$ sampled i.i.d.~from an unknown distribution $\calD$, while the learner observes a possibly corrupted loss sequence $\ell_1,\ldots,\ell_T$.
The corruptions may depend arbitrarily on the past trajectories and observed losses, and their total magnitude is measured by
$\calC\coloneqq \E\brk[\big]{\sum_{t=1}^T\sum_{h=0}^{H-1}\|\ell'_t(h)-\ell_t(h)\|_\infty} \in [0, HT]$.
When $\calC=0$, this regime reduces to the stochastic regime; when $\calC$ is arbitrary, it coincides with the adversarial regime.
For the stochastic and corrupted stochastic regimes, let
$\mu(s,a)\coloneqq \E_{\ell'\sim\calD}\brk*{\ell'(s,a)}$
be the mean uncorrupted loss function. Assume that the optimal policy for $\mu$ is unique, and let it be denoted by $\pist$. We define the suboptimality gap of a state-action pair $(s,a)$ by
$\Delta\colon\calS\times\calA\to[0,H]$ as $\Delta(s,a)\coloneqq Q^{\pi^\star}(s,a;\mu)-\min_{a'\in\calA}Q^{\pi^\star}(s,a';\mu)$ and let $\Delta_{\min}\coloneqq \min_{s, a\neq \pist(s)}\Delta(s,a) > 0$.

\subsection{Complexity measures in online MDPs}\label{sec:data_dependent_measures}

We use the following data-dependent complexity measures to state refined regret bounds. 
The first-order quantity introduced by \citet{lee2020bias}
is $L(\pi)\coloneqq\E\brk[\big]{\sum_{t=1}^T V^\pi(s_0;\ell_t)}$, which is the cumulative loss of a fixed comparator policy $\pi$.
The second-order quantity
$Q_\infty\coloneqq\min_{\ell^\star\in[0,1]^{S\times A}}\E\brk[\big]{\sum_{t=1}^T\sum_{h=0}^{H-1}\norm{\ell_t(h)-\ell^\star(h)}_\infty^2}$, which is small when the losses stay close to a fixed baseline loss, and the path-length quantity
$V_1\coloneqq\E\brk[\big]{\sum_{t=1}^{T-1}\nrm{\ell_{t+1}-\ell_t}_1}$, which is small when the losses vary slowly over time, are defined as in \citet{li2026data}.

Under unknown transitions, our regret bounds contain a transition-dependent complexity term that reflects the difficulty of estimating the transition kernel.
For a loss sequence $\ell=\{\ell_t\}_{t=1}^T$ and a policy sequence
$\pi_{1:T}=\{\pi_t\}_{t=1}^T$, define
  \begin{align}
    \Qtrans^{\pi_{1:T}}(\ell)
    \coloneqq
    \sum_{t=1}^T\sum_{s,a}
    q^{\pi_t}(s,a)
    \Var_{s'\sim P(\cdot\mid s,a)}
    \brk*{V^{\pi_t}(s';\ell_t)} \in [0,H^2T].
    \label{def:qtrans_loss}
\end{align} 
This quantity measures the variance of future cumulative losses across next states. Thus, transition uncertainty affects regret only when possible next states have different values   $V^{\pi_t}(s';\ell_t)$.
This mirrors transition-variance terms in stochastic MDPs, such as $\Var_{s'\sim P(\cdot\mid s,a)}\brk*{V^{\pi^\star}(s'; \mu)}$ \citep{zanette2019tighter,simchowitz2019non}.
Our definition replaces the optimal stochastic value $V^{\pi^\star}(\,\cdot\,;\mu)$ with learner-induced value functions so that the same idea applies to arbitrary loss sequences.
As discussed later, this transition-dependent term is unavoidable under unknown transitions, since learning the transition kernel can itself cause regret even when $Q_\infty = V_1 = 0$.
Furthermore, this term recovers a first-order regret bound through a self-bounding argument.  In particular, the relation $\E[\Qtrans^{\pi_{1:T}}(\ell)]\leq H(L(\pi)+\Reg_T(\pi))$, together with a regret bound of the form $\Reg_T(\pi)\lesssim \sqrt{\E[\Qtrans^{\pi_{1:T}}(\ell)]}+J$, implies $\Reg_T(\pi)\lesssim \sqrt{HL(\pi)}+J$.  Thus, the transition-dependent complexity is at least as adaptive as the first-order complexity, while it can be sharper because it only counts transition uncertainty that affects future cumulative losses.
We defer a detailed explanation and a simple illustrative example to \cref{app:qtrans-discussion}.

\section{Algorithm and technical components}
\label{sec:algorithm}
We are now ready to present our policy optimization algorithm for both the full-information and bandit settings.
The common template is given in
\cref{alg:OFTRL_unknown}, with the feedback-dependent definitions specified in
\cref{fig:def_hatQ_eta_b}.
Both algorithms follow the same policy optimization template and use new $Q$-function estimators and prediction terms to handle unknown transitions and adversarial losses. We describe the key components below.

\subsection{Common policy optimization template}
\paragraph{Optimistic follow-the-regularized-leader.}
Our algorithms rely on the optimistic follow-the-regularized-leader (OFTRL) framework \citep{chiang2012online,rakhlin2013online,steinhardt2014adaptivity}, which incorporates a prediction of the next loss vector into the standard FTRL.
Our policy update is based on the OFTRL policy-optimization scheme of \citet{li2026data}, which was developed for known transitions, and we modify it to handle unknown transitions.

Policy optimization reduces the MDP problem to a separate multi-armed bandit problem at each state.
By the performance-difference lemma \citep{kakade2002approximately}, the regret can be written as
$\Reg_T(\pi)=\E\brk[\big]{\sum_{s} q^{\pi}(s)\sum_{t}\inpr{\pi_t(\cdot\mid s) - \pi(\cdot\mid s), Q^{\pi_t}(s,\,\cdot\,;\ell_t)}}$.
Thus, the $Q$-function plays the role of the loss vector for updating the action distribution at each state. 
Under unknown transitions, we first form the loss prediction $m_t\in[0,1]^{S\times A}$ and then
compute its $Q$-function using an optimistic transition kernel from the
confidence set. The OFTRL update with a convex regularizer $\psi_t$ is defined, for each state $s \in \calS$, by
\begin{align}
    \pi_t(\cdot\mid s)\! =\! \argmin_{\pi(\cdot\mid s)\in\Delta(\calA)} \set[\bigg]{\!\inpr[\bigg]{\pi(\cdot\mid s), \sum_{\tau=1}^{t-1}\prn*{\hat{Q}_\tau(s, \cdot) - B_\tau(s, \cdot)} + Q^{\underP^m_t, \pi_t}(s, \cdot;m_t)} + \psi_t(\pi(\cdot\mid s))\!}.
    \label{eq:opt_pi_OFTRL}
\end{align}
Here, $\hat Q_\tau$ is a $Q$-function estimator, $B_\tau$ is a dilated exploration bonus, and
$Q^{\underP^m_t,\pi_t}(s,\cdot;m_t)$ is the optimistic prediction of $Q^{\pi_t}(s,\cdot;\ell_t)$. 
The transition kernel $\underP^m_t\in\calP_t$ is chosen to minimize $Q^{\tilP,\pi_t}(s,a;m_t)$ over $\tilP \in\calP_t$ for each state-action pair $(s,a)$. Both $\underP^m_t$ and $Q^{\underP^m_t,\pi_t}(s,\cdot;m_t)$ are computable by backward dynamic programming over the layers, since for $s\in\calS_h$, computing this $Q$-function only requires later-layer values and policies.
In our algorithms, we use the log-barrier regularizer
$\psi_t(\pi(\cdot\mid s))=\sum_{a\in\calA}\frac{1}{\eta_t(s,a)}\ln\prn[\big]{\frac{1}{\pi(a\mid s)}}$ ,
where $\eta_t(s,a)>0$ are state-action-wise time-varying, data-dependent learning rates.
In both the full-information and bandit settings, the algorithm uses the same update rule \eqref{eq:opt_pi_OFTRL}, and differs only in the construction of $\hat Q_t$, $B_t$, and $\eta_t(s,a)$.

The optimistic prediction $Q^{\underP^m_t,\pi_t}(s,\cdot;m_t)$ is obtained from a loss prediction sequence $m_t$.
We construct $m_t$ from the realized trajectory, starting with $m_1(s,a)=1/2$ for all $(s,a)$ and setting
\begin{align}
    m_{t+1}(s,a)=\begin{cases}
        (1-\xi)m_t(s,a)+\xi\ell_t(s,a), & \text{if } \I_t(s,a)=1,\\
        m_t(s,a), & \text{if } \I_t(s,a)=0,
    \end{cases}
    \label{def:predictor_sequence}
\end{align}
where $\xi\in(0,1/2)$ is a step size. Since the transition kernel is unknown, we use this predictor even in the full-information setting, where it is useful for deriving path-length bounds.

\paragraph{Dilated exploration bonus.}
In policy optimization, a naive state-wise update of the action distributions can lead to insufficient exploration. To address this, \citet{luo2021policy} introduced a dilated
exploration bonus $B_t(s,a)$, which is constructed in the same form as a $Q$-function,
    \begin{align}
    &B_t(s,a) = b_t(s) + \prn*{1 + \frac{1}{H}}\max_{\tilP\in\calP_t}\E_{s'\sim \tilP(\cdot\mid s,a), a'\sim \pi_t(\cdot\mid s')} \brk*{B_t(s',a')}. \label{def:dilated_bonus}
    \end{align}
Since $B_t$ is subtracted from the $Q$-function estimator in \cref{eq:opt_pi_OFTRL}, a state-action pair $(s,a)$ with a larger dilated exploration bonus is evaluated more optimistically and is therefore more likely to be visited.
This leads to the following lemma.
\begin{lem}[{\citealt[Lemma B.1]{luo2021policy}}]
\label{lem:dilated_bonus_main}
    Suppose that $P\in\calP_t$ for all $t\in[T]$.
    Suppose that $b_t(s)$ is a nonnegative function.
    Suppose also that, for a comparator policy $\pi$, there exists $J^\pi\in\R$ such that
    \begin{align}
        &\E\brk*{\sum_s q^{\pi}(s)\sum_{t=1}^T\sum_a \prn*{\pi_t(a\mid s) - \pi(a\mid s) } \prn*{Q^{\pi_t}(s,a;\ell_t) -  B_t(s,a)}}\\
        &\leq J^{\pi} + \E\brk*{\sum_{t = 1}^T\sum_s q^{\pi}(s) b_t(s)} + \E\brk*{\frac{1}{H} \sum_{t = 1}^T\sum_s\sum_a q^{\pi}(s) \pi_t(a\mid s)B_t(s,a)}.
        \label{eq:key_lemma_ineq_main}
    \end{align}
    Then, it holds that
        $\Reg_T(\pi) \leq J^{\pi} + 3 \, \E\brk[\Big]{\sum_{t = 1}^T V^{\overP^B_t, \pi_t}(s_0;b_t)}$,
    where $\overP^B_t \in \calP_t$ simultaneously attains the maxima in \cref{def:dilated_bonus} for all $(s,a)$.
\end{lem}
\vspace{-2pt}
From this lemma, once the condition \cref{eq:key_lemma_ineq_main} is verified, it remains only to evaluate the value function $V^{\overP^B_t, \pi_t}(s_0;b_t)$ of the bonus under the learner's own policies $\pi_t$, rather than a comparator policy $\pi$. For further intuition on the dilated bonus, we refer the reader to \citet{luo2021policy,dann2023best}.

\begin{figure*}[t]
\begin{algorithm}[H]
\caption{Data-Dependent Policy Optimization with Unknown Transitions} 
\label{alg:OFTRL_unknown}
\nl \textbf{Input:} $\gamma_t = \frac{1}{t}$, $\eta_1 = \frac{1}{12H^3}$, $m_1 = \frac{1}{2}$, $\xi = \frac{1}{4}$, $\psi_t(\pi(\cdot\mid s))=\sum_{a\in\calA}\frac{1}{\eta_t(s,a)}\ln\prn[\big]{\frac{1}{\pi(a\mid s)}}$\\
\nl \For{$t=1,2,\ldots$} {
    \nl Update $\calP_t$ by \cref{def:calP_t}. \\
    \nl For each $s\in \calS$, compute $\pi_t$ by the OFTRL update \cref{eq:opt_pi_OFTRL}. \nllabel{line:optimize_policy}\\
    \nl Set $Y_t\gets 1$ in the full-information setting, and $Y_t\gets \ind\brk[\big]{\max_{s,a}\frac{\eta_t(s,a)}{q_t(s)}\le \frac{1}{50\sqrt{H^3S^3A}}}$ in the bandit setting. \nllabel{line:virtual_episode}\\
    \nl If $Y_t=0$, insert a virtual episode and shift the indices of real episodes.\\
    \nl If $Y_t=1$, execute $\pi_t$, observe $\{(s_{t,h},a_{t,h})\}_{h=0}^{H-1}$, and receive the corresponding loss feedback.\\
    \nl Compute $\hat Q_t$, $\eta_{t+1}$, and $b_t$ as in \cref{fig:def_hatQ_eta_b}, and compute $B_t$ by \cref{def:dilated_bonus}.\\
    \nl Compute loss prediction $m_{t+1}(s,a)$ by \cref{def:predictor_sequence}.
}
\end{algorithm}
\vspace{-20pt}
\end{figure*}

\subsection{Feedback-dependent $Q$-function estimator}
\begin{figure}[tb]
{\relscale{0.85}
\begin{framed}
\vspace{-3pt}
\textbf{Full-information.}\\
\emph{$Q$-function estimator.}
\begin{align}
    \hat{Q}_t(s,a) = Q^{\underP^\ell_t, \pi_t}(s,a;\ell_t) \qquad \text{with}\quad \underP^\ell_t \in \argmin_{\tilP \in \calP_t}Q^{\tilP, \pi_t}(s,a;\ell_t) \text{ for all } (s,a). \label{eq:Q-estimate_full}
\end{align}
\vspace{-5pt}
\emph{Learning rate and exploration bonus.}
\vspace{2pt}
\begin{align}
    \frac{1}{\eta_{t+1}(s,a)} &= \frac{1}{\eta_{t}(s,a)} + \frac{\eta_t(s,a)\zeta_t(s,a)}{\ln(T)}, \label{eq:learning_rate_full}\\
    b_t(s) &=  7\sum_{a}\prn*{\frac{1}{\eta_{t+1}(s,a)} - \frac{1}{\eta_t(s,a)}}\ln (T), \label{eq:bonus_term_full}
\end{align}
\vspace{-8pt}
\begin{align}
    \!\!\!\!\text{where}\quad
    \zeta_t(s, a) = \pi_t(a\mid s)(1 - \pi_t(a \mid s))\sum_{b}\pi_t(b \mid s)\prn*{Q^{\underP^\ell_t, \pi_t}(s, b;\ell_t) - Q^{\underP^m_t, \pi_t}(s, b;m_t)}^2. \label{eq:zeta_full}
\end{align}
\vspace{-12pt}
\end{framed}
\begin{framed}
\vspace{-3pt}
\textbf{Bandit feedback.}\\
\emph{$Q$-function estimator.}
Let $L_{t,h}=\sum_{h'=h}^{H-1}\ell_t(s_{t,h'},a_{t,h'})$, $M_{t,h}=\sum_{h'=h}^{H-1}m_t(s_{t,h'},a_{t,h'})$, and $D_{t,h}=\sum_{h'=h}^{H-1}\abs{\ell_t(s_{t,h'},a_{t,h'})-m_t(s_{t,h'},a_{t,h'})}$, and define
\begin{align}
    \hat Q_t(s,a) = Q^{\underP^m_t,\pi_t}(s,a;m_t)+\frac{\I_t(s,a)\prn*{L_{t,h(s)}-M_{t,h(s)}}}{q_t(s)\pi_t(a\mid s)}Y_t-C_t(s,a)Y_t, \label{eq:Q-estimate_bandit}
\end{align}
\vspace{-7pt}
\begin{align}
    \text{where}\quad
    c_t(s,a) &= \rho_t(s)\frac{\I_t(s,a)D_{t,h(s)}}{q_t(s)\pi_t(a\mid s)}+\rho_t(s)^2H, \label{eq:def_new_c}\\
    C_t(s,a) &= c_t(s,a)+\max_{\tilP\in\calP_t}\E_{s'\sim\tilP(\cdot\mid s,a),\,a'\sim\pi_t(\cdot\mid s')}\brk*{c_t(s',a')+C_t(s',a')}. \label{eq:def_C}
\end{align}
\emph{Learning rate and exploration bonus.}
Let $(s_t^\dagger,a_t^\dagger)$ be any maximizer of $\eta_t(s,a)/q_t(s)$, and define
\begin{align}
    \frac{1}{\eta_{t+1}(s,a)} &=
    \begin{cases}
    \frac{1}{\eta_t(s,a)}+\frac{\eta_t(s,a)\zeta_t(s,a)}{q_t(s)^2\ln(T)}, & \text{if } t \text{ is a real episode}, \\
    \frac{1}{\eta_t(s,a)}\prn[\Big]{1+\frac{\ind\{(s_t^\dagger,a_t^\dagger)=(s,a)\}}{324H\ln(T)}}, & \text{if } t \text{ is a virtual episode},
    \end{cases} \label{eq:eta_update_policy}\\
    b_t(s) &= 6\sum_a\prn*{\frac{1}{\eta_{t+1}(s,a)}-\frac{1}{\eta_t(s,a)}}\ln(T)+3\sum_a\eta_t(s,a)\pi_t(a\mid s)^2C_t(s,a)^2Y_t, \label{eq:bonus_term_bandit}
\end{align}
\vspace{-8pt}
\begin{align}
    \text{where}\quad
    \zeta_t(s,a) &= \prn*{\I_t(s,a)-\pi_t(a\mid s)\I_t(s)}^2\prn*{L_{t,h(s)}-M_{t,h(s)}}^2. \label{eq:zeta_bandit}
\end{align}
\vspace{-12pt}
\end{framed}
}
\vspace{-5pt}
\caption{Definitions of the $Q$-function estimator $\hat Q_t$, the learning-rate update $\eta_{t+1}$, and the local exploration bonus $b_t$ used in \cref{alg:OFTRL_unknown}. }
\label{fig:def_hatQ_eta_b}
\vspace{-10pt}
\end{figure}
The main new component is the $Q$-function estimator.
The estimator differs between the full-information and bandit settings, and this choice also affects the learning-rate update and the local exploration bonus $b_t$. We collect the definitions in \cref{fig:def_hatQ_eta_b} and explain the idea below. 
Note that the quantities $\underP_t^\ell$, $\underP_t^m$, $B_t$, and $C_t$ in \cref{fig:def_hatQ_eta_b} are efficiently computable by backward dynamic programming with a state-action-wise greedy procedure over the confidence set, as in \citet{jin2020learning,luo2021policy}.

\paragraph{Full-information feedback.}
In the full-information setting, the learner observes the entire loss function $\ell_t$ after each episode, so the $Q$-function estimator can be formed directly from the entire loss $\ell_t$. 
We choose $\underP_t^\ell\in\calP_t$ so that it minimizes
$Q^{\tilP,\pi_t}(s,a;\ell_t)$ over $\tilP\in\calP_t$ for all state-action pairs
$(s,a)$, and design the new optimistic $Q$-function estimator by \cref{eq:Q-estimate_full}.
On the event that $P\in\calP_t$ for all $t$, this estimator is optimistic, and
the resulting bias is one-sided. This is the only additional bias caused by unknown transitions, and
it is controlled by the transition-dependent complexity $\Qtrans^{\pi_{1:T}}(\ell)$ in the regret analysis.

The corresponding learning-rate update and local exploration bonus are given in
\cref{eq:learning_rate_full,eq:bonus_term_full}.
The update of the learning rate is designed so that the penalty term
$(1/\eta_{t+1}(s,a) - 1/\eta_t(s,a))\log(T)$ matches the stability term
$\eta_t(s,a)\zeta_t(s,a)$, where $\zeta_t$ is defined in \cref{eq:zeta_full}.
Thus, both terms are controlled through the same data-dependent quantity $\zeta_t$.
The local exploration bonus $b_t$ is directly induced by this learning-rate update, with transition-estimation error included through $\zeta_t$, so no separate transition bonus is needed in the full-information setting.
Since the estimator uses the observed loss directly, the stability condition typically needed when analyzing OFTRL is
easy to satisfy without introducing virtual episodes. Accordingly,
Algorithm~\ref{alg:OFTRL_unknown} sets $Y_t=1$ throughout the full-information setting.

\paragraph{Bandit feedback.}
In the bandit setting, the learner observes losses only along the realized trajectory.
To deal with this, our new $Q$-function estimator \cref{eq:Q-estimate_bandit} therefore combines three ingredients: the optimistic prediction $Q^{\underP^m_t,\pi_t}(s,a;m_t)$, estimator of the prediction-error $Q^{\pi_t}(s,a;\ell_t-m_t)$, and the transition bonus $C_t(s,a)$. 
The main challenge is designing the transition bonus.
It has to account for both loss-estimation error from bandit feedback and transition-estimation error from the unknown transition kernel.
More precisely, the local transition bonus $c_t(s,a)$ must upper bound the estimator bias, while $c_t(s,a)$ itself remains controllable by data-dependent quantities.
In particular, it should satisfy a relation of the form $\E_t[c_t(s,a)] \geq \rho_t(s)\abs{Q^{\pi_t}(s,a;\ell_t - m_t)}$ for all $t,s,a$.
The simpler local transition bonus $c_t(s)=\rho_t(s)H$ used in \citet{dann2023best} controls the estimator bias only through a worst-case $\tilO(\sqrt T)$ bound and thus does not yield the desired data-dependent bound.

To tackle this challenge, we instead define the data-dependent local transition bonus $c_t(s,a)$ in \cref{eq:def_new_c} and its recursive version $C_t(s,a)$ in \cref{eq:def_C}, both constructed from the bandit feedback. The key point in our estimator is that, in the local transition bonus $c_t(s,a)$, the part that does not depend on the prediction error scales as $\rho_t(s)^2$ rather than $\rho_t(s)$.
This makes the worst-case contribution of the estimator bias only $\polylog(T)$, and the other terms can be bounded by data-dependent quantities. 
In addition, we use the prediction term $Q^{\underP^m_t,\pi_t}(s,a;m_t)$, which is not automatically a lower estimate of $Q^{\pi_t}(s,a;\ell_t)$, as also noted by \citet{li2026data}. We therefore define $C_t$ with an extra margin, so that the estimator remains optimistic without losing data-dependent control.

The learning-rate update and local exploration bonus are given in \cref{eq:eta_update_policy,eq:bonus_term_bandit}. As in the full-information
setting, the penalty term of regret $(1/\eta_{t+1}(s,a)-1/\eta_t(s,a))\log(T)$ is chosen to match the log-barrier stability term $\eta_t(s,a)\zeta_t(s,a)/q_t(s)^2$, where $\zeta_t$ is defined in
\cref{eq:zeta_bandit}. The local exploration bonus $b_t$ contains the term induced by this learning-rate update, together with an additional term involving $C_t(s,a)^2$, which accounts for transition-estimation error.

\paragraph{Virtual episodes.}
In the bandit setting, $B_t(s,a)$ can scale as
$1/q_t(s)^2$ through \cref{eq:bonus_term_bandit}, so the resulting bonus-induced stability term cannot be controlled by the adaptive learning-rate update alone when $q_t(s)$ is small. Therefore, following \citet{dann2023best}, we insert virtual episodes to control the ratio $\eta_t(s,a)/q_t(s)$. 
If $\max_{s,a}\eta_t(s,a)/q_t(s)>1/(50\sqrt{H^3S^3A})$ at the beginning of episode $t$, we set $Y_t=0$, make episode $t$ virtual, and shift the indices of all subsequent real episodes. The algorithm then uses the zero loss function and only updates the pair $(s_t^\dagger,a_t^\dagger)\in\argmax_{s,a}\eta_t(s,a)/q_t(s)$, multiplying $1/\eta_t(s_t^\dagger,a_t^\dagger)$ by $1+\frac{1}{324H\ln(T)}$.
It still computes $\hat Q_t$, $b_t$, and $B_t$, but keeps the prediction $m_t$ and the confidence set $\calP_t$ unchanged. 
Since the number of virtual episodes is $\calO(HSA\ln^2(T))$, their effect is lower-order, and we still use $T$ to denote the total number of episodes.\footnote{The data-dependent complexity measures in the main statements are defined only over real episodes.}

\section{Main results}
\label{sec:main_result}
\paragraph{Regret upper bounds.} We now state the guarantees of \cref{alg:OFTRL_unknown}. The first theorem is for the full-information setting, and the second is for the bandit setting. In both cases, the algorithm achieves first-order, second-order, and path-length bounds in the adversarial regime, and at the same time achieves a gap-dependent $\polylog(T)$ bound in the stochastic regime.
The transition-dependent complexity $\Qtrans^{\pi_{1:T}}(\ell)$ captures transition uncertainty that affects future cumulative losses and can imply first-order dependence on $L(\pi)$ through self-bounding.
\begin{thm}\label{thm:main_full}
In the full-information setting, for any comparator policy $\pi$, \cref{alg:OFTRL_unknown} guarantees
\begin{align}
    \Reg_T(\pi)
    &\lesssim \sqrt{H^2SA\ln(T) \min\set*{L(\pi),HT-L(\pi),Q_\infty,V_1}}\\
    &\qquad + \sqrt{S^2A\ln^2(T)\E\brk*{\Qtrans^{\pi_{1:T}}(\ell)}} + \bigl(H^4+HS^2\bigr)A\ln^2(T).
\end{align}
Under the stochastic regime with adversarial corruption, it simultaneously ensures $\Reg_T(\pi) \lesssim U + \sqrt{U\calC} + \bigl(H^4+HS^2\bigr)A\ln^2(T)$,
where $U = \frac{H^2S^2A\ln^2(T)}{\Delta_{\min}}$.
\end{thm}
\vspace{-2pt}
The above bound recovers the $\tilO(\sqrt{H^2S^2AT})$ worst-case rate of occupancy-measure-based algorithms~\citep{rosenberg2019online,jin2021best}. To our knowledge, this is the first policy-optimization result to remove the extra $H$ factor \citep{luo2021policy,dann2023best}.
In the stochastic regime, the leading gap-dependent term also has sharper dependence on $H$, $S$, and $A$ than that of \citet{jin2021best}, up to logarithmic factors.

\begin{thm}\label{thm:main_bandit}
In the bandit setting, for any comparator policy $\pi$, \cref{alg:OFTRL_unknown} guarantees
\begin{align}
    \Reg_T(\pi)
    &\lesssim \sqrt{H^3S^2A\ln^2(T) \min\set*{L(\pi), HT - L(\pi), Q_{\infty}, V_1}}\\
    &\qquad+\sqrt{S^2A\ln^2(T)\E\brk*{\Qtrans^{\pi_{1:T}}(\ell)}} + \bigl(H^3S^2A+H^{\frac{3}{2}}S^3A+H^{\frac{3}{2}}S^{\frac{5}{2}}A^{\frac{3}{2}}\bigr)\ln^2(T).
\end{align}
Under the stochastic regime with adversarial corruption, it simultaneously ensures $\Reg_T(\pi) \lesssim U + \sqrt{U\calC} + \bigl(H^{\frac{3}{2}}S^3A+H^{\frac{3}{2}}S^{\frac{5}{2}}A^{\frac{3}{2}}\bigr)\ln^2(T)$,
where $U = \frac{H^4S^2A\ln^2(T)}{\Delta_{\min}}$.
\end{thm}
\vspace{-2pt}
In the worst case, the bound for the adversarial regime recovers the worst-case rate of \citet{dann2023best}. 
For the first-order bound, the $\Qtrans^{\pi_{1:T}}(\ell)$ term is absorbed by the first-order term through the self-bounding argument. Thus, the final bound does not contain an explicit transition-dependent term.
In the stochastic regime, the leading gap-dependent term matches the $H^4S^2A/\Delta_{\min}$ dependence of \citet{dann2023best}.
However, the main contribution of \cref{thm:main_full,thm:main_bandit} is their data-dependent regret bounds under unknown transitions.
Together, these theorems give the first policy-optimization guarantees under unknown transitions with first-order, second-order, and path-length adaptivity.

\paragraph{Lower bounds.}
We show that the data-dependent terms in our adversarial guarantees cannot be avoided in general.
First, by running a standard hard instance for adversarial MDPs with unknown transitions \citep{jin2018q,domingues2021episodic} for the first $T'\le T$ episodes and setting all later losses to zero, the complexities satisfy $L(\pi)=\calO(HT')$, $Q_\infty=\calO(HT')$, $V_1=\calO(SAT')$, and $\E[\Qtrans^{\pi_{1:T}}(\ell)]=\calO(H^2T')$, while $\max_{\pi}\Reg_T(\pi)\ge\Omega(\sqrt{H^2SAT'})$. Thus, one can obtain $\Omega(\sqrt{HSA L(\pi)})$, $\Omega(\sqrt{HSAQ_\infty})$, $\Omega(\sqrt{H^2 V_1})$ and $\Omega(\sqrt{SA \E\brk*{\Qtrans^{\pi_{1:T}}(\ell)}})$ for the adversarial regime. Closing the remaining gaps is not specific to our data-dependent setting, and remains open even for worst-case policy optimization \citep{luo2021policy}.

The following proposition further shows that a transition-dependent term is necessary under unknown transitions. 
The lower-bound constructions of \citet[Theorem 4]{jin2018q} and \citet[Theorem 9]{domingues2021episodic} use time-invariant loss sequences, so the regret arises solely from transition uncertainty.
\begin{prop}\label{prop:qtrans}
For any algorithm, there exists an episodic MDP with unknown transitions and a time-invariant loss sequence such that $Q_\infty=V_1=0$, and the algorithm suffers $\max_{\pi}\Reg_T(\pi)=\Omega(\sqrt{H^2SAT})$. This lower bound holds even under full-information feedback.
\end{prop}
\vspace{-3pt}

\section{Proof overview}
\label{sec:proof}
We sketch the proof for the bandit setting, which contains the main
ideas also used in the full-information setting. The complete proofs are deferred to \cref{app:full_proof} for the full-information setting and \cref{app:bandit_proofs} for the bandit setting.
We use \cref{lem:dilated_bonus_main} as the main reduction for policy optimization.
Recall that once the inequality in \cref{eq:key_lemma_ineq_main} is verified, the lemma reduces the regret analysis to bounding $\sum_{t=1}^T V^{\overP^B_t,\pi_t}(s_0;b_t)$. We prove
\cref{eq:key_lemma_ineq_main} by decomposing its left-hand side into
$\E\brk[\big]{\sum_s q^\pi(s)(\regterm(s)+\biasterm(s)+\errorterm(s))}$, where
\begin{align}
\\[-10pt]
    \regterm(s) &\!=\! \sum_{t,a}  \prn{\pi_t(a\mid s)\!-\!\pi(a\mid s)} \prn*{\!\hat{Q}_t(s,a) \!-\!  B_t(s,a)\!},\\[-2pt]
    \biasterm(s) &\!=\! \sum_{t,a} \prn{\pi_t(a\mid s)\!-\!\pi(a\mid s)} \prn*{\!Q^{\pi_t}(s,a;\ell_t) \!-\!  \hat{Q}_t(s,a) \!+\! Q^{\underP^m_t, \pi_t}(s, a; m_t) \!-\! Q^{\pi_t}(s, a; m_t)\!},\\[-2pt]
    \errorterm(s) &\!=\! \sum_{t,a} \prn{\pi_t(a\mid s)\!-\!\pi(a\mid s)} \prn*{\!Q^{\pi_t}(s, a; m_t) \!-\! Q^{\underP^m_t, \pi_t}(s, a; m_t)\!}.
\end{align}
Here, $\biasterm(s)$ is the estimator bias, while $\errorterm(s)$ comes from evaluating the prediction using $\underP_t^m$ instead of the true transition kernel.
The standard analysis in \cref{lem:regterm_policy} bounds $\regterm(s)$ as
\begin{align}
\E[\regterm(s)] &\leq \calO\prn*{HS^2A\ln(T)} + \mathbb{E}\brk*{\sum_{t=1}^Tb_t(s)} + \mathbb{E}\brk*{\frac{1}{H}\sum_{t=1}^T\sum_a \pi_t(a \mid s) B_t(s,a)}.
\end{align}
The remaining terms are controlled in \cref{lem:biasterm_bandit,lem:errorterm_bandit} by
\begin{align}
\\[-10pt]
    \E\brk*{\!\sum_s q^\pi(s)\biasterm(s)\!} &\lesssim \tilO\prn*{\!\sqrt{\!H^2S^2A \iota^2 \E\brk*{\sum_{t:Y_t=1} \sum_{s,a}q^{\pi_t}(s,a)\max_{\tilP\in\calP_t} Q^{\tilP, \pi_t}(s,a;\prn{\ell_t - m_t}^2)}\!}},\\[-2pt]
    \E\brk*{\!\sum_s q^\pi(s)\errorterm(s)\!} &\lesssim \tilO\prn*{\sqrt{S^2A\iota^2\E\brk*{\Qtrans^{\pi_{1:T}}(m)}}}.
\end{align}
The key point is that $\biasterm(s)$ is reduced to the prediction error $\ell_t-m_t$. This follows from the design of $c_t(s,a)$, where the part independent of the prediction error scales with $\rho_t(s)^2$ rather than $\rho_t(s)$, and $\sum_{t:Y_t=1}\sum_s q^{\pi_t}(s)\rho_t(s)^2$ is only polylogarithmic in $T$ (see \cref{lem:bound_rho2}).
The error term is captured by $\Qtrans^{\pi_{1:T}}(m)$, which can be reduced to $\Qtrans^{\pi_{1:T}}(\ell)$ through
$\Qtrans^{\pi_{1:T}}(m) \leq 2\Qtrans^{\pi_{1:T}}(\ell) + 2H\sum_{t:Y_t=1} V^{\pi_t}(s_0;(\ell_t-m_t)^2)$
in \cref{lem:qtrans_m_to_ell}.
Thus, all remaining terms reduce to the prediction error $\ell_t-m_t$ and $\Qtrans^{\pi_{1:T}}(\ell)$, yielding the desired first-order, second-order, and path-length bounds.

For the stochastic regime with adversarial corruption, we use the same decomposition as above.
We show in \cref{lem:biasterm_bandit,lem:errorterm_bandit} that the individual terms can be bounded by factors of the form $\sum_s\sum_{a\neq\pi^\star(s)}\sqrt{\E\brk[\big]{\sum_{t:Y_t=1} q^{\pi_t}(s,a)}}$ or $\sqrt{\E\brk[\big]{\sum_{t:Y_t=1}\sum_{s,a}[q^{\pi_t}(s,a)-q^{\pio}(s,a)]_+}}$.
These quantities are then controlled by a self-bounding argument following \citet{dann2023best}, yielding the desired gap-dependent bound under adversarial corruption.
\section{Conclusion}
\label{sec:conclusion}
In this work, we resolved the open problem raised by \citet{dann2023best} on data-dependent regret bounds for policy optimization in MDPs with unknown transitions. Our algorithms achieve refined best-of-both-worlds guarantees, adapting to $L(\pi)$, $Q_\infty$, and $V_1$ in the adversarial regime while attaining a gap-dependent $\polylog(T)$ bound in the stochastic regime with adversarial corruption.
A natural future direction is to obtain variance-aware gap-dependent bounds as in \citet{chen2025sharp}, and to investigate whether the refined data-dependent guarantees can be achieved under aggregate feedback.

\section*{Acknowledgements}
TT is supported by JSPS KAKENHI Grant Number JP26K21297 and was supported by JST BOOST Grant Number JPMJBY25D4 during the final stages of this work. KY is partially supported by JSPS KAKENHI Grant Number JP24H00703.

\clearpage
\bibliographystyle{plainnat}
\bibliography{reference}


\crefalias{section}{appendix}
\crefalias{subsection}{appendix}
\crefalias{subsubsection}{appendix}

\newpage
\appendix
\tableofcontents
\section{Summary of notation}
For the reader's convenience, \cref{tab:notation-summary-main} collects the main notation used throughout the paper.

\begin{table}[H]
\centering
\caption{Summary of notation.}
\label{tab:notation-summary-main}
\resizebox{\textwidth}{!}{
\begin{tabular*}{\textwidth}{ll}
\toprule
\textbf{Symbol} & \textbf{Meaning} \\
\midrule
\textbf{Tabular MDPs} & \\
$\mathcal{M}=(\mathcal{S},\mathcal{A},P,H,s_0)$ & Episodic tabular MDP \\
$\mathcal{S}, S$ & State space and its size $S=|\mathcal{S}|$ \\
$\mathcal{A}, A$ & Action space and its size $A=|\mathcal{A}|$ \\
$P$ & Transition kernel \\
$H$ & Horizon length \\
$T$ & Number of episodes \\
$h(s)$ & Layer index of state $s$ \\
$s_{t,h},\,a_{t,h}$ & State / action at step $h$ in episode $t$ \\
$\ell_t(s,a)$ & Loss assigned to $(s,a)$ in episode $t$ \\
$\pi_t$ & Policy in episode $t$ \\
$\Reg_T(\pi)$ & Regret over $T$ episodes with comparator $\pi$ \\
$\I_t(s,a)$ & Visitation indicator of $(s,a)$ in episode $t$ \\
$n_t(s,a)$ & Number of visits to $(s,a)$ up to episode $t-1$ \\
$V^{\pi}(s;\ell)$ & Value function under policy $\pi$ from state $s$ with loss $\ell$ \\
$Q^{\pi}(s,a;\ell)$ & $Q$-function under policy $\pi$ from $(s,a)$ with loss $\ell$ \\
$q^{\pi}(s),\ q^{\pi}(s,a)$ & Occupancy measure under policy $\pi$ \\
$q^{\pi}(s',a'\mid s,a)$ & Conditional occupancy from $(s,a)$ to $(s',a')$ under $\pi$ \\
$\pio$ & Optimal policy for the regret \\
$\ell'_t(s,a)$ & Uncorrupted i.i.d.~loss  \\
$\calC$ & Corruption budget \\
$\mu(s,a)$ & Mean of $\ell'_t(s,a)$ \\
$\pist$ & Optimal deterministic policy under $\mu$\\
$\Delta(s,a), \Delta_{\min}$ & Suboptimality gap at $(s,a)$, and minimum positive gap \\
\midrule
\textbf{Data-dependent complexity measures} & \\
$L(\pi)$ & First-order complexity \\
$Q_\infty$ & Second-order complexity \\
$V_1$ & Path-length complexity \\
$\Qtrans^{\pi_{1:T}}(\ell)$ & Transition-dependent complexity in \cref{def:qtrans_loss}\\
\midrule
\textbf{Notation for \cref{alg:OFTRL_unknown}} & \\
$\psi_t$ & Regularizer in episode $t$\\
$\eta_t(s,a)$ & Learning rate for $(s,a)$ in episode $t$  \\
$m_t(s,a)$ & Loss prediction for $(s,a)$ \\
$\zeta_t(s,a)$ & Data-dependent term for updating $\eta_t(s,a)$ \\
$\hat Q_t(s,a)$ & $Q$-function estimator \\
$b_t(s), B_t(s,a)$ & Exploaration bonus and its recursive $Q$-function version \\
$c_t(s,a), C_t(s,a)$ & Transition bonus and its recursive $Q$-function version \\
$Y_t\in\{0,1\}$ & Episode indicator ($Y_t=1$ real, $Y_t=0$ virtual)  \\
$\mathcal{T}_r,\ \mathcal{T}_v$ & Sets of real and virtual episodes \\
$\overq^{\pi}_t(s), \underq^{\pi}_t(s)$ & Upper and lower occupancy measures \\
$\gamma_t, q_t(s)\coloneqq\overq^{\pi_t}_t(s)+\gamma_t$ & Exploration rate and smoothed state occupancy  \\
$\rho_t(s)\coloneqq (q_t(s) - \underq_t^{\pi_t}(s))/q_t(s)$ & relative occupancy width  \\
$L_{t,h(s)}, M_{t,h(s)}$ &  Realized / predicted suffix loss from layer $h(s)$  \\
$D_{t,h(s)}$ & Suffix prediction error from layer $h(s)$\\
$\calP_t$ & Confidence set at episode $t$ \\
$\underP^\ell_t$ & Transition in $\calP_t$ minimizing $Q(s,a;\ell_t)$ for all $(s,a)$ \\
$\underP^m_t$ & Transition in $\calP_t$ minimizing $Q(s,a;m_t)$ for all $(s,a)$ \\
$\overP^B_t$ & Transition in $\calP_t$ maximizing $B_t(s,a)$ for all $(s,a)$ \\
$\overP^C_t$ & Transition in $\calP_t$ maximizing $C_t(s,a)$ for all $(s,a)$ \\
\bottomrule
\end{tabular*}
}
\end{table}

We formalize the conditional occupancy measure $q^{\pi_t}(s',a'\mid s,a)$ as follows:
\begin{align}
\begin{cases}
0 & \text{if } h(s') < h(s),\\
0 & \text{if } h(s') = h(s) \text{ and } (s',a') \neq (s,a),\\
1 & \text{if } (s',a') = (s,a),\\
\Pr\set*{(s_{t,h(s')},a_{t,h(s')})=(s',a') \mid (s_{t,h(s)},a_{t,h(s)})=(s,a)},
& \text{if } h(s') > h(s).
\end{cases}
\end{align}

For $\underP^\ell_t$, $\underP^m_t$, $\overP^B_t$, and $\overP^C_t$, we choose kernels in $\calP_t$ that attain the required state-action-wise minimum or maximum for all $(s,a)$ simultaneously.
This is possible since $\calP_t$ is rectangular, and these kernels are efficiently computable by backward dynamic programming with a state-action-wise greedy procedure over the confidence set, as in \citet[Algorithm~3]{jin2020learning} and \citet[Algorithms~4--6]{luo2021policy}.
\section{Additional related work and discussion}
\subsection{Additional related work}
\label{app:add_relate}
\paragraph{Tabular MDPs.}
Adversarial MDPs have been studied since the works of \citet{even2009online,yu2009Markov}, and the episodic formulation was later introduced by \citet{zimin2013online}.
Algorithms for online tabular MDPs can be broadly divided into two approaches: occupancy-measure based \emph{global optimization} and \emph{policy optimization}. 
Global optimization solves an optimization problem over the set of occupancy measures, and has led to strong regret guarantees in several settings~\citep{zimin2013online,jin2020learning}.
However, global optimization can be computationally demanding because it optimizes over the occupancy measure space of the MDP. Policy optimization takes a more local approach by optimizing an action distribution at each state, which is often practical and computationally efficient~\citep{shani2020optimistic,luo2021policy}.

Under known transitions and bandit feedback, global optimization achieves the minimax-optimal regret rate $\tilO(\sqrt{HSAT})$ \citep{zimin2013online}.
In contrast, the best known bound for policy optimization is $\tilO(\sqrt{H^3SAT})$, achieved via dilated exploration \citep{luo2021policy}.
Compared with the global optimization rate, this bound incurs an extra factor of $H$, often attributed to $Q$-function estimation.

Unknown transitions introduce an additional difficulty, since the learner must also estimate the transition kernel.
Among global optimization, the best known regret bounds were obtained by \citet{rosenberg2019online} for full-information feedback and \citet{jin2020learning} for bandit feedback, both of order $\tilO(\sqrt{H^2S^2AT})$.
For policy optimization under bandit feedback, the best known guarantee was established by \citet{luo2021policy}, with regret $\tilO(\sqrt{H^4S^2AT})$.
On the other hand, minimax lower bounds of order $\Omega(\sqrt{H^2SAT})$ were established by \citet{jin2018q,domingues2021episodic}, leaving gaps in the dependence on the horizon and the number of states.
A line of work has sought to close these gaps, and \citet{tiapkin2025narrowing} proposed a policy-optimization algorithm under full-information feedback with regret $\tilO(\sqrt{H^6SAT})$.
More recent work has studied policy optimization under more challenging feedback models, including aggregate feedback \citep{lancewicki2025near,ito2025adapting}.

In the stochastic setting, near-optimal regret guarantees have been established by both model-based and value-based methods \citep{jaksch2010near,azar2017minimax,jin2018q,zanette2019tighter}. Beyond such gap-independent guarantees, a line of work has studied gap-dependent $\polylog(T)$ regret bounds for episodic tabular MDPs \citep{simchowitz2019non,dann2021beyond,xu2021fine}.

\paragraph{Best-of-both-worlds algorithms.}
The best-of-both-worlds guarantee aims to obtain near-optimal regret in both adversarial and stochastic regimes using a single algorithm.
This idea was first studied for multi-armed bandits by \citet{bubeck2012best}, and was later developed in several directions~\citep{seldin2014one,auer2016algorithm,seldin2017improved}.
A common approach is to use follow-the-regularized-leader with carefully chosen regularizers, which allows the algorithm to remain robust in the adversarial regime while improving automatically in the stochastic regime \citep{wei2018more,zimmert2021tsallis,ito2021parameter}.
The stochastic-regime analysis is often based on a self-bounding technique \citep{zimmert2021tsallis,masoudian2021improved}, and the same idea is useful for handling stochastic losses with adversarial corruption.

Best-of-both-worlds guarantees have also been studied in tabular MDPs.
For global optimization, \citet{jin2020simultaneously,jin2021best} first developed best-of-both-worlds algorithms, with subsequent extensions to adversarial transitions and aggregate feedback \citep{jin2023no,ito2025adapting}.
In policy optimization, \citet{dann2023best} established best-of-both-worlds guarantees under bandit feedback for several regularizers, including Tsallis entropy, Shannon entropy, and log-barrier regularization, while \citet{li2026data} later obtained refined data-dependent guarantees using the OFTRL framework under known transitions.

\paragraph{Data-dependent bounds.}
Data-dependent regret bounds adapt to the realized loss sequence instead of depending only on worst-case parameters.
They have been widely studied in online learning, including learning with expert advice \citep{littlestone1994weighted}, multi-armed bandits \citep{auer2002nonstochastic}, and online convex optimization \citep{zinkevich2003online}.
Typical examples in adversarial regimes include first-order bounds based on the comparator's cumulative loss~\citep{cesa1996worst,allenberg2006hannan,neu2015first}, second-order bounds based on loss fluctuations~\citep{hazan2011better,wei2018more,ito2021parameter}, and path-length bounds based on temporal variation~\citep{wei2018more,bubeck2019improved,ito2021parameter}.

In tabular MDPs, it is also important to obtain regret bounds that adapt to data-dependent complexity measures.
The first such result was obtained by \citet{lee2020bias}, who derived a first-order bound for MDPs with unknown transitions via global optimization.
For policy optimization, \citet{dann2023best} established best-of-both-worlds guarantees under unknown transitions and first-order data-dependent guarantees under known transitions, leaving open whether data-dependent guarantees are possible under unknown transitions.
More recently, \citet{li2026data} showed that, under known transitions, OFTRL-based policy optimization can adapt to multiple complexity measures, including first-order, second-order, and path-length measures.

A related line of work has studied gap-dependent bounds in stochastic regimes. For multi-armed bandits, \citet{audibert2007tuning} introduced variance-aware gap-dependent bounds, and related ideas have been incorporated into best-of-both-worlds analyses (\textit{e.g.,}~\citealt{ito2022adversarially}).
For tabular MDPs, recent works have also studied variance-aware gap-dependent $\polylog(T)$ regret bounds~\citep{simchowitz2019non,dann2021beyond,chen2025sharp,li2026data}.

These studies left open the case of data-dependent policy optimization under unknown transitions, especially when combined with best-of-both-worlds guarantees.
Our work addresses these gaps by showing that policy optimization can achieve refined data-dependent and best-of-both-worlds guarantees under unknown transitions, up to an additional transition-dependent complexity term.

\subsection{Discussion of the transition-dependent complexity}
\label{app:qtrans-discussion}
In this part, we discuss the role of the transition-dependent complexity $\Qtrans^{\pi_{1:T}}(\ell)$, which captures the difficulty caused by unknown transitions. 
By \cref{lem:Qtrans_bound}, we have
\begin{align}
    \Qtrans^{\pi_{1:T}}(\ell)
    \leq
    H\sum_{t=1}^T\sum_{s,a}q^{\pi_t}(s,a)\ell_t(s,a)^2. \label{eq:qtrans_example}   
\end{align}
Since $\ell_t(s,a)\in[0,1]$, taking expectations gives
  \begin{align}
\E\brk*{\Qtrans^{\pi_{1:T}}(\ell)}
&\leq
H\E\brk*{\sum_{t=1}^T\sum_{s,a}q^{\pi_t}(s,a)\ell_t(s,a)}
=
H\prn*{L(\pi)+\Reg_T(\pi)}
\label{eq:qtrans_first_order_relation}
\end{align} 
for each policy $\pi$.
From this inequality, we can show that the transition-dependent bound recovers a regret bound depending on the first-order complexity.
Suppose that for the any comparator policy $\pi$, there exists an absolute constant $c>0$ such that
\begin{align}
    \Reg_T(\pi)
    \leq
    c\sqrt{S^2A\ln^2(T)\E\brk*{\Qtrans^{\pi_{1:T}}(\ell)}}+J
\end{align}
for some $J > 0$, as the regret bounds in \cref{thm:main_full,thm:main_bandit}. Then, using the inequality \cref{eq:qtrans_first_order_relation}, the regret is upper bounded as
\begin{align}
    \Reg_T(\pi)
    &\leq
    c\sqrt{HS^2A\ln^2(T)\prn*{L(\pi)+\Reg_T(\pi)}}+J\\
    &\leq c\sqrt{HS^2A\ln^2(T)L(\pi)}+c\sqrt{HS^2A\ln^2(T)\Reg_T(\pi)}+J\\
    &\leq
    c\sqrt{HS^2A\ln^2(T)L(\pi)}+\frac{1}{2}\Reg_T(\pi)+\frac{1}{2}c^2HS^2A\ln^2(T)
    +J
    \label{eq:adv_bandit_first_order_with_qtrans}
\end{align}
where the last inequality follows from the AM--GM inequality.
Thus, we obtain
\begin{align}
    \Reg_T(\pi)
    &\lesssim \sqrt{HS^2A\ln^2(T)L(\pi)}+
    HS^2A\ln^2(T)+J.
\end{align}
This implies that $\Qtrans^{\pi_{1:T}}(\ell)$ is at least as adaptive as the first-order complexity.  However, the comparison in \cref{eq:qtrans_example} can be loose.  Transition uncertainty contributes to $\Qtrans^{\pi_{1:T}}(\ell)$ only when it creates variation in the future cumulative loss. 

The following example shows that $\Qtrans^{\pi_{1:T}}(\ell)$ can be small even when the cumulative loss is large.
Consider an $H=3$ layered MDP with layers $\mathcal{S}_0=\{s_0\}$, $\mathcal{S}_1=\{x_1,x_2\}$, and $\mathcal{S}_2=\{y\}$, followed by the terminal state. 
For each action $a$, the transition $P(\cdot\mid s_0,a)$ may be stochastic and unknown over $\{x_1,x_2\}$. 
For every action, the transitions from $x_1$ and $x_2$ go deterministically to $y$, and the transition from $y$ goes deterministically to the terminal state.

Let the losses at $s_0$, $x_1$ and $x_2$ be zero, and let the loss at $y$ be one:
\begin{align}
    \ell_t(s_0,a)=0, 
    \qquad
    \ell_t(x_i,a)=0 \quad (i\in\{1,2\}), 
    \qquad 
    \ell_t(y,a)=1,
    \qquad \forall a\in\mathcal{A},\ \forall t \in [T].
\end{align}
Since every trajectory passes through $y$ and the loss at $y$ is one, the cumulative loss is linear in $T$, and in particular
\begin{align}
    \sum_{t=1}^T\sum_{s,a}q^{\pi_t}(s,a)\ell_t(s,a)^2 = T.
\end{align}
However,   the actual losses satisfy
\begin{align}
    V^{\pi_t}(x_1;\ell_t)=V^{\pi_t}(x_2;\ell_t).
\end{align} 
Hence, for every action $a$,
  \begin{align}
    \Var_{s'\sim P(\cdot\mid s_0,a)}
    \left[
        V^{\pi_t}(s';\ell_t)
    \right]
    =0.
\end{align} 
The other transitions are deterministic, so no transition contributes to $\Qtrans^{\pi_{1:T}}(\ell)$, and thus $\Qtrans^{\pi_{1:T}}(\ell)=0$.

In this example, the right-hand side of \cref{eq:qtrans_example} is of order $HT$, whereas $\Qtrans^{\pi_{1:T}}(\ell)=0$.
By adding deterministic layers after $y$, the same construction gives an $\calO(H^2T)$ term on the right-hand side while keeping $\Qtrans^{\pi_{1:T}}(\ell)=0$. 
This shows that $\Qtrans^{\pi_{1:T}}(\ell)$ can be much smaller than the upper bound in \cref{eq:qtrans_example,eq:qtrans_first_order_relation}.

\section{High probability event}
We define the following good event $\calE=\calE_1 \cap \calE_2$, which holds with high probability as shown in \cref{dfn:good_event}.

\begin{align}
\calE_1 &= \set*{P\in\calP_t,\ \forall t=1,\dots,T}, \\
\calE_2 &= \set*{\sum_{t=1}^{T}\sum_{(s,a)\in\calS_h\times\calA}
\frac{q^{\pi_t}(s,a)}{n_t(s,a)}
\lesssim \abs{\calS_h}A\ln T+\ln(H/\delta),\ \forall h = 0, \dots, H-1}. \\
\end{align}

\begin{lem}[{Event $\calE_1$, \citealt[Lemma 2]{jin2020learning}}]
    With probability at least $1 - 4\delta$, we have $P \in \calP_t$ for all $t = 1, \dots, T$.
\end{lem}
\begin{lem}[{\citealt[Corollary D.3.4]{jin2021best}}]\label{lem:tildeP-P} 
If $P\in \calP_t$, then for all $\tilP \in \calP_t$
\begin{align}
\abs*{\tilP(s'\mid s, a) - P(s'\mid s,a)} \leq \min\set*{4\sqrt{\frac{P(s'\mid s, a)\iota}{n_t(s,a)}} + \frac{40\iota}{3n_t(s,a)}, 1}
\end{align}
\end{lem}

\begin{lem}[{Event $\calE_2$, \citealt[Lemma 10]{jin2020learning}}]\label{lem:occup_bound}
We have with probability at least $1 - \delta$
\begin{align}
    \sum_{t=1}^{T}\sum_{(s,a) \in \calS_h \times \calA} \frac{q^{\pi_t}(s,a)}{n_t(s,a)} &\lesssim \abs{\calS_h}A\ln T + \ln(H/\delta)
\end{align}
for all layers $h$.
\end{lem}

\begin{dfn}\label{dfn:good_event}
    Define $\calE$ to be the event that $P \in \calP_t$ for all $t$ and the bound in \cref{lem:occup_bound} holds. In this case, $\Pr(\calE) \geq 1 - 5\delta$.
\end{dfn}

\section{Auxiliary lemmas}
\subsection{Technical lemmas}
\begin{lem}[Performance difference lemma]
\label{lem:pdl}

 For any policies $\pi_1$ and $\pi_2$, and any loss function $\ell:S\times A \to \R$,
    \begin{align}
        V^{\pi_1}(s_0; \ell) - V^{\pi_2}(s_0; \ell)
        = \sum_s q^{\pi_2}(s)\sum_a \prn*{\pi_1(a \mid s) - \pi_2(a \mid s)}Q^{\pi_1}(s,a;\ell).
    \end{align}
\end{lem}
\begin{lem}[{\citealt[Lemma C.2]{dann2023best}}]
\label{lem:pdl_with_L}
    For any policies $\pi_1$ and $\pi_2$, and any function $L:S\times A \to \R$, define
    \begin{align}
        \ell(s,a)
        = L(s,a) - \E_{s'\sim P(\cdot \mid s, a),\, a'\sim \pi_1(\cdot \mid s')}\brk*{L(s',a')}.
    \end{align}
    Then, for all $(s,a)$, $Q^{\pi_1}(s,a;\ell) = L(s,a)$.
    Consequently,
    \begin{align}
        \sum_s q^{\pi_2}(s)\sum_a \prn*{\pi_1(a \mid s) - \pi_2(a \mid s)}L(s,a)
        = V^{\pi_1}(s_0; \ell) - V^{\pi_2}(s_0; \ell).
    \end{align}
\end{lem}

\begin{lem}[{Occupancy measure difference, \citealp[Lemma D.3.1]{jin2021best}}]\label{lem:occ_diff_lem} 
For any transition functions $P_1,P_2$ and any policy $\pi$,
\begin{align}
q^{P_1,\pi}(s) - q^{P_2,\pi}(s) &= \sum_{h=0}^{h(s)-1}\sum_{u\in \calS_h}\sum_{v\in\calA}\sum_{w\in \calS_{h+1}}q^{P_1,\pi}(u,v)\prn*{P_1(w\mid u,v)-P_2(w\mid u,v) } q^{P_2,\pi}(s\mid w) \\
& = \sum_{h=0}^{h(s)-1}\sum_{u\in \calS_h}\sum_{v\in\calA}\sum_{w\in \calS_{h+1}}q^{P_2,\pi}(u,v)\prn*{P_1(w\mid u,v)-P_2(w\mid u,v) } q^{P_1,\pi}(s\mid w).
\end{align}
\end{lem}

\begin{lem}\label{lem:concentration_var}
Suppose that the event $\calE$ holds. Let $\tilP_t\in\calP_t$ and $F : \calS_{h+1} \to [-\overline{F},\overline{F}]$. Then,
\begin{align}
    \abs*{\sum_{s' \in S_{h + 1}}\prn*{\tilP_t(s'\mid s, a) - P(s' \mid s, a)} F(s')} 
    &\lesssim \sqrt{\frac{\abs*{\calS_{h+1}}\Var_{s' \sim P(\cdot \mid s, a)}\brk*{F(s')}\iota}{n_t(s,a)}}  +\frac{\abs{\calS_{h+1}}\overline{F}\iota}{n_t(s,a)} 
\end{align}    
for all state-action pairs $(s,a) \in \calS_h \times \calA$.
\end{lem}
\begin{proof}
Since $\sum_{s' \in S_{h + 1}}\prn*{\tilP_t(s'\mid s, a) - P(s' \mid s, a)} = 0$, we may subtract any constant from $F$. Let
\begin{align}
    c = \E_{s' \sim P(\cdot \mid s, a)}\brk*{F(s')}.
\end{align}
Then, 
    \begin{align}
        \sum_{s' \in S_{h + 1}}\prn*{\tilP_t(s'\mid s, a) - P(s' \mid s, a)} F(s') 
        = \sum_{s' \in S_{h + 1}}\prn*{\tilP_t(s'\mid s, a) - P(s' \mid s, a)} \prn*{F(s') - c}
    \end{align}
Therefore, 
    \begin{align}
        &\abs*{\sum_{s' \in S_{h + 1}}\prn*{\tilP_t(s'\mid s, a) - P(s' \mid s, a)} F(s')} \\
        &\leq \sum_{s' \in S_{h + 1}}\abs*{\tilP_t(s'\mid s, a) - P(s' \mid s, a)} \abs*{F(s') - c}\\
        &\lesssim \sum_{s' \in S_{h + 1}}\sqrt{\frac{P(s'\mid s, a)\iota}{n_t(s,a)}} \abs*{F(s') - c} 
        + \sum_{s' \in S_{h + 1}}\frac{\iota}{n_t(s,a)} \abs*{F(s') - c}\tag{by \cref{lem:tildeP-P} and the assumption that $\calE$ holds}\\
        &\leq 
        \sqrt{\frac{\abs*{\calS_{h+1}}\sum_{s' \in S_{h + 1}} P(s'\mid s, a)\prn*{F(s') - c}^2 \iota}{n_t(s,a)}} +\frac{2\abs{\calS_{h+1}}\overline{F}\iota}{n_t(s,a)} \tag{by the Cauchy--Schwarz inequality and $\abs*{F(s')-c} \le 2\overline{F}$}\\
        &\leq 
        \sqrt{\frac{\abs*{\calS_{h+1}}\Var_{s' \sim P(\cdot \mid s, a)}\brk*{F(s')}\iota}{n_t(s,a)}}  +\frac{2\abs{\calS_{h+1}}\overline{F}\iota}{n_t(s,a)}.
    \end{align} 
\end{proof}

\begin{lem}\label{lem:complicated_lemma_var} \label{lem:signed_transition}
 Suppose that the event $\calE$ holds. Let $\tilP_t$ be a transition kernel in $\calP_t$, and let $g_t(s,a)\in[-G,G]$. Then, it holds that

\begin{align}
\left|\sum_t\bigl(V^{\pi_t}(s_0;g_t)-V^{\tilP_t,\pi_t}(s_0;g_t)\bigr)\right|
\lesssim\sqrt{S^2A\iota^2\Qtrans^{\pi_{1:T}}(g)}
 +HS^2AG\iota^2.
\end{align}

\end{lem}

\begin{proof}
For each layer $ h \in \{0, \dots, H-1\} $, let $ W_h \coloneqq \mathcal{S}_h \times \mathcal{A} \times \mathcal{S}_{h+1} $. Throughout the proof, whenever summation indices are omitted, expressions such as $ \sum_{u,v,w} $ or $ \sum_{x,y,z} $ stand for $ \sum_{h=0}^{H-1} \sum_{(s,a,s') \in W_h} $.

According to \cref{lem:occ_diff_lem}, we have
\begin{align}
&V^{\pi_t}(s_0;g_t)-V^{\tilP_t,\pi_t}(s_0;g_t) \\
&= \sum_{s,a}\prn*{q^{\pi_t}(s) - q^{\tilP_t,\pi_t}(s)}\pi_t(a\mid s)g_t(s,a)\\
&=  \sum_{s,a}\sum_{u,v,w} q^{\pi_t}(u,v)\prn*{ P(w\mid u,v) - \tilP_t(w\mid u,v) } q^{\tilP_t,\pi_t}(s\mid w)\pi_t(a\mid s)g_t(s,a) \\
&= \sum_{s,a}\sum_{u,v,w} q^{\pi_t}(u,v)\prn*{ P(w\mid u,v) - \tilP_t(w\mid u,v) } q^{\pi_t}(s\mid w)\pi_t(a\mid s)g_t(s,a)\\
&\qquad + \sum_{s}\sum_a\sum_{u,v,w} q^{\pi_t}(u,v)\prn*{ P(w\mid u,v) - \tilP_t(w\mid u,v) } \prn*{q^{\tilP_t, \pi_t}(s \mid w) - q^{\pi_t}(s\mid w)}\pi_t(a\mid s)g_t(s,a)\\
&=  \underbrace{\sum_{s,a}\sum_{u,v,w} q^{\pi_t}(u,v)\prn*{ P(w\mid u,v) - \tilP_t(w\mid u,v) } q^{\pi_t}(s\mid w)\pi_t(a\mid s)g_t(s,a)}_{\term_1} \\
&\qquad + \underbrace{\sum_{u,v,w}q^{\pi_t}(u,v)\prn*{P(w\mid u,v)-\tilP_t(w\mid u,v)}
 \prn*{V^{\tilP_t,\pi_t}(w;g_t)-V^{\pi_t}(w;g_t)}}_{\term_2}.
\end{align}
We first bound $\term_1$.
\begin{align}
\term_1
&\leq \sum_{u,v,w} q^{\pi_t}(u,v)\prn*{ P(w\mid u,v) - \tilP_t(w\mid u,v) } \sum_{s,a} q^{\pi_t}(s,a\mid w)g_t(s,a)\\
&= \sum_{u,v,w} q^{\pi_t}(u,v)\prn*{ P(w\mid u,v) - \tilP_t(w\mid u,v) } V^{\pi_t}(w; g_t)\\
&\leq \sum_{u,v} q^{\pi_t}(u,v)\abs*{\sum_w\prn*{ P(w\mid u,v) - \tilP_t(w\mid u,v) } V^{\pi_t}(w; g_t)}\\
&\lesssim \sum_{h=0}^{H-1}\sum_{(u,v) \in \calS_h \times \calA} q^{\pi_t}(u,v)\prn*{\sqrt{\frac{\abs*{\calS_{h+1}}\Var_{w \sim P(\cdot \mid u, v)}\brk*{V^{\pi_t}(w; g_t)}\iota}{n_t(u, v)}}  +\frac{\abs{\calS_{h+1}}HG\iota}{n_t(u, v)} } \tag{by \cref{lem:concentration_var}}\\
&= \underbrace{\sum_{h=0}^{H-1}\sum_{(u,v) \in \calS_h \times \calA} q^{\pi_t}(u,v)\prn*{\sqrt{\frac{\abs*{\calS_{h+1}}\Var_{w \sim P(\cdot \mid u, v)}\brk*{V^{\pi_t}(w; g_t)}\iota}{n_t(u, v)}}}}_{\term_{1a}}\\
&\qquad +\underbrace{\sum_{h=0}^{H-1}\abs{\calS_{h+1}}HG\iota\sum_{(u,v) \in \calS_h \times \calA}\frac{q^{\pi_t}(u,v)}{n_t(u, v)}}_{\term_{1b}}\\
\end{align}
For $\term_{1a}$,
\begin{align}
    \sum_{t = 1}^T\term_{1a} &\leq \sum_{t = 1}^T\sum_{h=0}^{H-1}\sum_{(u,v) \in \calS_h \times \calA} q^{\pi_t}(u,v)\prn*{\sqrt{\frac{\abs*{\calS_{h+1}}\Var_{w \sim P(\cdot \mid u, v)}\brk*{V^{\pi_t}(w; g_t)}\iota}{n_t(u, v)}}}\\
    &\leq \sqrt{\sum_{t = 1}^T\sum_{h=0}^{H-1}\sum_{(u,v) \in \calS_h \times \calA}\abs*{\calS_{h+1}}\frac{q^{\pi_t}(u,v)\iota}{n_t(u,v)}}\\
    &\qquad \times \sqrt{\sum_{t = 1}^T\sum_{h=0}^{H-1}\sum_{(u,v) \in \calS_h \times \calA}q^{\pi_t}(u,v)\Var_{w \sim P(\cdot \mid u,v)}\brk*{V^{\pi_t}(w; g_t)}}
    \tag{by the Cauchy--Schwarz inequality}\\
    &\lesssim \sqrt{\prn*{\sum_{h=0}^{H-1}\abs*{\calS_{h+1}}\abs{\calS_h}A\iota^2}\prn*{\sum_{t = 1}^T\sum_{u,v} q^{\pi_t}(u,v)\Var_{w \sim P(\cdot \mid u,v)}\brk*{V^{\pi_t}(w; g_t)}}} \tag{by \cref{lem:occup_bound} and the assumption that $\calE$ holds}\\
    &\leq \sqrt{S^2A\iota^2\prn*{\sum_{t = 1}^T\sum_{u,v} q^{\pi_t}(u,v)\Var_{w \sim P(\cdot \mid u,v)}\brk*{V^{\pi_t}(w; g_t)}}}, \label{eq:var_term1a}
\end{align}
where the last inequality uses $\sum_{h=0}^{H-1}\abs*{\calS_{h+1}}\abs{\calS_h} \leq S^2$.

For $\term_{1b}$,
 \begin{align}
    \sum_{t = 1}^T\term_{1b}
    &\leq \sum_{t = 1}^T\sum_{h=0}^{H-1}\abs{\calS_{h+1}}HG\iota\sum_{(u,v) \in \calS_h \times \calA}\frac{q^{\pi_t}(u,v)}{n_t(u, v)}\\
    &\lesssim \sum_{h=0}^{H-1}\abs{\calS_h}\abs{\calS_{h+1}}HAG\iota^2 \tag{by \cref{lem:occup_bound} and the assumption that $\calE$ holds}\\
    &\leq HS^2AG\iota^2. \label{eq:var_term1b}
\end{align}
Thus, combining \cref{eq:var_term1a,eq:var_term1b}, we obtain
\begin{align}
    \sum_{t = 1}^T\term_{1} \lesssim \sqrt{S^2A\iota^2\prn*{\sum_{t = 1}^T\sum_{u,v} q^{\pi_t}(u,v)\Var_{w \sim P(\cdot \mid u,v)}\brk*{V^{\pi_t}(w; g_t)}}} + HS^2AG\iota^2.
\end{align}

We next bound $\term_2$. First, for $(x,y)\in\calS_j\times\calA$, we have
\begin{align}
 \sum_z\abs*{P(z\mid x,y)-\tilP_t(z\mid x,y)}
 &\leq4\sqrt{\frac{\iota}{n_t(x,y)}}\sum_z\sqrt{P(z\mid x,y)}
 +\frac{40|\calS_{j+1}|\iota}{3n_t(x,y)}
 \tag{by \cref{lem:tildeP-P}}\\
 &\leq4\sqrt{\frac{|\calS_{j+1}|\iota}{n_t(x,y)}}
 +\frac{40|\calS_{j+1}|\iota}{3n_t(x,y)}. \tag{by the Cauchy--Schwarz inequality}
\end{align}
Combining this with $\sum_z\abs*{P(z\mid x,y)-\tilP_t(z\mid x,y)}\leq\sum_z\prn*{P(z\mid x,y)+\tilP_t(z\mid x,y)}=2$ gives
\begin{align}
 \sum_z\abs*{P(z\mid x,y)-\tilP_t(z\mid x,y)}
 &\leq\min\set*{2,\;4\sqrt{\frac{|\calS_{j+1}|\iota}{n_t(x,y)}}+\frac{40|\calS_{j+1}|\iota}{3n_t(x,y)}}\\
 &\leq4\sqrt{\frac{|\calS_{j+1}|\iota}{n_t(x,y)}}
 +\sqrt{\frac{80|\calS_{j+1}|\iota}{3n_t(x,y)}}
 \lesssim\sqrt{\frac{|\calS_{j+1}|\iota}{n_t(x,y)}}, \label{eq:transition_row_l1}
\end{align}
where we used $\min\{2,a+b\}\leq a+\min\{2,b\}$ and $\min\{2,b\}\leq\sqrt{2b}$ for $a,b\geq0$.
Next, for $w\in\calS_{h+1}$, \cref{lem:occ_diff_lem} gives
\begin{align}
 &\abs*{V^{\tilP_t,\pi_t}(w;g_t)-V^{\pi_t}(w;g_t)}\\
 &\leq\sum_{j=h+1}^{H-1}\sum_{(x,y)\in\calS_j\times\calA}
 q^{\pi_t}(x,y\mid w)\sum_z\abs*{P(z\mid x,y)-\tilP_t(z\mid x,y)}
 \abs*{V^{\tilP_t,\pi_t}(z;g_t)}\\
 &\leq HG\sum_{j=h+1}^{H-1}\sum_{(x,y)\in\calS_j\times\calA}q^{\pi_t}(x,y\mid w)
 \sum_z\abs*{P(z\mid x,y)-\tilP_t(z\mid x,y)}\\
 &\lesssim HG\sum_{j=h+1}^{H-1}\sum_{(x,y)\in\calS_j\times\calA}q^{\pi_t}(x,y\mid w)
 \sqrt{\frac{|\calS_{j+1}|\iota}{n_t(x,y)}}.
\end{align}  
Using this bound, we obtain
\begin{align}
 |\term_2|
 &\leq\sum_{u,v,w}q^{\pi_t}(u,v)\abs*{P(w\mid u,v)-\tilP_t(w\mid u,v)}
 \abs*{V^{\tilP_t,\pi_t}(w;g_t)-V^{\pi_t}(w;g_t)}\\
 &\lesssim\sum_{u,v,w}q^{\pi_t}(u,v)
 \prn*{\sqrt{\frac{P(w\mid u,v)\iota}{n_t(u,v)}}+\frac{\iota}{n_t(u,v)}}\abs*{V^{\tilP_t,\pi_t}(w;g_t)-V^{\pi_t}(w;g_t)}
 \tag{by \cref{lem:tildeP-P}}\\
 &\lesssim HG\underbrace{\sum_{0\leq h<j\leq H-1}\sum_{(u,v,w)\in W_h}
 \sum_{(x,y)\in\calS_j\times\calA}q^{\pi_t}(u,v)
 \sqrt{\frac{P(w\mid u,v)\iota}{n_t(u,v)}}q^{\pi_t}(x,y\mid w)
 \sqrt{\frac{|\calS_{j+1}|\iota}{n_t(x,y)}}}_{\term_{2a}}\\
 &\qquad+HG\underbrace{\sum_{h=0}^{H-1}|\calS_{h+1}|
 \sum_{(u,v)\in\calS_h\times\calA}\frac{q^{\pi_t}(u,v)\iota}{n_t(u,v)}}_{\term_{2b}}.
\end{align}
  For $\term_{2a}$,     applying the Cauchy--Schwarz inequality for each pair $h<j$ before summing over the layers gives
\begin{align}
 \sum_{t=1}^T\term_{2a}
 &\leq\sum_{0\leq h<j\leq H-1}
 \sqrt{\sum_{t=1}^T\sum_{\substack{(u,v,w)\in W_h\\(x,y)\in\calS_j\times\calA}}
 \frac{q^{\pi_t}(u,v)\iota}{n_t(u,v)}q^{\pi_t}(x,y\mid w)}\\
 &\qquad\times\sqrt{\sum_{t=1}^T\sum_{\substack{(u,v,w)\in W_h\\(x,y)\in\calS_j\times\calA}}
 q^{\pi_t}(u,v)P(w\mid u,v)q^{\pi_t}(x,y\mid w)
 \frac{|\calS_{j+1}|\iota}{n_t(x,y)}}\\
 &\leq\sum_{0\leq h<j\leq H-1}
 \sqrt{|\calS_{h+1}|\sum_{t=1}^T\sum_{(u,v)\in\calS_h\times\calA}\frac{q^{\pi_t}(u,v)\iota}{n_t(u,v)}}
 \sqrt{|\calS_{j+1}|\sum_{t=1}^T\sum_{(x,y)\in\calS_j\times\calA}\frac{q^{\pi_t}(x,y)\iota}{n_t(x,y)}}\\
 &\lesssim A\iota^2\sum_{0\leq h<j\leq H-1}
 \sqrt{|\calS_h||\calS_{h+1}||\calS_j||\calS_{j+1}|}\\
 &\leq A\iota^2\prn*{\sum_{h=0}^{H-1}\sqrt{|\calS_h||\calS_{h+1}|}}^2\\
 &\leq A\iota^2\prn*{\frac12\sum_{h=0}^{H-1}\prn*{|\calS_h|+|\calS_{h+1}|}}^2
 \tag{by the AM--GM inequality}\\
 &=S^2A\iota^2, \label{eq:var_term2a}
\end{align}  
  where the   second inequality follows from
$\sum_{(x,y)\in\calS_j\times\calA}q^{\pi_t}(x,y\mid w)\leq1$ and
$\sum_{(u,v,w)\in W_h}q^{\pi_t}(u,v)P(w\mid u,v)q^{\pi_t}(x,y\mid w)=q^{\pi_t}(x,y)$ for $j>h$,
and the third inequality follows from \cref{lem:occup_bound}.
  For $\term_{2b}$,     \cref{lem:occup_bound} yields
\begin{align}
 \sum_{t=1}^T\term_{2b}
 &\lesssim A\iota^2\sum_{h=0}^{H-1}|\calS_h||\calS_{h+1}|
 \leq S^2A\iota^2. \label{eq:var_term2b}
\end{align}  
   Combining   \cref{eq:var_term2a,eq:var_term2b}  ,  we obtain
    \begin{align}
 \sum_{t=1}^T|\term_2|\lesssim HS^2AG\iota^2. \label{eq:var_term2c}
\end{align}  

 Combining the bounds   for  $\term_1$ and $\term_2$    and using the definition of $\Qtrans^{\pi_{1:T}}(g)$, we obtain
\begin{align}
 \sum_{t=1}^T\bigl(V^{\pi_t}(s_0;g_t)-V^{\tilP_t,\pi_t}(s_0;g_t)\bigr)
 &\lesssim\sqrt{S^2A\iota^2\Qtrans^{\pi_{1:T}}(g)}
 +HS^2AG\iota^2.
\end{align}
Applying the same bound to $-g_t$ completes the proof.
\end{proof}

\begin{lem}\label{lem:Qtrans_bound}
For arbitrary signed functions $g_t:\calS\times\calA\to\mathbb R$,
\begin{align}
 \Qtrans^{\pi_{1:T}}(g)\le H\sum_{t,s,a}q^{\pi_t}(s,a)g_t(s,a)^2.
\end{align}
\end{lem}

\begin{proof}
It suffices to prove the following inequality for each $t$:
\begin{equation}
\sum_{s,a} q^{\pi_t}(s,a)\Var_{s'\sim P(\cdot\mid s,a)}\brk*{V^{\pi_t}(s';g_t)} \leq
H\sum_{s,a} q^{\pi_t}(s,a)g_t(s,a)^2 .
\end{equation}

Let
$(X_{t,0},A_{t,0},X_{t,1},A_{t,1},\dots,X_{t,H})$
be the random trajectory generated by $(P,\pi_t)$.
For $h=0,1,\dots,H$ and $x\in \mathcal S_h$, define 
\begin{equation}
\calG_{t,h} \coloneqq \sum_{h'=h}^{H-1}g_t(X_{t,h'}, A_{t,h'}),\quad
\calV_{t,h}(x) \coloneqq \E_{\pi_t}[\calG_{t,h}\mid X_{t,h}=x],\quad
\calM_{t,h}(x) \coloneqq \E_{\pi_t}[\calG_{t,h}^2\mid X_{t,h}=x].
\end{equation}
In particular, $\calV_{t,H}=\calM_{t,H}=0$.

Since $\calG_{t,h} = g_t(X_{t,h},A_{t,h}) + \calG_{t,h+1}$, we have for any state $x$
\begin{align}
\calV_{t,h}(x) = \sum_a \pi_t(a\mid x)\prn*{g_t(x,a) + \sum_{x'} P(x'\mid x,a)\calV_{t, h+1}(x')},
\end{align}
and
\begin{align}
\calM_{t,h}(x) &=
\sum_a \pi_t(a\mid x)\prn*{g_t(x,a)^2+2g_t(x,a)\sum_{x'} P(x'\mid x,a)\calV_{t,h+1}(x')}\\
&\qquad +\sum_a \pi_t(a\mid x)\sum_{x'} P(x'\mid x,a)\calM_{t,h+1}(x').
\end{align}
Then, we have
\begin{align}
&\calM_{t,h}(x) - \calV_{t,h}(x)^2\\
&= \sum_a \pi_t(a\mid x)\prn*{g_t(x,a)^2 + 2g_t(x,a)\sum_{x'} P(x'\mid x,a)\calV_{t,h+1}(x') + \sum_{x'} P(x'\mid x,a)\calM_{t,h+1}(x')}\\
&\qquad - \prn*{\sum_a \pi_t(a\mid x)\prn*{g_t(x,a) + \sum_{x'} P(x'\mid x,a)\calV_{t,h+1}(x')}}^2 \\
&= \sum_a \pi_t(a\mid x)\sum_{x'} P(x'\mid x,a)\prn*{\calM_{t,h+1}(x') - \calV_{t,h+1}(x')^2} \\
&\qquad + \sum_a \pi_t(a\mid x)\prn*{\sum_{x'} P(x'\mid x,a)\calV_{t,h+1}(x')^2 - \prn*{\sum_{x'} P(x'\mid x,a)\calV_{t,h+1}(x')}^2} \\
&\qquad + \sum_a \pi_t(a\mid x)\prn*{g_t(x,a) + \sum_{x'} P(x'\mid x,a)\calV_{t,h+1}(x')}^2\\
&\qquad - \prn*{\sum_a \pi_t(a\mid x)\prn*{g_t(x,a) + \sum_{x'} P(x'\mid x,a)\calV_{t,h+1}(x')}}^2\\
&\geq \sum_a \pi_t(a\mid x)\sum_{x'} P(x'\mid x,a)\prn*{\calM_{t,h+1}(x') - \calV_{t,h+1}(x')^2} \\
&\qquad + \sum_a \pi_t(a\mid x)\Var_{x' \sim P(\cdot \mid x,a)}\brk*{\calV_{t,h+1}(x')},
\end{align}
where the last inequality follows from Jensen's inequality
\begin{align}
    &\prn*{\sum_a \pi_t(a\mid x)\prn*{g_t(x,a) + \sum_{x'} P(x'\mid x,a)\calV_{t,h+1}(x')}}^2 \\
    &\leq \sum_a \pi_t(a\mid x)\prn*{g_t(x,a) + \sum_{x'} P(x'\mid x,a)\calV_{t,h+1}(x')}^2.
\end{align}
Thus, we have
\begin{align}
\sum_a \pi_t(a\mid x)\Var_{x' \sim P(\cdot \mid x,a)}\brk*{\calV_{t,h+1}(x')} 
&\leq \prn*{\calM_{t,h}(x) - \calV_{t,h}(x)^2}\\
&\qquad - \sum_a \pi_t(a\mid x)\sum_{x'} P(x'\mid x,a)\prn*{\calM_{t,h+1}(x') - \calV_{t,h+1}(x')^2}. \\
\end{align}

Multiplying by $q^{\pi_t}(x)$ and summing over $x\in \mathcal S_h$, we obtain
\begin{align}
\sum_{(x,a)\in \mathcal S_h\times\mathcal A}
q^{\pi_t}(x,a)\Var_{x' \sim P(\cdot \mid x,a)}\brk*{\calV_{t,h+1}(x')} 
&\leq \sum_{x\in \mathcal S_h} q^{\pi_t}(x)\prn*{\calM_{t,h}(x) - \calV_{t,h}(x)^2}\\
&\qquad -\sum_{x'\in \mathcal S_{h+1}} q^{\pi_t}(x')\prn*{\calM_{t,h+1}(x') - \calV_{t,h+1}(x')^2}.
\end{align}
Summing over $h=0,1,\dots,H-1$ and using $\calV_{t,h+1}(x') = V^{\pi_t}(x'; g_t)$, we obtain
\begin{align}
\sum_{s,a} q^{\pi_t}(s,a)\Var_{s'\sim P(\cdot\mid s,a)}\brk*{V^{\pi_t}(s';g_t)}
\leq \calM_{t,0}(s_0) - \calV_{t,0}(s_0)^2
\leq \calM_{t,0}(s_0). \label{eq:Var_leq_M}
\end{align}

Finally,
\begin{align}
\calM_{t,0}(s_0)
&= \E_{\pi_t}\brk*{\prn*{\sum_{h=0}^{H-1}g_t(X_{t,h}, A_{t,h})}^2}\\
&\leq H\E_{\pi_t}\brk*{\sum_{h=0}^{H-1}g_t(X_{t,h}, A_{t,h})^2}
= H\sum_{s,a} q^{\pi_t}(s,a)g_t(s,a)^2, \label{eq:M_leq_phi}
\end{align}
where the inequality follows from the Cauchy–Schwarz inequality.
Combining \cref{eq:Var_leq_M,eq:M_leq_phi} completes the proof.
\end{proof}

From the proof of \cref{lem:complicated_lemma_cond}, we obtain the following bound, whose lower-order term improves that of \citet[Lemma C.4]{dann2023best}.
\begin{lem}
\label{lem:complicated_lemma}
Suppose that the event $\calE$ holds. Let $\tilP_t^s$ be a transition kernel in $\calP_t$ which may depend on $s$, and let $g_t(s)\in[0,G]$. Then
\begin{align}
        \sum_{t=1}^T \sum_s \abs*{ q^{\pi_t}(s) -  q^{\tilP_t^s, \pi_t}(s)} g_t(s) \lesssim \sqrt{HS^2A  \iota^2 \sum_{t=1}^T \sum_s q^{\pi_t}(s)g_t(s)^2} + S^3AG\iota^2. 
\end{align}
\end{lem}
\begin{lem}
\label{lem:complicated_lemma_cond}
Suppose that the event $\calE$ holds. Let $\tilP^{s,\tils}_t$ be a transition kernel in $\calP_t$ which may depend on $s,\tils$, and let $g_t(s)\in[0,G]$. Then
\begin{align}
        &\sum_{t=1}^T\sum_{\tils}q^{\pi_t}(\tils)\sum_s\abs*{q^{\pi_t}(s\mid\tils)-q^{\tilP_t^{s,\tils},\pi_t}(s\mid\tils)}g_t(s)\\
        &\lesssim\sqrt{H^3S^2A\iota^2\sum_{t=1}^T\sum_s q^{\pi_t}(s)g_t(s)^2}+  HS^3AG\iota^2.
\end{align}
\end{lem}
\begin{proof}
For each layer $ h \in \{0, \dots, H-1\} $, let $ W_h \coloneqq \mathcal{S}_h \times \mathcal{A} \times \mathcal{S}_{h+1} $. Throughout the proof, whenever summation indices are omitted, expressions such as $ \sum_{u,v,w} $ or $ \sum_{x,y,z} $ stand for $ \sum_{h=0}^{H-1} \sum_{(s,a,s') \in W_h} $.

By \cref{lem:occ_diff_lem}, for any $s,\tils$ and policy $\pi$,
{\small
\begin{align}
 &q^{\pi}(s\mid\tils)-q^{\tilP_t^{s,\tils},\pi}(s\mid\tils)\\
 &=\sum_{u,v,w}q^{\pi}(u,v\mid\tils)
 \prn*{P(w\mid u,v)-\tilP_t^{s,\tils}(w\mid u,v)}q^{\tilP_t^{s,\tils},\pi}(s\mid w)\\
 &=\sum_{u,v,w}q^{\pi}(u,v\mid\tils)
 \prn*{P(w\mid u,v)-\tilP_t^{s,\tils}(w\mid u,v)}q^{\pi}(s\mid w)\\
 &\quad+\sum_{u,v,w}q^{\pi}(u,v\mid\tils)
 \prn*{P(w\mid u,v)-\tilP_t^{s,\tils}(w\mid u,v)}
 \prn*{q^{\tilP_t^{s,\tils},\pi}(s\mid w)-q^{\pi}(s\mid w)}.
\end{align}
}
Taking absolute values, multiplying by $g_t(s)$, and summing over $s$, we obtain
{\small
\begin{align}
&\sum_s\abs*{q^{\pi}(s\mid\tils)-q^{\tilP_t^{s,\tils},\pi}(s\mid\tils)}g_t(s)\\
&\lesssim\sum_s\sum_{u,v,w}q^{\pi}(u,v\mid\tils)
 \prn*{\sqrt{\frac{P(w\mid u,v)\iota}{n_t(u,v)}}+\frac{\iota}{n_t(u,v)}}q^{\pi}(s\mid w)g_t(s)\\
&\quad+\sum_s\sum_{u,v,w}q^{\pi}(u,v\mid\tils)
 \prn*{\sqrt{\frac{P(w\mid u,v)\iota}{n_t(u,v)}}+\frac{\iota}{n_t(u,v)}}
 \abs*{q^{\tilP_t^{s,\tils},\pi}(s\mid w)-q^{\pi}(s\mid w)}g_t(s)
 \tag{by \cref{lem:tildeP-P}}\\
&\lesssim\sum_s\sum_{u,v,w}q^{\pi}(u,v\mid\tils)
 \sqrt{\frac{P(w\mid u,v)\iota}{n_t(u,v)}}q^{\pi}(s\mid w)g_t(s)\\
&\qquad+SG\sum_{h=0}^{H-1}|\calS_{h+1}|\sum_{(u,v)\in\calS_h\times\calA}
 \frac{q^{\pi}(u,v\mid\tils)\iota}{n_t(u,v)}\\
&\qquad+\sum_s\sum_{u,v,w}q^{\pi}(u,v\mid\tils)
 \sqrt{\frac{P(w\mid u,v)\iota}{n_t(u,v)}}
 \abs*{q^{\tilP_t^{s,\tils},\pi}(s\mid w)-q^{\pi}(s\mid w)}g_t(s),
 \label{eq:cond_occup_diff}
\end{align} 
}
  where we used $q^{\pi}(s\mid w)+\abs*{q^{\tilP_t^{s,\tils},\pi}(s\mid w)-q^{\pi}(s\mid w)}\leq2$ and $g_t(s)\leq G$.
 For   $w\in\calS_{h+1}$, we have
{\small
\begin{align}
 &\abs*{q^{\tilP_t^{s,\tils},\pi}(s\mid w)-q^{\pi}(s\mid w)}\\
 &\leq\sum_{j=h+1}^{H-1}\sum_{(x,y)\in\calS_j\times\calA}q^{\pi}(x,y\mid w)
 \sum_z\abs*{P(z\mid x,y)-\tilP_t^{s,\tils}(z\mid x,y)}q^{\tilP_t^{s,\tils},\pi}(s\mid z)\tag{by \cref{lem:occ_diff_lem}}\\
 &\lesssim\sum_{j=h+1}^{H-1}\sum_{(x,y)\in\calS_j\times\calA}q^{\pi}(x,y\mid w)
 \sqrt{\frac{|\calS_{j+1}|\iota}{n_t(x,y)}}.
\tag{by \cref{eq:transition_row_l1}}\end{align}
  }
Combining these bounds gives

{\small
\begin{align}
 &\sum_{t=1}^T\sum_{\tils}q^{\pi_t}(\tils)\sum_s\abs*{q^{\pi_t}(s\mid\tils)-q^{\tilP_t^{s,\tils},\pi_t}(s\mid\tils)}g_t(s)\\
 &\lesssim\sum_{t=1}^T\sum_s\sum_{u,v,w}
 \prn*{\sum_{\tils}q^{\pi_t}(\tils)q^{\pi_t}(u,v\mid\tils)}
 \sqrt{\frac{P(w\mid u,v)\iota}{n_t(u,v)}}q^{\pi_t}(s\mid w)g_t(s)\\
 &\quad+SG\sum_{t=1}^T\sum_{h=0}^{H-1}|\calS_{h+1}|\sum_{(u,v)\in\calS_h\times\calA}
 \prn*{\sum_{\tils}q^{\pi_t}(\tils)q^{\pi_t}(u,v\mid\tils)}\frac{\iota}{n_t(u,v)}\\
 &\quad+SG\sum_{t=1}^T\sum_{0\leq h<j\leq H-1}
 \sum_{\substack{(u,v,w)\in W_h\\(x,y)\in\calS_j\times\calA}}
 \prn*{\sum_{\tils}q^{\pi_t}(\tils)q^{\pi_t}(u,v\mid\tils)}
 \sqrt{\frac{P(w\mid u,v)\iota}{n_t(u,v)}}
 q^{\pi_t}(x,y\mid w)\sqrt{\frac{|\calS_{j+1}|\iota}{n_t(x,y)}}
 \tag{by \cref{eq:cond_occup_diff} and $\sum_sg_t(s)\leq SG$}\\
 &\lesssim\underbrace{H\sum_{t=1}^T\sum_s\sum_{u,v,w}q^{\pi_t}(u,v)
 \sqrt{\frac{P(w\mid u,v)\iota}{n_t(u,v)}}q^{\pi_t}(s\mid w)g_t(s)}_{\term_1}\\*
 &\quad+\underbrace{HSG\sum_{t=1}^T\sum_{h=0}^{H-1}|\calS_{h+1}|\sum_{(u,v)\in\calS_h\times\calA}
 \frac{q^{\pi_t}(u,v)\iota}{n_t(u,v)}}_{\term_2}\\*
 &\quad+\underbrace{HSG\sum_{t=1}^T\sum_{0\leq h<j\leq H-1}
 \sum_{\substack{(u,v,w)\in W_h\\(x,y)\in\calS_j\times\calA}}
 q^{\pi_t}(u,v)\sqrt{\frac{P(w\mid u,v)\iota}{n_t(u,v)}}
 q^{\pi_t}(x,y\mid w)\sqrt{\frac{|\calS_{j+1}|\iota}{n_t(x,y)}}}_{\term_3}.
 \tag{by $\sum_{\tils}q^{\pi_t}(\tils)q^{\pi_t}(u,v\mid\tils)\leq Hq^{\pi_t}(u,v)$}
 \label{eq:conditional_decomp}
\end{align}
}

We first bound $\term_1$. Fix $h$. Then
    \begin{align}
        &H\sum_{t=1}^T \sum_s \prn*{\sum_{(u,v,w) \in W_h} q^{\pi_t}(u,v) \sqrt{\frac{P(w\mid u,v)\iota}{n_t(u,v)}}q^{\pi_t}(s\mid w)}g_t(s) \\
        &\leq H\sum_{t=1}^T \sum_s \prn*{ \sum_{(u,v,w) \in \calS_h\times\calA\times\calS_{h+1}} q^{\pi_t}(u,v)\prn*{\alpha P(w\mid u,v) g_t(s)^2 + \frac{\iota}{\alpha n_t(u,v)}}q^{\pi_t}(s\mid w)} \tag{holds for any $\alpha>0$ by the AM-GM inequality}\\
        &= \alpha H \sum_{t=1}^T\sum_s\sum_{(u,v,w) \in \calS_h\times\calA\times\calS_{h+1}} q^{\pi_t}(u,v) P(w\mid u,v) q^{\pi_t}(s\mid w)g_t(s)^2\\
        &\qquad + \frac{H}{\alpha}\sum_{t=1}^T \sum_s \sum_{(u,v,w) \in \calS_h\times\calA\times\calS_{h+1}} \frac{q^{\pi_t}(u,v)\iota}{n_t(u,v)}q^{\pi_t}(s\mid w)  \\
        &\leq \alpha H \sum_{t=1}^T\sum_s  q^{\pi_t}(s)g_t(s)^2 + \frac{H^2\abs{\calS_{h+1}}}{\alpha}\sum_{t=1}^T \sum_{(u,v) \in \calS_h \times \calA} \frac{q^{\pi_t}(u,v)\iota}{n_t(u,v)} \\
        &\lesssim \alpha H \sum_{t=1}^T\sum_s  q^{\pi_t}(s)g_t(s)^2 + \frac{H^2\abs{\calS_{h+1}}\abs{\calS_h}A\iota^2}{\alpha} \tag{by \cref{lem:occup_bound} and the assumption that $\calE$ holds}\\
        &\lesssim \sqrt{H^3 \abs{\calS_h}\abs{\calS_{h+1}} A  \iota^2 \sum_{t=1}^T \sum_s q^{P, \pi_t}(s)g_t(s)^2}   \tag{by picking the optimal $\alpha$} \\
        &\leq (\abs{\calS_h} + \abs{\calS_{h+1}})\sqrt{H^3 A  \iota^2 \sum_{t=1}^T \sum_s q^{P, \pi_t}(s)g_t(s)^2}.\tag{by the AM--GM inequality} \\
    \end{align}
Summing over $h$ gives
\begin{align}
    \sum_{h=0}^{H-1}(\abs{\calS_h} + \abs{\calS_{h+1}})\sqrt{H^3 A \iota^2 \sum_{t=1}^T \sum_s q^{\pi_t}(s)g_t(s)^2}
     \lesssim \sqrt{H^3 S^2A  \iota^2 \sum_{t=1}^T \sum_s q^{\pi_t}(s)g_t(s)^2}. \label{eq:conditional_decomp1}
\end{align}
For $\term_2$ and $\term_3$, applying \cref{eq:var_term2b} and \cref{eq:var_term2a}, respectively, gives
\begin{align}
 \term_2+\term_3\lesssim HSG\cdot S^2A\iota^2=HS^3AG\iota^2.
 \label{eq:conditional_decomp2}
\end{align}
Combining \cref{eq:conditional_decomp1,eq:conditional_decomp2} proves \cref{lem:complicated_lemma_cond}.
For \cref{lem:complicated_lemma}, set $\tils=s_0$ in \cref{eq:cond_occup_diff} and sum over $t$.
The same estimates, with the prefactor $H$ in each of $\term_1$, $\term_2$, and $\term_3$ omitted, prove \cref{lem:complicated_lemma}.
\end{proof}

\subsection{Data-dependent and self-bounding lemmas}
\begin{lem}\label{lem:max_to_value}
It holds that
 \begin{align}   
    \E\brk*{\sum_{t=1}^T\sum_{s,a} q^{\pi_t}(s,a)\max_{\tilP\in\calP_t} Q^{\tilP, \pi_t}(s, a;\prn*{\ell_t - m_t}^2)} \lesssim H\E\brk*{\sum_{t=1}^TV^{\pi_t}(s_0;\prn*{\ell_t - m_t}^2)} + HS^3A\iota^2.
\end{align}
  
\end{lem}

 \begin{proof}
On the event $\bar{\calE}$, $\sum_{t=1}^T\sum_{s,a}q^{\pi_t}(s,a)\max_{\tilP\in\calP_t}Q^{\tilP,\pi_t}(s,a;(\ell_t-m_t)^2)=\calO(H^2T)$ gives a contribution $\calO(H^2T\delta)$, which is absorbed by taking $\delta=1/T^2$.

Hence, it suffices to prove the bound on $\calE$.
For each $t$,  let $\overP^Q_t\in\calP_t$ be such that
\begin{align}
    Q^{\overP^Q_t,\pi_t}(s,a;\prn*{\ell_t-m_t}^2)
    = \max_{\tilP\in\calP_t}Q^{\tilP,\pi_t}(s,a;\prn*{\ell_t-m_t}^2)
\end{align}
for all $(s,a)$.
Then,
    \begin{align}
        &\sum_{t=1}^T \sum_{s,a} q^{\pi_t}(s, a) Q^{\overP^Q_t, \pi_t}(s, a;\prn*{\ell_t - m_t}^2)\\
        &\leq \sum_{t=1}^T \sum_{s,a} q^{\pi_t}(s, a) Q^{\pi_t}(s, a;\prn*{\ell_t - m_t}^2)\\
        &\qquad + \sum_{t=1}^T \sum_{s,a} q^{\pi_t}(s, a)\prn*{Q^{\overP^Q_t, \pi_t}(s, a;\prn*{\ell_t - m_t}^2) -Q^{\pi_t}(s, a;\prn*{\ell_t - m_t}^2)}. \label{eq:qmax_decomp} 
    \end{align}
  
For the first term, we have
\begin{align}
    &\sum_{t=1}^T \sum_{s,a} q^{\pi_t}(s,a)Q^{\pi_t}(s, a;\prn*{\ell_t - m_t}^2)\\
    &= \sum_{t=1}^T \sum_{s,a} q^{\pi_t}(s,a) \sum_{s',a'}q^{\pi_t}(s',a'\mid s, a)\prn*{\ell_t(s',a') - m_t(s',a')}^2 \\
    &\leq H\sum_{t=1}^T \sum_{s',a'}q^{\pi_t}(s',a')\prn*{\ell_t(s',a') - m_t(s',a')}^2 \\
    &= H \sum_{t=1}^TV^{\pi_t}(s_0; \prn*{\ell_t - m_t}^2) \label{eq:qmax_first}
\end{align}
For the second term,  we obtain
\begin{align}
    &\sum_{t=1}^T \sum_{s,a} q^{\pi_t}(s, a)\prn*{Q^{\overP^Q_t, \pi_t}(s, a;\prn*{\ell_t - m_t}^2) -Q^{\pi_t}(s, a;\prn*{\ell_t - m_t}^2)} \\
    &= \sum_{t=1}^T \sum_{s,a} q^{\pi_t}(s, a)\sum_{s'}\prn*{q^{\overP^Q_t, \pi_t}(s'\mid s,a) - q^{\pi_t}(s'\mid s,a)}\sum_{a'}\pi_t(a'\mid s')\prn*{\ell_t(s',a') - m_t(s',a')}^2 \\
    &= \sum_{t=1}^T \sum_{s} q^{\pi_t}(s)\sum_{s'}\prn*{q^{\overP^Q_t, \pi_t}(s'\mid s) - q^{\pi_t}(s'\mid s)}\sum_{a'}\pi_t(a'\mid s')\prn*{\ell_t(s',a') - m_t(s',a')}^2 \\
    &\lesssim \sqrt{H^3S^2A  \iota^2 \sum_{t=1}^T \sum_{s',a'} q^{\pi_t}(s',a')\prn*{\ell_t(s',a') - m_t(s',a')}^2} + HS^3A\iota^2 \tag{\mbox{by \mbox{\cref{lem:complicated_lemma_cond}}}}\\
    &= \sqrt{H^3S^2A  \iota^2 \sum_{t=1}^T V^{\pi_t}(s_0; \prn*{\ell_t - m_t}^2)} + HS^3A\iota^2. \label{eq:qmax_second}
\end{align}

  Combining \cref{eq:qmax_decomp,eq:qmax_first,eq:qmax_second}, we obtain
 \begin{align}
&\sum_{t=1}^T \sum_{s,a} q^{\pi_t}(s, a) Q^{\overP^Q_t, \pi_t}(s, a;\prn*{\ell_t - m_t}^2)\\
&\lesssim H\sum_{t=1}^TV^{\pi_t}(s_0;\prn*{\ell_t - m_t}^2) + \sqrt{H^3S^2A  \iota^2 \sum_{t=1}^T V^{\pi_t}(s_0; \prn*{\ell_t - m_t}^2)} + HS^3A\iota^2\\
&\lesssim H\sum_{t=1}^TV^{\pi_t}(s_0;\prn*{\ell_t - m_t}^2) + HS^3A\iota^2,
\end{align}
  
 where the last line follows from the AM--GM inequality since 
$ax + \sqrt{bx} \leq ax + ax + \frac{b}{2a}$ for $a>0$ and $b,x\geq0$.
 Taking expectations completes the proof.
\end{proof}

\begin{lem}[{\citealt[Lemma E.9]{li2026data}}]
\label{lem:L-M_to_l-m}
It holds that
\begin{align}
    \E\brk*{\sum_{t=1}^T\sum_{s,a}\I_t(s,a)(L_{t,h(s)} - M_{t,h(s)})^2} \leq H^2\E\brk*{\sum_{t=1}^T\sum_{s,a}\I_t(s,a)(\ell_t(s,a) - m_t(s,a))^2}.
\end{align}
\end{lem}
\begin{lem}[{\citealt[Lemma F.12]{li2026data}}]\label{lem:general_predict_result}
Suppose $m_t$ is defined in \cref{def:predictor_sequence} with some constant $\xi \in (0, \frac{1}{2})$. Then, for any comparator policy $\pi$, it holds that
\begin{align}
    \E\brk*{\sum_{t=1}^TV^{\pi_t}(s_0;\prn*{\ell_t-m_t}^2)}
    &= \E\brk*{\sum_{t=1}^T\sum_{s,a}\I_t(s,a)(\ell_t(s,a) - m_t(s,a))^2}\\
    &\lesssim \min\set*{L(\pi)  + \Reg_T(\pi), HT - L(\pi)  - \Reg_T(\pi), Q_{\infty}, V_1} + SA.
\end{align}
\end{lem}

\begin{lem}\label{lem:qtrans_m_to_ell}
For any sequence of policies $\{\pi_t\}_{t=1}^T$ and any prediction sequence
$\{m_t\}_{t=1}^T$, it holds that
\begin{align}
    \Qtrans^{\pi_{1:T}}(m)
    &\leq
    2\Qtrans^{\pi_{1:T}}(\ell)
    +
    2H\sum_{t=1}^T V^{\pi_t}(s_0;(\ell_t-m_t)^2).
    \label{eq:qtrans_m_to_ell}
\end{align}
\end{lem}
\begin{proof}
For any state $s$,  
\begin{align}
    V^{\pi_t}(s;m_t)=V^{\pi_t}(s;\ell_t)+V^{\pi_t}(s;m_t-\ell_t).
\end{align}
By using $\Var(X+Y)\leq 2\Var(X)+2\Var(Y)$, we have
  \begin{align}
    &\sum_{t=1}^T\sum_{s,a}q^{\pi_t}(s,a)\Var_{s'\sim P(\cdot\mid s,a)}\brk*{V^{\pi_t}(s';m_t)}\\
    &\leq 2\sum_{t=1}^T\sum_{s,a}q^{\pi_t}(s,a)\Var_{s'\sim P(\cdot\mid s,a)}\brk*{V^{\pi_t}(s';\ell_t)}+2\sum_{t=1}^T\sum_{s,a}q^{\pi_t}(s,a)\Var_{s'\sim P(\cdot\mid s,a)}\brk*{V^{\pi_t}(s';m_t-\ell_t)}\\
    &\leq 2\Qtrans^{\pi_{1:T}}(\ell)+2\Qtrans^{\pi_{1:T}}(m-\ell).
\end{align} 
  Thus, we have
\begin{align}
    \Qtrans^{\pi_{1:T}}(m)\leq 2\Qtrans^{\pi_{1:T}}(\ell)+2\Qtrans^{\pi_{1:T}}(m-\ell).
\end{align} 
Finally, \cref{lem:Qtrans_bound} with   $g_t=m_t-\ell_t$ gives
\begin{align}
    \Qtrans^{\pi_{1:T}}(m-\ell)\leq H\sum_{t=1}^T V^{\pi_t}(s_0;(\ell_t-m_t)^2).
\end{align} 
Therefore, 
\begin{align}
    \Qtrans^{\pi_{1:T}}(m)\leq 2\Qtrans^{\pi_{1:T}}(\ell)+2H\sum_{t=1}^T V^{\pi_t}(s_0;(\ell_t-m_t)^2).
\end{align}
\end{proof}

\begin{lem}[{\citet[Lemma C.5]{dann2023best}}] \label{lem:policy_diff_bound}
    For any two policies $\pi_1$ and $\pi_2$, it holds that
    \begin{align}
        \sum_{s,a}\abs*{q^{\pi_1}(s,a) - q^{\pi_2}(s,a)} \leq H \sum_{s,a}q^{\pi_1}(s)\abs*{\pi_1(a \mid s)- \pi_2(a\mid s)}.
    \end{align}
\end{lem}
\begin{lem}\label{lem:policy_diff_layer}
    For any policy $\pi$ and layer $h$, it holds that
    \begin{align}
        \sum_{(s,a) \in \calS_h \times \calA}\abs*{q^{\pi}(s,a) - q^{\pist}(s,a)} \leq 2\sum_{h'=0}^{h } \sum_{s \in \calS_{h'}}\sum_{a\neq \pist(s)}q^{\pi}(s,a).
    \end{align}
\end{lem}
\begin{proof}
    Since $\pist$ is deterministic, we have
    \begin{align}
        &\sum_{(s,a) \in \calS_h \times \calA}\abs*{q^{\pi}(s,a) - q^{\pist}(s,a)} \\
        &= \sum_{(s,a) \in \calS_h \times \calA}\abs*{q^{\pi}(s)\pi(a\mid s) - q^{\pist}(s)\ind\brk*{{a = \pist(s)}}} \\
        &\leq \sum_{(s,a) \in \calS_h \times \calA}\abs*{q^{\pi}(s)\pi(a\mid s) - q^{\pi}(s)\ind\brk*{{a = \pist(s)}}}  +  \sum_{s\in \calS_h}\abs*{q^{\pi}(s) - q^{\pist}(s)} \tag{by the triangle inequality}\\
        &= \sum_{s \in \calS_h}\sum_{a \neq \pist(s)}q^{\pi}(s,a) + \sum_{s\in \calS_h}q^{\pi}(s)(1 - \pi(\pist(s)\mid s)) + \sum_{s\in \calS_h}\abs*{q^{\pi}(s) - q^{\pist}(s)} \\
        &= 2\sum_{s \in \calS_h}\sum_{a \neq \pist(s)}q^{\pi}(s,a) + \sum_{s\in \calS_h}\abs*{q^{\pi}(s) - q^{\pist}(s)}. \label{eq:occupancy_action_l1}
    \end{align}
    Next, by the triangle inequality,
    \begin{align}
        &\sum_{s'\in \calS_{h+1}}\abs*{q^{\pi}(s') - q^{\pist}(s')} \\
        &=  \sum_{s'\in \calS_{h+1}}\abs*{\sum_{s \in \calS_h}\sum_{a}\prn*{q^{\pi}(s)\pi(a\mid s) P(s'\mid s, a) - q^{\pist}(s)\ind\brk*{a = \pist(s)}P(s'\mid s, \pist(s))}} \\
        &\leq\sum_{s'\in \calS_{h+1}}\abs*{\sum_{s \in \calS_h}\prn*{q^{\pi}(s)P(s'\mid s, \pist(s)) - q^{\pist}(s)P(s'\mid s, \pist(s))}} \\
        &\qquad + \sum_{s'\in \calS_{h+1}}\abs*{\sum_{s \in \calS_h}\sum_{a\neq \pist(s)}q^{\pi}(s)\pi(a\mid s) P(s'\mid s, a) + \sum_{s \in \calS_h}q^{\pi}(s)(\pi(\pist(s)\mid s) - 1)P(s'\mid s, \pist(s))} \\
        &\leq \sum_{s'\in \calS_{h+1}}\sum_{s \in \calS_h}\abs*{q^{\pi}(s) - q^{\pist}(s)}P(s'\mid s, \pist(s))  \\
        &\qquad + \sum_{s'\in \calS_{h+1}}\sum_{s \in \calS_h}\sum_{a\neq \pist(s)}q^{\pi}(s,a) \abs*{P(s'\mid s, a) -P(s'\mid s, \pist(s))} \\
        &\leq \sum_{s \in \calS_h}\abs*{q^{\pi}(s) - q^{\pist}(s)} + 2\sum_{s \in \calS_h}\sum_{a\neq \pist(s)}q^{\pi}(s,a),\label{eq:state_l1_recursion}
    \end{align}
    where the last inequality follows from $\sum_{s'\in \calS_{h+1}}\abs*{P(s'\mid s, a) -P(s'\mid s, \pist(s))} \leq 2$ for all state-action pairs $(s,a)$.

    Applying \eqref{eq:state_l1_recursion} recursively and using $q^\pi(s_0)=q^{\pist}(s_0)$, we obtain
    \begin{align}
        \sum_{s \in \calS_h}\abs*{q^{\pi}(s) - q^{\pist}(s)} \leq 2\sum_{h'=0}^{h - 1} \sum_{s \in \calS_{h'}}\sum_{a\neq \pist(s)}q^{\pi}(s,a). \label{eq:state_l1_bound}
    \end{align}
    Combining \cref{eq:occupancy_action_l1,eq:state_l1_bound} yields
    \begin{align}
        \sum_{(s,a) \in \calS_h \times \calA}\abs*{q^{\pi}(s,a) - q^{\pist}(s,a)} \leq 2\sum_{h'=0}^{h } \sum_{s \in \calS_{h'}}\sum_{a\neq \pist(s)}q^{\pi}(s,a).
    \end{align} 
    This completes the proof.
\end{proof}

\begin{lem}[{\citealt[Section 2.1]{jin2020learning}}]
\label{lem:corruption_self-bounding}
Under the stochastic regime with adversarial corruption, for any sequence of policies $\{\pi_t\}_{t=1}^T$, the regret satisfies the following $(\Delta,2\calC,T)$ self-bounding constraint:
\begin{equation}
    \Reg_T(\pio)\geq \E\brk*{\sum_{t=1}^T \sum_{s}\sum_{a\neq \pist(s)} q^{\pi_t}(s,a)\Delta(s,a)} - 2 \calC. \label{eq:adversarial_regime_with_self-bounding_constraints}
\end{equation}
\end{lem}
\begin{lem}
\label{lem:general_self_bounding}
Let $G:S\times A\to\mathbb{R}_{\geq 0}$ be any nonnegative function. 
Under the stochastic regime with adversarial corruption, for any $\alpha>0$, it holds that
    \begin{align}
        \sum_{s}\sum_{a\neq\pi^\star(s)}G(s,a)\sqrt{\E\brk*{\sum_{t=1}^T q^{\pi_t}(s,a)}} &\leq \alpha(\Reg_T(\pio) + 2\calC) + \sum_{s}\sum_{a\neq\pi^\star(s)} \frac{G(s,a)^2}{4\alpha\Delta(s,a)}. \label{eq:selfbound_pit}
    \end{align}
\end{lem}
\begin{proof}
For each $(s,a)$ with $a\neq \pi^\star(s)$, the AM--GM inequality and \cref{lem:corruption_self-bounding} imply
\begin{align}
    &\sum_{s}\sum_{a\neq\pi^\star(s)}G(s,a)\sqrt{\E\brk*{\sum_{t=1}^T q^{\pi_t}(s,a)}} \\
    &\leq \alpha\E\brk*{\sum_{t=1}^T\sum_{s}\sum_{a\neq\pi^\star(s)} q^{\pi_t}(s,a)\Delta(s,a)} + \sum_{s}\sum_{a\neq\pi^\star(s)} \frac{G(s,a)^2}{4\alpha\Delta(s,a)}\tag{by the AM--GM inequality with $\alpha > 0$}\\
    &\leq \alpha(\Reg_T(\pio) + 2\calC) + \sum_{s}\sum_{a\neq\pi^\star(s)} \frac{G(s,a)^2}{4\alpha\Delta(s,a)}. \tag{by \cref{lem:corruption_self-bounding}}
\end{align}
\end{proof}

\begin{lem}
\label{lem:self_bound_layer_positive}
Let $G(h)$ be any nonnegative function defined on layer $h$. Under the stochastic regime with adversarial corruption, for any $\alpha>0$, it holds that
\begin{align}
    &\sum_{h = 0}^{H-1}\sqrt{G(h)\E\brk*{\sum_{t=1}^T\sum_{(s,a) \in \calS_h \times \calA}\brk*{q^{\pi_t}(s,a) - q^{\pio}(s,a)}_{+}}}\\
    &\qquad\leq  \alpha(\Reg_T(\pio) + 4\calC) + \frac{\prn*{\sum_{h = 0}^{H-1}\sqrt{G(h)}}^2}{4\alpha\Delta_{\min}}.
\end{align}
\end{lem}

\begin{proof}
Applying \cref{lem:policy_diff_layer}, we obtain
\begin{align}
    &\E\brk*{\sum_{t=1}^T\sum_{(s,a) \in \calS_h \times \calA}\brk*{q^{\pi_t}(s,a) - q^{\pio}(s,a)}_{+}}\\
    &= \E\brk*{\frac{1}{2}\sum_{t=1}^T\sum_{(s,a) \in \calS_h \times \calA}\abs*{q^{\pi_t}(s,a)-q^{\pio}(s,a)}}\\
    &\leq \E\brk*{\frac{1}{2}\sum_{t=1}^T\sum_{(s,a) \in \calS_h \times \calA}\prn*{\abs*{q^{\pi_t}(s,a)-q^{\pist}(s,a)} + \abs*{q^{\pio}(s,a)-q^{\pist}(s,a)}}}\\
    &\leq \E\brk*{\sum_{t=1}^T\sum_{h'=0}^{h } \sum_{s \in \calS_{h'}}\sum_{a\neq \pist(s)}\prn*{q^{\pi_t}(s,a) + q^{\pio}(s,a)}}.
\end{align}
We next bound the comparator contribution. Using the uncorrupted mean loss function $\mu$, we have
\begin{align}
    &\E\brk*{\sum_{t=1}^T\sum_s\sum_{a\neq\pist(s)}q^{\pio}(s,a)\Delta(s,a)}\\
    &=\E\brk*{\sum_{t=1}^T\prn*{V^{\pio}(s_0;\mu)-V^{\pist}(s_0;\mu)}}\\
    &=\E\brk*{\sum_{t=1}^T\prn*{V^{\pio}(s_0;\mu)-V^{\pio}(s_0;\ell_t)}}+\E\brk*{\sum_{t=1}^T\prn*{V^{\pio}(s_0;\ell_t)-V^{\pist}(s_0;\ell_t)}}\\
    &\qquad+\E\brk*{\sum_{t=1}^T\prn*{V^{\pist}(s_0;\ell_t)-V^{\pist}(s_0;\mu)}}.
\end{align}
The middle term is nonpositive by the definition of $\pio$. Moreover, since $\E[\ell'_t]=\mu$, for any fixed policy $\pi$,   \begin{align}
 \left|\E\brk*{\sum_{t=1}^T\bigl(V^\pi(s_0;\ell_t)-V^\pi(s_0;\mu)\bigr)}\right|
 &=\left|\E\brk*{\sum_{t=1}^T\bigl(V^\pi(s_0;\ell_t)-V^\pi(s_0;\ell'_t)\bigr)}\right|\\
 &\le\E\brk*{\sum_{t=1}^T\sum_{h=0}^{H-1}\|\ell_t(h)-\ell'_t(h)\|_\infty}
 =\calC.
\end{align} 
Therefore,
\begin{align}
    \E\brk*{\sum_{t=1}^T\sum_s\sum_{a\neq\pist(s)}q^{\pio}(s,a)\Delta(s,a)}
    &\leq 2\calC.
    \label{eq:pio_delta_self}
\end{align}
Combining \cref{eq:pio_delta_self} with \cref{lem:corruption_self-bounding}, we get
\begin{align}
    \E\brk*{\sum_{t=1}^T\sum_s\sum_{a\neq\pist(s)}\prn*{q^{\pi_t}(s,a)+q^{\pio}(s,a)}\Delta(s,a)}
    &\leq \Reg_T(\pio)+4\calC.
    \label{eq:self_bound_pi_t_pio}
\end{align}
Then, for any $\alpha > 0$, we have
\begin{align}
    &\sum_{h = 0}^{H-1}\sqrt{G(h)\E\brk*{\sum_{t=1}^T\sum_{(s,a) \in \calS_h \times \calA}\brk*{q^{\pi_t}(s,a) - q^{\pio}(s,a)}_{+}}}\\
    &\leq \sum_{h = 0}^{H-1}\sqrt{G(h)\E\brk*{\sum_{t=1}^T\sum_{s}\sum_{a\neq \pist(s)}\prn*{q^{\pi_t}(s,a) + q^{\pio}(s,a)}}}\\
    &\leq \sum_{h = 0}^{H-1}\sqrt{G(h)}\sqrt{\frac{1}{\Delta_{\min}}\E\brk*{\sum_{t=1}^T\sum_{s}\sum_{a\neq \pist(s)}\prn*{q^{\pi_t}(s,a) + q^{\pio}(s,a)}\Delta(s,a)}}\\
    &\leq \sum_{h = 0}^{H-1}\sqrt{G(h)}\sqrt{\frac{\Reg_T(\pio) + 4\calC}{\Delta_{\min}}}\\
    &\leq \alpha(\Reg_T(\pio) + 4\calC) + \frac{\prn*{\sum_{h = 0}^{H-1}\sqrt{G(h)}}^2}{4\alpha\Delta_{\min}},
\end{align}
where we use the AM--GM inequality.
\end{proof}

\begin{lem}
\label{lem:self_bound_occ_positive}
Let $G>0$. Under the stochastic regime with adversarial corruption, for any $\alpha>0$, it holds that
\begin{align}
\sqrt{G \E\brk*{\sum_{t=1}^T\sum_{s,a}\brk*{q^{\pi_t}(s,a)-q^{\pio}(s,a)}_+}}&\leq \alpha\prn*{\Reg_T(\pio) + 4\calC} + \frac{HG}{4\alpha\Delta_{\min}}.
\end{align}
\end{lem}

\begin{proof}
Applying \cref{lem:policy_diff_bound} to compare any policy $\pi$ with $\pist$, we get
\begin{align}
    \sum_{t=1}^T\sum_{s,a}\abs*{q^{\pi}(s,a)-q^{\pist}(s,a)}
    &\leq H \sum_{t=1}^T\sum_{s,a}q^{\pi}(s)\abs*{\pi(a \mid s)- \pist(a\mid s)}\\
    &\leq H \sum_{t=1}^T\sum_{s}\sum_{a \neq \pist(s)}q^{\pi}(s)\pi(a \mid s)\\
    &\qquad + H \sum_{t=1}^T\sum_{s}q^{\pi}(s)\prn*{1 - \pi(\pist(s) \mid s)}\\
    &\leq 2H \sum_{t=1}^T\sum_{s}\sum_{a \neq \pist(s)}q^{\pi}(s)\pi(a \mid s).
\end{align}
Using this bound with $\pi=\pi_t$ and $\pi=\pio$, for any $\alpha>0$ we obtain
  \begin{align}
    &\sqrt{G\E\brk*{\sum_{t=1}^T\sum_{s,a}\brk*{q^{\pi_t}(s,a)-q^{\pio}(s,a)}_+}}\\
    &= \sqrt{\frac12G\E\brk*{\sum_{t=1}^T\sum_{s,a}\abs*{q^{\pi_t}(s,a)-q^{\pio}(s,a)}}}\\
    &\leq \sqrt{\frac12G\E\brk*{\sum_{t=1}^T\sum_{s,a}\prn*{\abs*{q^{\pi_t}(s,a)-q^{\pist}(s,a)} + \abs*{q^{\pio}(s,a)-q^{\pist}(s,a)}}}}\\
    &\leq \sqrt{HG\E\brk*{\sum_{t=1}^T\sum_{s}\sum_{a \neq \pist(s)}\prn*{q^{\pi_t}(s,a) + q^{\pio}(s,a)}}}\\
    &\leq \sqrt{HG\frac{1}{\Delta_{\min}}\E\brk*{\sum_{t=1}^T\sum_{s}\sum_{a \neq \pist(s)}\prn*{q^{\pi_t}(s,a) + q^{\pio}(s,a)}\Delta(s,a)}}\\
    &\leq \sqrt{HG\frac{\Reg_T(\pio) + 4\calC}{\Delta_{\min}}} \tag{by \cref{lem:corruption_self-bounding} and \cref{eq:pio_delta_self}}\\
    &\leq \alpha(\Reg_T(\pio) + 4\calC) + \frac{HG}{4\alpha\Delta_{\min}}, \label{eq:self_bounding_min} 
\end{align} 
where the last inequality follows from the AM--GM inequality.
\end{proof}

\section{General lemma of optimistic follow-the-regularized-leader}
In this section, we provide a regret analysis of optimistic follow-the-regularized-leader (OFTRL) for the MDP setting.

\begin{lem}[{OFTRL with policy optimization, \citealt[Lemma C.4]{li2026data}}]
\label{lem:OFTRL_log_barrier_policy}  
Suppose that a sequence of probability vectors
$p_1, \dots, p_T \in \triangle(\calA)$
is given by OFTRL in 
  \begin{align}
    p_t = \argmin_{p\in\Delta(\calA)}\set*{\inpr*{p, \sum_{\tau = 1}^{t - 1}\ell_\tau + m_t} + \psi_t(p)}, \qquad \psi_t(p) = \sum_a \frac{1}{\eta_t(a)}\ln\prn*{\frac{1}{p(a)}}
\end{align}  
with $\eta_1(a)=\eta_1$ for all $a$, and positive nonincreasing learning rates, and let losses $\set{\ell_t}_{t=1}^T$, loss predictions $\set{m_t}_{t=1}^{T+1}$ and $\set{x_t}_{t=1}^T$ be such that
\begin{align}
    \eta_t(a) p_t(a)(\ell_t(a) - m_t(a) + x_t) &\geq -\frac{1}{2} \label{eq:OFTRL_constraint_policy}
\end{align}
for all $t,a$.
Then, for any  $u\in\triangle(\mathcal{A})$, the OFTRL algorithm achieves
  \begin{align}
    \sum_{t=1}^T \inpr*{ p_t-u, \ell_t} 
    &\leq  \frac{A\ln(AT^2)}{\eta_1} + \sum_{t=1}^T \sum_a \prn*{\frac{1}{\eta_{t+1}(a)} - \frac{1}{\eta_t(a)}}\ln(AT^2)\\
    &\qquad +  \sum_{t=1}^T  \sum_a  \eta_t(a)p_t(a)^2(\ell_t(a)-m_t(a) + x_t)^2\\
    &\qquad + \frac{1}{T^2}\sum_{t=1}^T \inpr*{ -u + \frac{1}{A}\one, \ell_t } + 2\nrm*{m_{T+1}}_{\infty}. 
\end{align} 
\end{lem}

\section{Regret analysis for the full-information setting (\cref{thm:main_full})}
\label{app:full_proof}
In this section, we prove the best-of-both-worlds results for the full-information setting. We present the bound for the adversarial regime in \cref{thm:app_full_adv} and the bound for the stochastic regime with adversarial corruption in \cref{thm:app_full_corruption}. Together, these results establish \cref{thm:main_full}.

\subsection{Auxiliary lemmas}
Building on the policy optimization framework of \citet{luo2021policy,dann2023best}, we use the following key lemma to derive our regret bounds.

\begin{lem}[Restatement of \cref{lem:dilated_bonus_main}]
\label{lem:dilated_bonus}
    Suppose that $P\in\calP_t$ for all $t\in[T]$.
    Suppose that $b_t(s)$ is a nonnegative loss function.
    Suppose also that, for a comparator policy $\pi$, there exists   $J^\pi\in\mathbb R$  such that
    \begin{align}
        &\E\brk*{\sum_s q^{\pi}(s)\sum_{t=1}^T\sum_a \prn*{\pi_t(a\mid s) - \pi(a\mid s) } \prn*{Q^{\pi_t}(s,a;\ell_t) -  B_t(s,a)}}\\
        &\leq J^{\pi} + \E\brk*{\sum_{t = 1}^T\sum_s q^{\pi}(s) b_t(s)} + \E\brk*{\frac{1}{H} \sum_{t = 1}^T\sum_s\sum_a q^{\pi}(s) \pi_t(a\mid s)B_t(s,a)}
        .\label{eq:key_lemma_ineq}
    \end{align}
    Then, it holds that
      \begin{align}
        \Reg_T(\pi) &\leq J^{\pi} + 3 \, \E\brk*{\sum_{t = 1}^T V^{\overP^B_t, \pi_t}(s_0;b_t)},
    \end{align} 
    where $\overP^B_t \in \calP_t$ simultaneously attains the maxima in \cref{def:dilated_bonus} for all $(s,a)$.
\end{lem}

To show \cref{eq:key_lemma_ineq}, we choose $b_t$ in \cref{eq:bonus_term_full} and decompose the LHS of \cref{eq:key_lemma_ineq} as
\begin{align}
    &\sum_s q^{\pi}(s) \sum_{t,a} \prn*{\pi_t(a\mid s) -  \pi(a\mid s) } \prn*{Q^{\pi_t}(s,a;\ell_t) -  B_t(s,a)}\\
    &= \sum_s q^{\pi}(s)  \underbrace{\sum_{t,a}  \prn*{\pi_t(a\mid s) -  \pi(a\mid s)} \prn*{\hat{Q}_t(s,a) -  B_t(s,a)}}_{\regterm(s)}\\
    &\qquad + \sum_s q^{\pi}(s) \underbrace{\sum_{t,a} \prn*{\pi_t(a\mid s) -  \pi(a\mid s)} \prn*{Q^{\pi_t}(s,a;\ell_t) -  \hat{Q}_t(s,a)}}_{\biasterm(s)}. \label{eq:decompose_full}
\end{align}

\begin{lem}[full-information setting]
\label{lem:bound_B_full}
The variables $b_t(s)$ in \cref{eq:bonus_term_full} and $B_t(s,a)$ in \cref{def:dilated_bonus} satisfy
\begin{align}
    B_t(s,a) \leq 2
    \label{eq:bound_eta_pi_B_full}
\end{align}
for all episodes $t$ and state-action pairs $(s,a)$.
\end{lem}

\begin{proof}
We first upper bound the dilated bonus-to-go $B_t(s,a)$ by unrolling the dilated recursion.
Since $(1+1/H)^H\le 3$, we obtain
\begin{align}
    B_t(s,a) \leq 3\sum_{s'} q^{\overP^B_t, \pi_t}(s'\mid s,a)\,b_t(s'). \label{eq:dilated_bonus_upper}
\end{align}

By the definition of $b_t$ in \cref{eq:bonus_term_full} and the learning-rate update \cref{eq:learning_rate_full},
\begin{align}
    b_t(s) 
    &= 7\sum_a  \prn*{\frac{1}{\eta_{t+1}(s, a)} - \frac{1}{\eta_t(s, a)}}\ln(T)\\
    &\leq 7\sum_{a} \eta_t(s,a)\zeta_t(s,a)\\
    &\leq 7\eta_1H^2,
\end{align}
where we used $\sum_a \zeta_t(s,a)\le H^2$ and $\eta_t(s,a)\le \eta_1$.

Using $\eta_1 = \frac{1}{12H^3}$ and \cref{eq:dilated_bonus_upper}, we get
\allowdisplaybreaks
\begin{align}
    B_t(s,a) &\leq 3\sum_{s'} q^{\overP^B_t, \pi_t}(s'\mid s,a)b_t(s') \leq 21\eta_1H^3 \leq 2.\label{eq:bound_B_full}
\end{align}
\end{proof}

\begin{lem}[{\citealt[Lemma G.3]{dann2023best}}]
\label{lem:eta_lemma_normal}
    Let $\eta_1>0, \eta_2, \eta_3, \ldots$ be updated by 
    \begin{align}
        \frac{1}{\eta_{t+1}} = \frac{1}{\eta_t} + \eta_t \phi_t\qquad \forall t\geq 1
    \end{align}
    with $0\leq \phi_t\leq \eta_t^{-2}$. 
    Then,
    \begin{align}
        \frac{1}{\eta_{t+1}} \geq \frac{1}{2}\sqrt{\sum_{\tau=1}^{t+1} \phi_\tau}. 
    \end{align}
\end{lem}
\begin{lem}[full-information setting]
\label{lem:learning_rate_full}
Suppose that the learning rates are updated according to \cref{eq:learning_rate_full}. Then, it holds
    \begin{align}
        \eta_t(s,a) \leq \frac{2\sqrt{\ln(T)}}{\sqrt{\sum_{\tau = 1}^t\zeta_\tau(s,a)} }
    \end{align}
for all episodes $t$ and state-action pairs $(s,a)$.
\end{lem}

\begin{proof}
    Fix $(s,a)$ and define $\phi_t(s,a)=\frac{\zeta_t(s,a)}{\ln(T)}$. Then the update rule of learning rates can be written as
    \begin{align}
        \frac{1}{\eta_{t+1}(s,a)} = \frac{1}{\eta_t(s,a)} + \eta_t(s,a) \phi_t(s,a).
    \end{align}
    To apply \cref{lem:eta_lemma_normal}, we verify that
     $\phi_t(s,a)\le \frac{1}{\eta_t(s,a)^2}$. Indeed, 
    \begin{align}
        \phi_t(s,a)\eta_t(s,a)^2 
        = \frac{\eta_t(s,a)^2\zeta_t(s,a)}{\ln(T)}
        \leq \frac{\eta_1^2H^2}{\ln(T)}
        \leq \frac{H^2}{\ln(T)}\cdot \frac{1}{144H^6} 
        \leq 1,
    \end{align}
    which follows from $\eta_1 = \frac{1}{12H^3}$.

    Then, by \cref{lem:eta_lemma_normal}, we have
    \begin{align}
        \eta_t(s,a) \leq \frac{2}{\sqrt{\sum_{\tau = 1}^t \phi_\tau }}
        \leq \frac{2\sqrt{\ln(T)}}{\sqrt{\sum_{\tau = 1}^t\zeta_\tau(s,a)} }.
    \end{align}
\end{proof}

\subsection{Common regret analysis}
In this part, we upper bound the right-hand side of \cref{eq:key_lemma_ineq}.
\begin{lem}[full-information setting]
\label{lem:regterm_full}
    For each state $s\in\calS$ and any comparator policy $\pi$, it holds that
    \begin{align}
        \regterm(s) &\leq O\prn*{H^3A\ln(T)} + \sum_{t=1}^Tb_t(s) + \frac{1}{H}\sum_{t=1}^T\sum_a \pi_t(a \mid s) B_t(s,a).
    \end{align}
\end{lem}

\begin{proof}
We will apply \cref{lem:OFTRL_log_barrier_policy} with $p_t = \pi_t(\cdot \mid s)$ and $\ell_t = \hat{Q}_t(s,a) - B_t(s,a)$ for each $s \in \calS$.
To do so, in what follows, we will check the conditions of \Cref{lem:OFTRL_log_barrier_policy}.
Let
\begin{align}
x_t &= \inpr*{-\pi_t(\cdot \mid s) , \hatQ_t(s,\cdot) - Q^{\underP^m_t, \pi_t}(s, \cdot;m_t)}\\
&= -\sum_{a}\pi_t(a\mid s) \prn*{Q^{\underP^\ell_t, \pi_t}(s, a;\ell_t) - Q^{\underP^m_t, \pi_t}(s, a;m_t)}
\end{align}
and verify that for all $(s,a)$, $\eta_t(s,a)\pi_t(a\mid s)\prn*{\hat{Q}_t(s,a)-B_t(s,a) - Q^{\underP^m_t, \pi_t}(s, a;m_t) + x_t}\geq -{1}/{2}$. Indeed, 
\begin{align}
    &\eta_t(s,a)\pi_t(a\mid s)\prn*{\hat{Q}_t(s,a)-B_t(s,a)-Q^{\underP^m_t, \pi_t}(s, a;m_t) + x_t} \\ 
    &\geq \eta_t(s,a)\pi_t(a\mid s) \prn*{-H-B_t(s,a)-H} \\
    &\geq -4\eta_1H \geq -\frac{1}{2}, 
\end{align}
where the bounds in the second lines use $-H\leq \hat{Q}_t(s,a)-Q^{\underP^m_t, \pi_t}(s, a;m_t) \leq H$, $B_t(s,a)\leq 2\leq 2H$ from \cref{lem:bound_B_full}.

Thus, \cref{lem:OFTRL_log_barrier_policy} implies that, for any comparator policy $\pi$,
  \begin{align}
        & \regterm(s)\\
        &\leq \frac{A\ln(AT^2)}{\eta_1} + \sum_{t=1}^T \sum_a \prn*{\frac{1}{\eta_{t+1}(s, a)} - \frac{1}{\eta_t(s, a)}}\ln(AT^2)\\
        &\qquad + \sum_{t=1}^T \sum_a  \eta_t(s, a)\pi_t(a \mid s)^{2}\prn*{Q^{\underP^\ell_t, \pi_t}(s, a;\ell_t) - Q^{\underP^m_t, \pi_t}(s, a;m_t) - B_t(s, a)+ x_t}^2 \\
        &\qquad + \frac{1}{T^2}\sum_{t=1}^T \inpr*{ -\pi(\cdot \mid s) + \frac{1}{A}\one, \hatQ_t(s, \cdot) - B_t(s, \cdot) } + 2\nrm*{Q^{\underP^m_{T+1}, \pi_{T+1}}(s, \cdot;m_{T+1})}_{\infty} \\
        &\leq \frac{3A\ln(T)}{\eta_1} + \frac{6H}{T} + 2H + 3\sum_{t=1}^T\sum_a \prn*{\frac{1}{\eta_{t+1}(s, a)} - \frac{1}{\eta_t(s, a)}}\ln(T) \tag{$\nrm*{Q^{\underP^m_{T+1}, \pi_{T+1}}(s, \cdot;m_{T+1})}_{\infty} \leq H$} \\
        &\qquad + 2\sum_{t=1}^T \sum_a  \eta_t(s, a)\pi_t(a\mid s)^{2}\prn*{Q^{\underP^\ell_t, \pi_t}(s, a;\ell_t) - Q^{\underP^m_t, \pi_t}(s, a;m_t) +x_t}^2\\
        &\qquad + 2\sum_{t=1}^T \sum_a  \eta_t(s, a)\pi_t(a\mid s)^{2} B_t(s, a)^2 \\ 
        &\leq \calO(H^3A\ln(T)) + 3\sum_{t=1}^T\sum_a \prn*{\frac{1}{\eta_{t+1}(s, a)} - \frac{1}{\eta_t(s, a)}}\ln(T)  \\
        &\qquad + 2\underbrace{\sum_{t=1}^T \sum_a  \eta_t(s, a)\pi_t(a\mid s)^{2}\prn*{Q^{\underP^\ell_t, \pi_t}(s, a;\ell_t) - Q^{\underP^m_t, \pi_t}(s, a;m_t)+x_t}^2}_{\text{stability-term}}\\
        &\qquad + \frac{1}{H}\sum_{t=1}^T\sum_a \pi_t(a \mid s) B_t(s,a).
        \label{eq:regterm_full_mid}
    \end{align} 
    Here, the second inequality follows from 
    \begin{align}
        &\frac{1}{T^2}\sum_{t=1}^T \inpr*{ -\pi(\cdot \mid s) + \frac{1}{A}\one, \hatQ_t(s, \cdot) - B_t(s, \cdot) }\\
        &\leq \frac{1}{T^2}\nrm*{-\pi(\cdot \mid s) + \frac{1}{A}\one}_1\nrm*{\sum_{t=1}^T\prn*{\hatQ_t(s, \cdot) - B_t(s, \cdot)} }_\infty \\
        &\leq \frac{2T\cdot 3H}{T^2} = \frac{6H}{T},
    \end{align}
    where we used $\nrm*{-\pi(\cdot \mid s) + \frac{1}{A}\one}_1\leq 2$,
    $\abs*{\hatQ_t(s,a)} \leq H$ and $B_t(s,a) \leq 2$ from \cref{lem:bound_B_full}.

    It remains to bound the \text{stability-term}. Define
    \begin{align}
        \calQ(s,a) \coloneq Q^{\underP^\ell_t, \pi_t}(s, a;\ell_t) - Q^{\underP^m_t, \pi_t}(s, a;m_t).
    \end{align}
Since
\begin{align}
    x_t = -\sum_b \pi_t(b \mid s)\calQ(s,b),
\end{align}
we have
\begin{align}
    \prn*{\calQ(s,a)  - \sum_{b}\pi_t(b \mid s) \calQ(s,b)}^2
    &= \prn*{(1 - \pi_t(a \mid s))\calQ(s,a) - \sum_{b\neq a}\pi_t(b \mid s) \calQ(s,b)}^2\\
    &\leq 2(1 - \pi_t(a \mid s))^2\calQ(s,a)^2 + 2\prn*{\sum_{b\neq a}\pi_t(b \mid s) \calQ(s,b)}^2\\
    &\leq 2(1 - \pi_t(a \mid s))^2\calQ(s,a)^2\\
    &\qquad + 2(1 - \pi_t(a \mid s))\sum_{b\neq a}\pi_t(b \mid s) \calQ(s,b)^2,
\end{align}
where the last inequality uses the Cauchy–Schwarz inequality.
Thus, we obtain 
      \begin{align}
        \pi_t(a \mid s)^2\prn*{\calQ(s,a)  - \sum_{b}\pi_{t}(b \mid s) \calQ(s,b)}^2 
        &\leq 2\pi_t(a \mid s)(1 - \pi_t(a \mid s))\sum_{b}\pi_t(b \mid s) \calQ(s, b)^2\\
    \end{align} 
Therefore, 
    \begin{align}
        &\eta_t(s, a)\pi_t(a\mid s)^{2}\prn*{Q^{\underP^\ell_t, \pi_t}(s, a;\ell_t) - Q^{\underP^m_t, \pi_t}(s, a;m_t) + x_t}^2   \\
        &\leq 2\eta_t(s,a)\pi_t(a\mid s)(1 - \pi_t(a \mid s))\sum_{b}\pi_t(b \mid s)\prn*{Q^{\underP^\ell_t, \pi_t}(s, b;\ell_t) - Q^{\underP^m_t, \pi_t}(s, b;m_t)}^2\\
        &\leq 2\eta_t(s,a)\zeta_t(s, a),
    \end{align}
    where the last equality follows from the definition of $\zeta_t(s, a) = \pi_t(a\mid s)(1 - \pi_t(a \mid s))\sum_{b}\pi_t(b \mid s)\prn*{Q^{\underP^\ell_t, \pi_t}(s, b;\ell_t) - Q^{\underP^m_t, \pi_t}(s, b;m_t)}^2$.

Then, \text{stability-term} is evaluated as
    \begin{align}
        \text{stability-term} &= \sum_{t=1}^T \sum_a  \eta_t(s, a)\pi_t(a\mid s)^{2}\prn*{Q^{\underP^\ell_t, \pi_t}(s, a;\ell_t) - Q^{\underP^m_t, \pi_t}(s, a;m_t) + x_t}^2 \\
        &\leq 2\sum_{t=1}^T \sum_a \eta_t(s,a)\zeta_t(s,a) \\
        &\leq 2\sum_{t=1}^T \sum_a\prn*{\frac{1}{\eta_{t+1}(s,a)} - \frac{1}{\eta_t(s,a)}}\ln(T),
    \end{align}
    where the last inequality follows from \cref{eq:learning_rate_full}.

Combining the above bound on the stability-term with \cref{eq:regterm_full_mid}, we obtain, for any policy $\pi$,
    \begin{align}
        \regterm(s) 
        &\leq \calO\prn*{H^3A\ln(T)} + 7\sum_{t=1}^T\sum_a  \prn*{\frac{1}{\eta_{t+1}(s, a)} - \frac{1}{\eta_t(s, a)}}\ln(T) \\
        &\qquad + \frac{1}{H}\sum_{t=1}^T\sum_a \pi_t(a \mid s) B_t(s,a)\\
        &= \calO\prn*{H^3A\ln(T)} + \sum_{t=1}^Tb_t(s) + \frac{1}{H}\sum_{t=1}^T\sum_a \pi_t(a \mid s) B_t(s,a).
    \end{align}
\end{proof}

\begin{lem}[full-information setting]
\label{lem:biasterm_full}
    Suppose that the event $\calE$ holds. Then, for any comparator policy $\pi$, it holds that
     \begin{align}
        \sum_s q^{\pi}(s)\biasterm(s)
        &\lesssim 
        \sqrt{S^2A\iota^2\Qtrans^{\pi_{1:T}}(\ell)} + HS^{2}A\iota^2. \label{eq:bound_bias_1}
    \end{align} 
    Also,
    \begin{align}
        \sum_s q^{\pi}(s)\biasterm(s)
        &\lesssim 
        \sum_{h = 0}^{H-1}\sqrt{H^2\abs*{S_{h}}\abs*{S_{h + 1}}A\iota^2 \sum_{t=1}^T\sum_{(s,a) \in \calS_h \times \calA}\brk*{q^{\pi_t}(s,a) - q^{\pi}(s,a)}_{+}}\\
        &\qquad + HS^2A\iota^2. \label{eq:bound_bias_2}
    \end{align}
\end{lem}

\begin{proof}
\begin{align}
   \sum_s q^{\pi}(s)\biasterm(s)
   &=\sum_s q^{\pi}(s) \sum_{t,a} \prn*{\pi_t(a\mid s) -  \pi(a\mid s)} \prn*{Q^{\pi_t}(s, a; \ell_t) - Q^{\underP^\ell_t, \pi_t}(s, a; \ell_t)}\\
    &= \sum_{t=1}^T \sum_{s,a} \prn*{q^{\pi_t}(s,a) - q^{\pi}(s,a)} z_t(s,a),
\end{align}
where, for each $t$, \cref{lem:pdl_with_L} is applied with
\begin{align}
    L(s,a)=Q^{\pi_t}(s, a; \ell_t) - Q^{\underP^\ell_t, \pi_t}(s, a; \ell_t).
\end{align}
Thus,
\begin{align}
    z_t(s,a)  &= Q^{\pi_t}(s, a; \ell_t) - Q^{\underP^\ell_t, \pi_t}(s, a; \ell_t)  \\
    &\quad - \E_{s'\sim P(\cdot\mid s,a), a'\sim\pi_t(\cdot\mid s')}\brk*{Q^{\pi_t}(s', a'; \ell_t) - Q^{\underP^\ell_t, \pi_t}(s', a'; \ell_t)}.
\end{align}
Expanding the definition of $z_t(s,a)$, we obtain
\begin{align}
    z_t(s,a) 
    &= \E_{s'\sim P(\cdot\mid s,a), a'\sim\pi_t(\cdot\mid s')}\brk*{Q^{\pi_t}(s', a'; \ell_t)} - \E_{s'\sim \underP^\ell_t(\cdot\mid s,a), a'\sim\pi_t(\cdot\mid s')}\brk*{Q^{\underP^\ell_t, \pi_t}(s', a'; \ell_t)}\\
    &\qquad - \E_{s'\sim P(\cdot\mid s,a), a'\sim\pi_t(\cdot\mid s')}\brk*{Q^{\pi_t}(s', a'; \ell_t) - Q^{\underP^\ell_t, \pi_t}(s', a'; \ell_t)}
    \\
    &= \E_{s'\sim P(\cdot\mid s,a), a'\sim\pi_t(\cdot\mid s')}\brk*{Q^{\underP^\ell_t, \pi_t}(s', a'; \ell_t)} - \E_{s'\sim \underP^\ell_t(\cdot\mid s,a), a'\sim\pi_t(\cdot\mid s')}\brk*{Q^{\underP^\ell_t, \pi_t}(s', a'; \ell_t)}
    \\
    &= \E_{s'\sim P(\cdot\mid s,a)}\brk*{V^{\underP^\ell_t, \pi_t}(s'; \ell_t)} - \E_{s'\sim \underP^\ell_t(\cdot\mid s,a)}\brk*{V^{\underP^\ell_t, \pi_t}(s'; \ell_t)}.
\end{align}
By the definition of $\underP^\ell_t$, which minimizes $Q(s,a;\ell_t)$ for all $(s,a)$, we have $z_t(s,a) \geq 0$.

First, we show \cref{eq:bound_bias_1}. By the fact that $z_t(s,a) \geq 0$,
\begin{align}
\sum_{t=1}^T \sum_{s,a} \prn*{q^{\pi_t}(s,a) - q^{\pi}(s,a)} z_t(s,a)
&\leq  \sum_{t=1}^T V^{\pi_t}(s_0; z_t)\\
&\leq \sum_{t=1}^T \prn*{V^{\pi_t}(s_0; \ell_t) - V^{\underP^\ell_t, \pi_t}(s_0; \ell_t)} \tag{by \cref{lem:pdl_with_L}}\\
&= \sum_{t=1}^T \sum_{s} \prn*{q^{\pi_t}(s) - q^{\underP^\ell_t, \pi_t}(s)}\sum_a\pi_t(a \mid s) \ell_t(s,a). \label{eq:bound_bias_1_mid}
\end{align}

Applying  \cref{lem:signed_transition}  to \cref{eq:bound_bias_1_mid}, we obtain
 \begin{align}
    \sum_{t=1}^T \sum_{s,a} \prn*{q^{\pi_t}(s,a) - q^{\pi}(s,a)} z_t(s,a) 
    &\lesssim \sqrt{S^2A\iota^2\Qtrans^{\pi_{1:T}}(\ell)} + HS^{2}A\iota^2.
\end{align} 
This proves \cref{eq:bound_bias_1}.

Next, we show \cref{eq:bound_bias_2}. By \cref{lem:concentration_var},
\begin{align}
    z_t(s,a) &= \sum_{s'}\prn*{P(s' \mid s, a) - \underP^\ell_t(s' \mid s,a)}V^{\underP^\ell_t, \pi_t}(s'; \ell_t)\\
    &\lesssim \sqrt{\frac{\abs*{S_{h(s) + 1}}\Var_{s' \sim P(\cdot \mid s, a)}\brk*{V^{\underP^\ell_t, \pi_t}(s'; \ell_t)}\iota}{n_t(s,a)}}  +\frac{\abs{S_{h(s)+1}}H\iota}{n_t(s,a)}.
\end{align}
Using $z_t(s,a)\geq 0$, we further have
\begin{align}
&\sum_{t=1}^T \sum_{s,a} \prn*{q^{\pi_t}(s,a) - q^{\pi}(s,a)} z_t(s,a)\\
&\leq \sum_{t=1}^T \sum_{s,a} \brk*{q^{\pi_t}(s,a) - q^{\pi}(s,a)}_{+} z_t(s,a)\\
&\lesssim \sum_{t=1}^T \sum_{s,a} \brk*{q^{\pi_t}(s,a) - q^{\pi}(s,a)}_{+} \prn*{\sqrt{\frac{\abs*{S_{h(s) + 1}}\Var_{s' \sim P(\cdot \mid s, a)}\brk*{V^{\underP^\ell_t, \pi_t}(s'; \ell_t)}\iota}{n_t(s,a)}}  +\frac{\abs{S_{h(s)+1}}H\iota}{n_t(s,a)}}\\
&\leq \sum_{h = 0}^{H-1}\sum_{t=1}^T \sum_{(s,a) \in \calS_h \times \calA} \brk*{q^{\pi_t}(s,a) - q^{\pi}(s,a)}_{+} \sqrt{\frac{\abs*{S_{h + 1}}\Var_{s' \sim P(\cdot \mid s, a)}\brk*{V^{\underP^\ell_t, \pi_t}(s'; \ell_t)}\iota}{n_t(s,a)}}\\
&\qquad + \sum_{t=1}^T \sum_{s,a} q^{\pi_t}(s,a) \frac{\abs{S_{h(s)+1}}H\iota}{n_t(s,a)}. \label{eq:bound_bias_2_mid}
\end{align}

Here, by the Cauchy–Schwarz inequality, the first term in \cref{eq:bound_bias_2_mid} is bounded as
\begin{align}
&\sum_{t=1}^T \sum_{(s,a) \in \calS_h \times \calA} \brk*{q^{\pi_t}(s,a) - q^{\pi}(s,a)}_{+} \sqrt{\frac{\abs*{S_{h + 1}}\Var_{s' \sim P(\cdot \mid s, a)}\brk*{V^{\underP^\ell_t, \pi_t}(s'; \ell_t)}\iota}{n_t(s,a)}}\\
&\leq \sqrt{\sum_{t=1}^T \sum_{(s,a) \in \calS_h \times \calA} \brk*{q^{\pi_t}(s,a) - q^{\pi}(s,a)}_{+}}\\
&\qquad \times \sqrt{\sum_{t=1}^T \sum_{(s,a) \in \calS_h \times \calA} \brk*{q^{\pi_t}(s,a) - q^{\pi}(s,a)}_{+} \frac{\abs*{S_{h + 1}}\Var_{s' \sim P(\cdot \mid s, a)}\brk*{V^{\underP^\ell_t, \pi_t}(s'; \ell_t)}\iota}{n_t(s,a)}}.
\end{align}
The term inside the second square root is further bounded as
\begin{align}
&\sum_{t=1}^T \sum_{(s,a) \in \calS_h \times \calA} \brk*{q^{\pi_t}(s,a) - q^{\pi}(s,a)}_{+} \frac{\abs*{S_{h + 1}}\Var_{s' \sim P(\cdot \mid s, a)}\brk*{V^{\underP^\ell_t, \pi_t}(s'; \ell_t)}\iota}{n_t(s,a)}\\
&\leq H^2\abs*{S_{h + 1}} \sum_{t=1}^T \sum_{(s,a) \in \calS_h \times \calA} \frac{q^{\pi_t}(s,a) \iota}{n_t(s,a)} \\
&\lesssim H^2 \abs*{S_{h}}\abs*{S_{h + 1}}A\iota^2. \tag{by \cref{lem:occup_bound}}
\end{align}
The second term in \cref{eq:bound_bias_2_mid} is bounded as
\begin{align}
    \sum_{t=1}^T \sum_{s,a} q^{\pi_t}(s,a) \frac{\abs{S_{h(s)+1}}H\iota}{n_t(s,a)} 
    &\leq \sum_{h = 0}^{H - 1}\abs*{S_{h + 1}} \sum_{t=1}^T \sum_{(s,a) \in \calS_h \times \calA} q^{\pi_t}(s,a) \frac{H\iota}{n_t(s,a)} \\
    &\lesssim H\sum_{h = 0}^{H - 1} \abs*{S_{h + 1}} \abs*{S_{h}}A\iota^2 \\
    &\leq HS^2A\iota^2.
\end{align}
Therefore,
\begin{align}
    &\sum_{t=1}^T \sum_{s,a} \prn*{q^{\pi_t}(s,a) - q^{\pi}(s,a)} z_t(s,a) \\
    &\lesssim \sum_{h = 0}^{H-1}\sqrt{H^2\abs*{S_{h}}\abs*{S_{h + 1}}A\iota^2 \sum_{t=1}^T\sum_{(s,a) \in \calS_h \times \calA}\brk*{q^{\pi_t}(s,a) - q^{\pi}(s,a)}_{+}} + HS^2A\iota^2\\
\end{align}
This proves \cref{eq:bound_bias_2}, and hence completes the proof.
\end{proof}
\begin{lem}[full-information setting]
\label{lem:Vbt_full}
Suppose that the event $\calE$ holds. Then, it holds that
 \begin{align}
\sum_{t=1}^T V^{\overP_t^B,\pi_t}(s_0;b_t)
\lesssim \sqrt{\ln(T)}\sum_{s,a}\sqrt{\sum_{t=1}^T q^{\pi_t}(s)\zeta_t(s,a)}
+S^2A\iota^2.
\end{align} 
\end{lem}
\begin{proof}

First, we have
\begin{align}
    \sum_{t=1}^T V^{\overP_t^B,\pi_t}(s_0;b_t)
    &=\sum_{t=1}^T\sum_s q^{\overP^B_t,\pi_t}(s)b_t(s)
    \\
    &=\sum_{t=1}^T\sum_s q^{\pi_t}(s)b_t(s)+\sum_{t=1}^T\sum_s\prn*{q^{\overP^B_t,\pi_t}(s)-q^{\pi_t}(s)}b_t(s).
    \label{eq:regret_overB}
\end{align}
To bound the second term, define
  \begin{align}
    g_t(s,a)\coloneq b_t(s)\leq 7\eta_1H^2=\frac{7}{12H}\coloneq G.
\end{align} 
Applying \cref{lem:complicated_lemma_var,lem:Qtrans_bound}  with this choice of   $g_t$ yields
\begin{align}
    \sum_{t=1}^T\sum_s\prn*{q^{\overP^B_t,\pi_t}(s)-q^{\pi_t}(s)}b_t(s)
    &\lesssim\sqrt{HS^2A\iota^2\sum_{t=1}^T\sum_sq^{\pi_t}(s)b_t(s)^2}
    +HS^2AG\iota^2\\
    &\leq\sqrt{HS^2AG\iota^2\sum_{t=1}^T\sum_sq^{\pi_t}(s)b_t(s)}+HS^2AG\iota^2\\
    &\lesssim\sqrt{S^2A\iota^2\sum_{t=1}^T\sum_s q^{\pi_t}(s)b_t(s)}+S^2A\iota^2,
\end{align} 
 where we used $b_t(s)^2\leq G b_t(s)$ and $HG=7/12$.
 Combining this with \cref{eq:regret_overB}, we obtain
\begin{align}
    \sum_{t=1}^T\sum_s q^{\overP^B_t,\pi_t}(s)b_t(s)
    &\lesssim \sum_{t=1}^T\sum_s q^{\pi_t}(s)b_t(s)+\sqrt{S^2A\iota^2\sum_{t=1}^T\sum_sq^{\pi_t}(s)b_t(s)}+S^2A\iota^2\\
    &\lesssim \sum_{t=1}^T\sum_s q^{\pi_t}(s)b_t(s)+S^2A\iota^2,
    \label{eq:qB_to_q}
\end{align} 
where the last inequality follows from the AM--GM inequality    .

We next bound $\sum_{t=1}^T\sum_s q^{\pi_t}(s)b_t(s)$. 

By \mbox{\cref{lem:learning_rate_full}}, we have
\begin{align}
    &\sum_{t=1}^T\sum_s q^{\pi_t}(s)b_t(s)=7\sum_{t=1}^T\sum_{s,a}q^{\pi_t}(s)\eta_t(s,a)\zeta_t(s,a)\\
    &\lesssim\sqrt{\ln(T)}\sum_{t=1}^T\sum_{s,a}
    \frac{q^{\pi_t}(s)\sqrt{\zeta_t(s,a)}\times\sqrt{\zeta_t(s,a)}}{\sqrt{\sum_{\tau=1}^t\zeta_\tau(s,a)}}\\
    &\leq\sqrt{\ln(T)}\sum_{s,a}\sum_{t=1}^T
    \frac{q^{\pi_t}(s)\zeta_t(s,a)}{\sqrt{\sum_{\tau=1}^t q^{\pi_\tau}(s)\zeta_\tau(s,a)}}\\
    &\lesssim\sqrt{\ln(T)}\sum_{s,a}\sqrt{\sum_{t=1}^T q^{\pi_t}(s)\zeta_t(s,a)}.
    \label{eq:res_qBzeta}
\end{align}

Here we used $q^{\pi_t}(s)\leq1$ and, for $x_t\geq0$,
\begin{align}
\sum_{t=1}^T\frac{x_t}{\sqrt{\sum_{\tau=1}^t x_\tau}}
&\leq2\sum_{t=1}^T\prn*{\sqrt{\sum_{\tau=1}^t x_\tau}-\sqrt{\sum_{\tau=1}^{t-1}x_\tau}}
=2\sqrt{\sum_{t=1}^T x_t}.
\end{align}

By \cref{eq:res_qBzeta,eq:qB_to_q}, we obtain
  \begin{align}
    \sum_{t=1}^T V^{\overP^B_t,\pi_t}(s_0;b_t)
    &\lesssim\sqrt{\ln(T)}\sum_{s,a}\sqrt{\sum_{t=1}^T q^{\pi_t}(s)\zeta_t(s,a)}
    +S^2A\iota^2.
\end{align} 
This completes the proof.
\end{proof}

\begin{lem}[full-information setting]
\label{lem:regret_to_zetat}
For any comparator policy $\pi$, \cref{alg:OFTRL_unknown} guarantees
     \begin{align}
        \Reg_T(\pi) &\lesssim A(H^4+HS^2)\iota^2 + \sum_{s,a}\sqrt{\ln(T)\E\brk*{\sum_{t=1}^T q^{\pi_t}(s) \zeta_t(s,a)}}\\
        &\qquad + \min\Biggl\{\sqrt{S^2A\iota^2\E\brk*{\Qtrans^{\pi_{1:T}}(\ell)}},\\
        &\qquad\qquad\qquad \sum_{h = 0}^{H-1}\sqrt{H^2\abs*{S_{h}}\abs*{S_{h + 1}}A\iota^2
        \E\brk*{\sum_{t=1}^T\sum_{(s,a) \in \calS_h \times \calA}\brk*{q^{\pi_t}(s,a) - q^{\pi}(s,a)}_{+}}}\Biggr\}.
    \end{align} 
\end{lem}
\begin{proof}
By the definition of the regret decomposition in \cref{eq:decompose_full},
 \begin{align}
    &\E\brk*{\sum_s q^{\pi}(s)\sum_{t,a} \prn*{\pi_t(a\mid s) -  \pi(a\mid s) } \prn*{Q^{\pi_t}(s,a;\ell_t) -  B_t(s,a)}}\\
    &= \E\brk*{\sum_s q^{\pi}(s) \cdot \regterm(s)} + \E\brk*{\sum_s q^{\pi}(s) \cdot \biasterm(s)}\\
    &\leq \calO\prn*{H^4 A\ln(T)} + \E\brk*{\sum_{t = 1}^T\sum_s q^{\pi}(s) b_t(s)} + \E\brk*{\frac{1}{H} \sum_{t = 1}^T\sum_s q^{\pi}(s)\sum_a \pi_t(a\mid s)B_t(s,a)}\\
    &\qquad + \E\brk*{\sum_s q^{\pi}(s) \cdot \biasterm(s)},
\end{align} 
where we used \cref{lem:regterm_full}.
Absorbing the $\calO(H^2T\delta)$ contribution of $\bar{\calE}_1$ into $\calO(H^4A\ln(T))$, we obtain
\begin{align} 
\Reg_T (\pi)
&\leq \calO\prn*{H^4 A\ln(T)} + 3\E\brk*{\sum_{t = 1}^T V^{\overP^B_t, \pi_t}(s_0;b_t)} + \E\brk*{\sum_s q^{\pi}(s) \cdot \biasterm(s)}.
\end{align}
Here, by \cref{lem:Vbt_full} and Jensen's inequality, with $\delta=1/T^2$, we obtain
 \begin{align}
    &
\E\brk*{\sum_{t=1}^T V^{\overP^B_t,\pi_t}(s_0;b_t)}
\lesssim
\sqrt{\ln(T)}\sum_{s,a}\sqrt{\E\brk*{\sum_{t=1}^T q^{\pi_t}(s)\zeta_t(s,a)}} + S^2A\iota^2+ \calO(\delta)\times \calO(HT),
\end{align} 
where we used the trivial bound
$\sum_{t=1}^T V^{\overP^B_t,\pi_t}(s_0;b_t)\leq \calO(HT)$
on $\bar{\calE}$ from \cref{lem:bound_B_full}.

Then, using $\delta = \frac{1}{T^2}$, we have
 \begin{align}
\Reg_T (\pi)&\lesssim H^4A\ln(T)+S^2A\iota^2 + \sqrt{\ln(T)}\sum_{s,a}\sqrt{\E\brk*{\sum_{t=1}^T q^{\pi_t}(s)\zeta_t(s,a)}}\\
    &\qquad + \E\brk*{\sum_s q^{\pi}(s) \cdot \biasterm(s)}, \label{eq:regret_with_qb}
\end{align} 

By \cref{eq:regret_with_qb,lem:biasterm_full} and Jensen's inequality, together with the trivial bound $\calO(H^2T)$ on $\bar{\calE}$ and the choice $\delta=\calO(1/T^2)$, we obtain, for any policy $\pi$,
 \begin{align}
    \Reg_T(\pi)
    &\lesssim H^4A\ln(T)+S^2A\iota^2\\
    &\qquad + \sqrt{\ln(T)}\sum_{s,a}\sqrt{\E\brk*{\sum_{t=1}^T q^{\pi_t}(s)\zeta_t(s,a)}}\\
    &\qquad + \min\Biggl\{\sqrt{S^2A\iota^2\E\brk*{\Qtrans^{\pi_{1:T}}(\ell)}},\\
    &\qquad\qquad\qquad \sum_{h = 0}^{H-1}\sqrt{H^2\abs*{S_{h}}\abs*{S_{h + 1}}A\iota^2 \E\brk*{\sum_{t=1}^T\sum_{(s,a) \in \calS_h \times \calA}\brk*{q^{\pi_t}(s,a) - q^{\pi}(s,a)}_{+}}}\Biggr\}\\
    &\qquad + HS^{2}A\iota^2 + \calO(\delta)\times\calO(H^2T)\\
    &\lesssim A(H^4+HS^2)\iota^2 + \sum_{s,a}\sqrt{\ln(T)\E\brk*{\sum_{t=1}^T q^{\pi_t}(s) \zeta_t(s,a)}}\\
    &\qquad + \min\Biggl\{\sqrt{S^2A\iota^2\E\brk*{\Qtrans^{\pi_{1:T}}(\ell)}},\\
    &\qquad\qquad\qquad \sum_{h = 0}^{H-1}\sqrt{H^2\abs*{S_{h}}\abs*{S_{h + 1}}A\iota^2
    \E\brk*{\sum_{t=1}^T\sum_{(s,a) \in \calS_h \times \calA}\brk*{q^{\pi_t}(s,a) - q^{\pi}(s,a)}_{+}}}\Biggr\}.
\end{align} 
This completes the proof.
\end{proof}

\subsection{Proof for the adversarial regime}
\begin{lem}[full-information setting]
\label{lem:zeta_to_max}
It holds that
\begin{align}
    &\sum_{s,a}\sqrt{\ln(T)\E\brk*{\sum_{t=1}^T q^{\pi_t}(s) \zeta_t(s,a)}} \\
    &\lesssim \sqrt{HSA \ln(T) \E\brk*{\sum_{t=1}^T \sum_{s,a} q^{\pi_t}(s, a) \max_{\tilP\in\calP_t} Q^{\tilP, \pi_t}(s, a;\prn*{\ell_t - m_t}^2)}}.
\end{align}
\end{lem}
\begin{proof}
By the definition of $\zeta_t(s,a)$, we have
\begin{align}   
        &\sum_{s,a}\sqrt{\ln(T)\E\brk*{\sum_{t=1}^T q^{\pi_t}(s) \zeta_t(s,a)} } \\
        &= \sum_{s,a}\sqrt{\ln(T)\E\brk*{\sum_{t=1}^T q^{\pi_t}(s,a)(1 - \pi_t(a \mid s))\sum_{b}\pi_t(b \mid s)\prn*{Q^{\underP^\ell_t, \pi_t}(s, b;\ell_t) - Q^{\underP^m_t, \pi_t}(s, b;m_t)}^2}}\\
        &\leq \sqrt{SA\ln(T)\E\brk*{\sum_{t=1}^T\sum_{s} q^{\pi_t}(s)\sum_{b}\pi_t(b \mid s)\prn*{Q^{\underP^\ell_t, \pi_t}(s, b;\ell_t) - Q^{\underP^m_t, \pi_t}(s, b;m_t)}^2}}, \label{eq:common_bound_adv_full}
\end{align}
where the inequality follows from the Cauchy–Schwarz inequality.

We next bound $Q^{\underP^\ell_t, \pi_t}(s, b;\ell_t) - Q^{\underP^m_t, \pi_t}(s, b;m_t)$.
\begin{align}
    &Q^{\underP^\ell_t, \pi_t}(s, a;\ell_t) - Q^{\underP^m_t, \pi_t}(s, a;m_t)\\
    &= \prn*{ Q^{\underP^\ell_t, \pi_t}(s, a;\ell_t) - Q^{\underP^m_t, \pi_t}(s, a;\ell_t)} + \prn*{Q^{\underP^m_t, \pi_t}(s, a;\ell_t) - Q^{\underP^m_t, \pi_t}(s, a;m_t)}\\
    &\leq 0 + Q^{\underP^m_t, \pi_t}(s, a;\ell_t - m_t) \\
    &\leq \max_{\tilP \in \calP_t} Q^{\tilP, \pi_t}(s,a;\abs*{\ell_t - m_t}).
\end{align}
Similarly,
\begin{align}
    Q^{\underP^\ell_t, \pi_t}(s, a;\ell_t) - Q^{\underP^m_t, \pi_t}(s, a;m_t) \geq -\max_{\tilP \in \calP_t} Q^{\tilP, \pi_t}(s,a;\abs*{\ell_t - m_t})
\end{align}
Here, for any transition kernel $\tilP$, we have
\begin{align}
Q^{\tilP,\pi_t}(s,a;\abs{\ell_t-m_t})
&= \sum_{s',a'} q^{\tilP,\pi_t}(s',a'\mid s,a)\abs{\ell_t(s',a')-m_t(s',a')} \\
&\leq \sqrt{\sum_{s',a'} q^{\tilP,\pi_t}(s',a'\mid s,a)}
\sqrt{\sum_{s',a'} q^{\tilP,\pi_t}(s',a'\mid s,a)(\ell_t(s',a')-m_t(s',a'))^2} \\
&\leq \sqrt{H Q^{\tilP,\pi_t}(s,a;(\ell_t-m_t)^2)}.
\end{align}

Therefore,
\begin{align}
    \prn*{Q^{\underP^\ell_t, \pi_t}(s, a;\ell_t) - Q^{\underP^m_t, \pi_t}(s, a;m_t)}^2
    &\leq \prn*{\max_{\tilP \in \calP_t} Q^{\tilP, \pi_t}(s,a;\abs*{\ell_t - m_t})}^2 \\
    &\leq  H\max_{\tilP \in \calP_t}Q^{\tilP, \pi_t}(s,a;\prn*{\ell_t - m_t}^2) \label{eq:max_square_to_H}
\end{align}
Combining this with \cref{eq:common_bound_adv_full} completes the proof.
\end{proof}

\begin{thm}[full-information setting]
\label{thm:app_full_adv}
For any comparator policy $\pi$, \cref{alg:OFTRL_unknown} guarantees
     \begin{align}
        \Reg_T(\pi)
        &\lesssim \sqrt{H^2SA  \ln(T) \min\set*{L(\pi),HT-L(\pi),Q_\infty,V_1}}\\
        &\qquad + \sqrt{S^2A\ln^2(T)\E\brk*{\Qtrans^{\pi_{1:T}}(\ell)}} + A(H^4+HS^2)\ln^2(T).
    \end{align} 
\end{thm}
\begin{proof}
By \cref{lem:regret_to_zetat,lem:zeta_to_max,lem:max_to_value} and the choice $\delta=1/T^2$, which implies $\iota\lesssim \ln(T)$, we have, for any policy $\pi$,
  \begin{align}
    \Reg_T(\pi) &\lesssim A(H^4+HS^2)\ln^2(T)\\
    &\qquad + \sqrt{HSA  \ln(T) \E\brk*{\sum_{t=1}^T \sum_{s,a} q^{\pi_t}(s, a) \max_{\tilP\in\calP_t} Q^{\tilP, \pi_t}(s, a;\prn*{\ell_t - m_t}^2)}}\\
    &\qquad + \sqrt{S^2A\ln^2(T)\E\brk*{\Qtrans^{\pi_{1:T}}(\ell)}}\\
    &\lesssim A(H^4+HS^2)\ln^2(T) + \sqrt{H^2SA  \ln(T) \E\brk*{\sum_{t=1}^TV^{\pi_t}(s_0;\prn*{\ell_t - m_t}^2)}}\\
    &\qquad + \sqrt{S^2A\ln^2(T)\E\brk*{\Qtrans^{\pi_{1:T}}(\ell)}}.
\end{align} 
Combining this with \cref{lem:general_predict_result},  we obtain
 \begin{align}
    \Reg_T(\pi) &\lesssim \sqrt{H^2SA  \ln(T) \min\set*{L(\pi) + \Reg_T(\pi),HT-L(\pi)-\Reg_T(\pi),Q_\infty,V_1}}\\
    &\qquad + \sqrt{S^2A\ln^2(T)\E\brk*{\Qtrans^{\pi_{1:T}}(\ell)}} + A(H^4+HS^2)\ln^2(T). \label{eq:adv_bound_mid}
\end{align} 
Now, we derive the first-order bound. Assume $\Reg_T(\pi)\geq0$. By \cref{eq:adv_bound_mid}, there exists an absolute constant $c>0$ such that
 \begin{align}
    \Reg_T(\pi) 
    &\leq cA(H^4+HS^2)\ln^2(T) + c\sqrt{H^2SA  \ln(T) L(\pi)}\\
    &\qquad +c\sqrt{H^2SA  \ln(T) \Reg_T(\pi)}\\
    &\qquad + c\sqrt{S^2A\ln^2(T)\E\brk*{\Qtrans^{\pi_{1:T}}(\ell)}}\\
    &\leq cA(H^4+HS^2)\ln^2(T) + c\sqrt{H^2SA  \ln(T) L(\pi)}\\
    &\qquad +\frac{c^2}{2}H^2SA  \ln(T) + \frac{1}{2}\Reg_T(\pi)\\
    &\qquad + c\sqrt{S^2A\ln^2(T)\E\brk*{\Qtrans^{\pi_{1:T}}(\ell)}},
\end{align} 
where the last inequality follows from the AM--GM inequality.
Thus, we obtain
 \begin{align}
    \Reg_T(\pi) 
    &\lesssim A(H^4+HS^2)\ln^2(T) + \sqrt{H^2SA  \ln(T) L(\pi)}\\
    &\qquad + \sqrt{S^2A\ln^2(T)\E\brk*{\Qtrans^{\pi_{1:T}}(\ell)}}. \label{eq:first_order_bound1}
\end{align} 
Furthermore, we derive the bound in terms of $HT-L(\pi)$. If
$\Reg_T(\pi)<0$, the desired upper bound is trivial. Otherwise,
$\Reg_T(\pi)\ge 0$ and \cref{eq:adv_bound_mid} gives
 \begin{align}
    \Reg_T(\pi)
    &\lesssim
    A(H^4+HS^2)\ln^2(T)
    + \sqrt{H^2SA\ln(T)\prn*{HT-L(\pi)-\Reg_T(\pi)}}
    \\ &\qquad + \sqrt{S^2A\ln^2(T)\E\brk*{\Qtrans^{\pi_{1:T}}(\ell)}} \\
    &\le
    A(H^4+HS^2)\ln^2(T)
    + \sqrt{H^2SA\ln(T)\prn*{HT-L(\pi)}} 
    \\ &\qquad + \sqrt{S^2A\ln^2(T)\E\brk*{\Qtrans^{\pi_{1:T}}(\ell)}} .
    \label{eq:first_order_bound2}
\end{align} 

Combining \cref{eq:adv_bound_mid,eq:first_order_bound1,eq:first_order_bound2} completes the proof.
\end{proof}

\subsection{Proof for the stochastic regime with adversarial corruption}
\begin{thm}[full-information setting]
\label{thm:app_full_corruption}
Under the stochastic regime with adversarial corruption, \cref{alg:OFTRL_unknown} guarantees that, for any policy $\pi$, 
 \begin{align}
    \Reg_T(\pi) \lesssim U + \sqrt{U\calC} + A(H^4+HS^2)\ln^2(T),
\end{align} 
where $U = \frac{H^2S^2A\ln^2(T)}{\Delta_{\min}}$.
\end{thm}
\begin{proof}
Since $\Reg_T(\pi) \leq \Reg_T(\pio)$ for all $\pi$, it suffices to bound $\Reg_T(\pio)$. 
From \cref{lem:regret_to_zetat} and the choice $\delta=1/T^2$, which implies $\iota\lesssim \ln(T)$, we have
 \begin{align}
     \Reg_T(\pio)
    &\lesssim A(H^4+HS^2)\ln^2(T) + 2H\ln(T)\sum_{s}\sum_{a\neq \pist(s)}\sqrt{\E\brk*{\sum_{t=1}^T q^{\pi_t}(s,a)}}\\
    &\qquad + \sum_{h = 0}^{H-1}\sqrt{H^2\abs*{S_{h}}\abs*{S_{h + 1}}A\ln^2(T) \E\brk*{\sum_{t=1}^T\sum_{(s,a) \in \calS_h \times \calA}\brk*{q^{\pi_t}(s,a) - q^{\pio}(s,a)}_{+}}}.
\end{align} 
Applying \cref{lem:general_self_bounding,lem:self_bound_layer_positive}, for any $\alpha > 0$, we get
 \begin{align}
    &
    \Reg_T(\pio)
    \\
    &\leq \alpha(2\Reg_T(\pio) + 6\calC) + \calO\prn*{\sum_{s}\sum_{a\neq \pist(s)}\frac{H^2\ln^2(T)}{\alpha\Delta(s,a)} + A(H^4+HS^2)\ln^2(T)}\\
    &\qquad + \calO\prn*{\frac{\prn*{\sum_{h = 0}^{H-1}\sqrt{H^2\abs*{S_{h}}\abs*{S_{h + 1}}A\ln^2(T)}}^2}{4\alpha\Delta_{\min}}}.
\end{align} 
Here, 
\begin{align}
    \prn*{\sum_{h = 0}^{H-1}\sqrt{H^2\abs*{S_{h}}\abs*{S_{h + 1}}A\ln^2(T)}}^2 
    &\leq \prn*{\sum_{h = 0}^{H-1}\frac{1}{2}\prn*{\abs*{S_{h}} + \abs*{S_{h + 1}}}\sqrt{H^2A\ln^2(T)}}^2 \tag{by the AM--GM inequality}\\
    &= \prn*{S\sqrt{H^2A\ln^2(T)}}^2 = H^2S^2A\ln^2(T). \label{eq:coeff_gap2}
\end{align}
Then, using \cref{eq:coeff_gap2} and choosing $\alpha = \min\set*{\frac{1}{4}, \sqrt{\frac{U}{\calC}}}$ with $U = \frac{H^2S^2A\ln^2(T)}{\Delta_{\min}}$, we get
 \begin{align}
    \Reg_T(\pio) \lesssim U + \sqrt{U\calC} + A(H^4+HS^2)\ln^2(T).
\end{align} 
Therefore, since $\Reg_T(\pi)\leq \Reg_T(\pio)$ for all $\pi$ by the definition of $\pio$, the desired bound follows.

 \end{proof}

\section{Regret analysis for the bandit setting (\cref{thm:main_bandit})}
\label{app:bandit_proofs}
 In this section, we prove the best-of-both-worlds results for the bandit setting. We present the bound for the adversarial regime in \cref{thm:app_bandit_adv} and the bound for the stochastic regime with adversarial corruption in \cref{thm:app_bandit_corruption}. Together, these results establish \cref{thm:main_bandit}.

Throughout this section, virtual episodes are inserted into the episode sequence, shifting the indices of subsequent real episodes. 

We denote by $\calT_r=\{t\in[T]:Y_t=1\}$ and $\calT_v=\{t\in[T]:Y_t=0\}$ the sets of real and virtual episodes, with sizes $\abs{\calT_r}$ and $\abs{\calT_v}$, respectively.

\subsection{Auxiliary lemmas}
Following the full-information setting, \cref{lem:dilated_bonus} reduces the regret analysis to proving \cref{eq:key_lemma_ineq}.
With the choice of $b_t$ in \cref{eq:bonus_term_bandit}, we decompose the left-hand side of \cref{eq:key_lemma_ineq} as
\begin{align}
    &\sum_s q^{\pi}(s) \sum_{t,a} \prn*{\pi_t(a\mid s) -  \pi(a\mid s) } \prn*{Q^{\pi_t}(s,a;\ell_t) -  B_t(s,a)}\\
    &= \sum_s q^{\pi}(s)  \underbrace{\sum_{t,a}  \prn*{\pi_t(a\mid s) -  \pi(a\mid s)} \prn*{\hat{Q}_t(s,a) -  B_t(s,a)}}_{\regterm(s)}\\
    &\qquad + \sum_s q^{\pi}(s) \underbrace{\sum_{t,a} \prn*{\pi_t(a\mid s) -  \pi(a\mid s)} \prn*{Q^{\pi_t}(s,a;\ell_t) -  \hat{Q}_t(s,a) + Q^{\underP^m_t, \pi_t}(s, a; m_t) - Q^{\pi_t}(s, a; m_t)}}_{\biasterm(s)}\\
    &\qquad + \sum_s q^{\pi}(s) \underbrace{\sum_{t,a} \prn*{\pi_t(a\mid s) -  \pi(a\mid s)} \prn*{Q^{\pi_t}(s, a; m_t) - Q^{\underP^m_t, \pi_t}(s, a; m_t)}}_{\errorterm(s)}.\label{eq:decompose_bandit}
\end{align}

\begin{lem}[bandit setting]\label{lem:qC_size_bound}
It holds that
 \begin{align}
    q_t(s)\pi_t(a\mid s)C_t(s,a)&\leq 6HS, \qquad \E_t\brk*{C_t(s,a)}\leq 2HS(t+1). \label{eq:qC_size_bound}
\end{align} 
for any episode $t$ and state-action pair $(s,a)$ on every history.
\end{lem}
 \begin{proof}
We first prove a conditional occupancy property. For any $\tilP\in\calP_t$, any states $s,s'$ with $h(s')>h(s)$, and any actions $a,a'$, we have
\begin{align}
    q_t(s)\pi_t(a\mid s)q^{\tilP,\pi_t}(s',a'\mid s,a)
    &\leq q_t(s')\pi_t(a'\mid s').
    \label{eq:conditional_occupancy_domination}
\end{align}
Hence,
\begin{align}
    \overq_t^{\pi_t}(s')
    &\geq \overq_t^{\pi_t}(s)\pi_t(a\mid s)q^{\tilP,\pi_t}(s'\mid s,a).
\end{align}
Using $q_t(x)=\overq_t^{\pi_t}(x)+\gamma_t$ and $\pi_t(a\mid s)q^{\tilP,\pi_t}(s'\mid s,a)\leq 1$, we obtain
\begin{align}
    q_t(s')\pi_t(a'\mid s')
    &=\prn*{\overq_t^{\pi_t}(s')+\gamma_t}\pi_t(a'\mid s')\\
    &\geq \prn*{\overq_t^{\pi_t}(s)+\gamma_t}\pi_t(a\mid s)q^{\tilP,\pi_t}(s'\mid s,a)\pi_t(a'\mid s')\\
    &=q_t(s)\pi_t(a\mid s)q^{\tilP,\pi_t}(s',a'\mid s,a),
\end{align}
which proves \cref{eq:conditional_occupancy_domination}. 

Now, by the definition of $C_t$ and \cref{eq:conditional_occupancy_domination} with $\tilP=\overP^C_t$,
\begin{align}
    &q_t(s)\pi_t(a\mid s)C_t(s,a)\\
    &\leq 2q_t(s)\pi_t(a\mid s)\sum_{s',a'}q^{\overP^C_t,\pi_t}(s',a'\mid s,a)\rho_t(s')\frac{\I_t(s',a')D_{t,h(s')}}{q_t(s')\pi_t(a'\mid s')}\\
    &\qquad+2Hq_t(s)\pi_t(a\mid s)\sum_{s',a'}q^{\overP^C_t,\pi_t}(s',a'\mid s,a)\rho_t(s')^2\\
    &\leq 2\sum_{s',a'}\rho_t(s')\I_t(s',a')D_{t,h(s')}+2H\sum_{s',a'}q_t(s')\pi_t(a'\mid s')\rho_t(s')^2. \label{eq:qC_bound_middle}
\end{align}
Since $\rho_t(s')\leq 1$, $D_{t,h(s')}\leq H$, 
\begin{align}
    2\sum_{s',a'}\rho_t(s')\I_t(s',a')D_{t,h(s')}
    &\leq 2H^2.
\end{align}
Moreover, using $\rho_t(s')\leq 1$ and $\sum_{s',a'}q_t(s')\pi_t(a'\mid s')\leq S + S\gamma_t \leq 2S$,
\begin{align}
    2H\sum_{s',a'}q_t(s')\pi_t(a'\mid s')\rho_t(s')^2
    &\leq 4HS.
\end{align}
Combining the last three displays gives
\begin{align}
    q_t(s)\pi_t(a\mid s)C_t(s,a)
    &\leq 6HS.
\end{align}
This proves the first inequality in \cref{eq:qC_size_bound}.

We next prove the second inequality. Let $\overP^C_t$ be a measurable transition kernel that attains the maximum in \cref{eq:def_C}. Then, it holds that
 \begin{align}
    C_t(s,a)&\leq 2\sum_{s',a'}q^{\overP^C_t,\pi_t}(s',a'\mid s,a)c_t(s',a'). 
\end{align} 
Since $q_t(s)\geq1/t$, $\rho_t(s)\leq1$, and $D_{t,h(s)}\leq H$,

\begin{align}
 \E_t\brk*{c_t(s,a)}
 &\leq \rho_t(s)\frac{Hq^{\pi_t}(s)}{q_t(s)}+H\rho_t(s)^2
 \leq H(t+1).
\end{align}

By using $q^{\overP^C_t,\pi_t}(s',a'\mid s,a)\leq\pi_t(a'\mid s')$ for $h(s')>h(s)$, we get
\begin{align}
 \E_t\brk*{C_t(s,a)}\leq2HS(t+1),
\end{align}
which proves the second inequality. \end{proof}

\begin{lem}[bandit setting]
\label{lem:bound_eta_pi_B}
The variables $b_t(s)$ in \cref{eq:bonus_term_bandit} and $B_t(s,a)$ in \cref{def:dilated_bonus} satisfy
 \begin{align}
    \eta_t(s,a)\pi_t(a\mid s)B_t(s,a) &\leq \frac{1}{6H} , \quad B _{t} ( s , a ) \leq \frac{ 8 \sqrt{HS^3A}}{\gamma_t} \label{eq:bound_eta_pi_B}
\end{align} 
for any episode $t$ and state-action pair $(s,a)$.
\end{lem}

\begin{proof}
Let $R_t = \max_{s,a}\frac{\eta_t(s,a)}{q_t(s)}$.

We first consider the case when $t$ is a real episode.
In real episodes, by the definition of $b_t$,
 \begin{align}
    b_t(s) 
    &= 6\sum_a  \prn*{\frac{1}{\eta_{t+1}(s, a)} - \frac{1}{\eta_t(s, a)}}\ln(T) + 3 \sum_a \eta_t(s,a)\pi_t(a\mid s)^2 C_t(s,a)^2 
\end{align} 
For the first term, the learning-rate update gives
 \begin{align}
    &6\sum_a  \prn*{\frac{1}{\eta_{t+1}(s, a)}- \frac{1}{\eta_t(s, a)}}\ln(T) \\
    &= \frac{6}{q_t(s)^2}\sum_a \eta_t(s,a)\zeta_t(s,a) \\
    &= \frac{6}{q_t(s)^2}\sum_a \eta_t(s,a)\prn*{\I_t(s,a)-\pi_t(a\mid s)\I_t(s)}^2\prn*{L_{t,h(s)}-M_{t,h(s)}}^2 \\
    &\leq
     \frac{6H^2}{q_t(s)^2}\max_a\eta_t(s,a)
     \sum_a\prn*{\I_t(s,a)-\pi_t(a\mid s)\I_t(s)}^2 \\
    &\leq
     \frac{12H^2}{q_t(s)^2}\max_a\eta_t(s,a) \\
    &\leq
     \frac{12H^2}{q_t(s)}
     \max_a\frac{\eta_t(s,a)}{q_t(s)} . \label{eq:bonus_term_upper_real1}
\end{align} 
For the second term, by \cref{lem:qC_size_bound},
\begin{align}
    3 \sum_a \eta_t(s,a)\pi_t(a\mid s)^2 C_t(s,a)^2
    & = 3\sum_{a} \frac{\eta_t(s,a)q_t(s)^2\pi_t(a\mid s)^2 C_t(s,a)^2}{q_t(s)^2}\\
    &\leq \frac{108H^2S^2\sum_a \eta_t(s,a)}{q_t(s)^2} \tag{by \cref{lem:qC_size_bound}}\\ 
    &\leq \frac{108H^2S^2A}{q_t(s)}\max_{a}\frac{\eta_t(s, a)}{q_t(s)} . \label{eq:bonus_term_upper_real2}
\end{align}
Combining \cref{eq:bonus_term_upper_real1,eq:bonus_term_upper_real2}, we obtain
\begin{align}
    b_t(s) &\leq \frac{12H^2 + 108H^2S^2A}{q_t(s)}\max_{a}\frac{\eta_t(s, a)}{q_t(s)} \leq \frac{120H^2S^2A}{q_t(s)}\max_{a}\frac{\eta_t(s, a)}{q_t(s)} \label{eq:bonus_term_upper_real}
\end{align}
Using \cref{eq:bonus_term_upper_real,eq:dilated_bonus_upper}, we have
\allowdisplaybreaks
 \begin{align}
    B_t(s,a) &\leq 3\sum_{h = h(s)}^{H-1}\sum_{s' \in \calS_h} q^{\overP^B_t, \pi_t}(s'\mid s,a)b_t(s') \\
    &\leq 360H^2S^2AR_t \prn*{\sum_{h = h(s)}^{H-1}\sum_{s' \in \calS_h} q^{\overP^B_t, \pi_t}(s'\mid s,a)\frac{1}{q_t(s')}} \\
    &\leq 360H^2S^2AR_t \prn*{\sum_{h = h(s)}^{H-1}\sum_{s' \in \calS_h} q^{\overP^B_t,\pi_t}(s'\mid s,a)\frac{1}{\overq^{\pi_t}_t(s)\pi_t(a\mid s)q^{\overP^B_t,\pi_t}(s'\mid s,a) + \gamma_t}} \\
    &\leq 360H^2S^2AR_t \prn*{\sum_{h = h(s)}^{H-1}\sum_{s' \in \calS_h}\frac{1}{\overq^{\pi_t}_t(s)\pi_t(a\mid s) + \gamma_t}} \\
    &\leq 360H^2S^3AR_t\cdot\frac{1}{\overq^{\pi_t}_t(s)\pi_t(a\mid s) + \gamma_t} \label{eq:bound_B_real}\\
    &\leq \frac{ 8 \sqrt{HS^3A}}{\gamma_t} , 
\end{align} 
where in the last line we used $R_t \leq \frac{1}{50\sqrt{H^3S^3A}}$ that holds in real episodes.

By using \cref{eq:bound_B_real}, we also have
\begin{align}
    \eta_t(s,a)\pi_t(a \mid s)B_t(s,a) 
    &\leq 360H^2S^3AR_t\cdot\frac{\eta_t(s,a)\pi_t(a\mid s)}{\overq^{\pi_t}_t(s)\pi_t(a\mid s) + \gamma_t} \\
    &\leq 360H^2S^3AR_t\cdot\frac{\eta_t(s,a)}{q_t(s)}\\ 
    &\leq 360H^2S^3AR_t^2\\ 
    &\leq \frac{1}{6H},
\end{align}
where we used $\overq^{\pi_t}_t(s)\pi_t(a\mid s)+\gamma_t \geq q_t(s)\pi_t(a\mid s)$, $\frac{\eta_t(s,a)}{q_t(s)}\leq R_t$.

We next consider the case when $t$ is a virtual episode.
In a virtual episode, only the single pair $(s_t^\dagger,a_t^\dagger)$ is updated, and thus
 \begin{align}
    b_t(s) &= 6\sum_a  \prn*{\frac{1}{\eta_{t+1}(s, a)} - \frac{1}{\eta_t(s, a)}}\ln(T) \\
    &\leq \sum_{a} \frac{\ind\{(s^\dagger_t, a^\dagger_t) = (s,a)\}}{54\eta_t(s,a)H} \\
    &= \sum_{a} \frac{\ind\{(s^\dagger_t, a^\dagger_t) = (s,a)\}}{54Hq_t(s)} \cdot \frac{1}{\max_{s', a'}\frac{\eta_t(s', a')}{q_t(s')}} \tag{since $(s^\dagger_t, a^\dagger_t) \in \argmax_{s,a} \frac{\eta_t(s,a)}{q_t(s)}$}\\
    &= \frac{\ind\{s^\dagger_t = s\}}{q_t(s)}\cdot \frac{1}{54H\max_{s', a'}\frac{\eta_t(s', a')}{q_t(s')}}. \label{eq:bonus_term_upper_virtual}
\end{align} 

Using \cref{eq:bonus_term_upper_virtual,eq:dilated_bonus_upper}, we have
\begin{align}
    B_t(s,a) &\leq 3\sum_{h = h(s)}^{H-1}\sum_{s' \in \calS_h} q^{\overP^B_t,\pi_t}(s'\mid s,a)b_t(s')\\ 
    &\leq \frac{1}{18HR_t} {\sum_{h = h(s)}^{H-1}\sum_{s' \in \calS_h} q^{\overP^B_t,\pi_t}(s'\mid s,a)\frac{\ind\set{s^\dagger_t = s'}}{q_t(s')}}\\
    &\leq \frac{1}{18HR_t} {\sum_{h = h(s)}^{H-1}\sum_{s' \in \calS_h} q^{\overP^B_t,\pi_t}(s'\mid s,a)\frac{\ind\set{s_t^\dagger = s'}}{\overq^{\pi_t}_t(s)\pi_t(a\mid s)q^{\overP^B_t,\pi_t}(s'\mid s,a) + \gamma_t}}\\
    &\leq \frac{1}{18HR_t} {\sum_{h = h(s)}^{H-1}\sum_{s' \in \calS_h}\frac{\ind\set{s_t^\dagger = s'}}{\overq^{\pi_t}_t(s)\pi_t(a\mid s) + \gamma_t}}\\
    &\leq \frac{1}{18HR_t} \frac{1}{\overq^{\pi_t}_t(s)\pi_t(a\mid s) + \gamma_t} \label{eq:bound_B_virtual}\\
    & \leq \frac{3\sqrt{HS^3A}}{\gamma_t},
\end{align}
where in the last inequality we used $R_t > \frac{1}{50\sqrt{H^3S^3A}}$ in a virtual episode.
By using \cref{eq:bound_B_virtual}, we also have
\begin{align}
    \eta_t(s,a)\pi_t(a \mid s)B_t(s,a) 
    &\leq \frac{1}{18HR_t} \frac{\eta_t(s,a)\pi_t(a\mid s)}{\overq^{\pi_t}_t(s)\pi_t(a\mid s) + \gamma_t}\\ 
    &\leq \frac{1}{18HR_t} \frac{\eta_t(s,a)}{q_t(s)} \\ 
    &\leq \frac{1}{18H}\leq \frac{1}{6H},
\end{align}
where we used $\overq^{\pi_t}_t(s)\pi_t(a\mid s)+\gamma_t \geq q_t(s)\pi_t(a\mid s)$, $\frac{\eta_t(s,a)}{q_t(s)}\leq R_t$.
This completes the proof.
\end{proof}

The following two lemmas are based on the \citet[Lemmas E.5 and E.6]{li2026data}. We adapt the proof to our choice of
$\gamma_t$ and our virtual-episode rule, and include the details for completeness.
\begin{lem}[bandit setting]\label{lem:learning_rate_policy}
 Suppose that the learning rates are updated according to \mbox{\cref{eq:eta_update_policy}}  . Then,  it holds
    \begin{align}
        \eta_t(s,a) \leq \frac{2\sqrt{\ln(T)}}{\sqrt{\sum_{\tau\leq t: \tau\in\mathcal{T}_r} \frac{\zeta_\tau(s,a)}{q_\tau(s)^2} }}
    \end{align}
for any   episode $t$ and state-action pair $(s,a)$.
  \end{lem}
\begin{proof}
    Let  $\phi_t(s,a)=\frac{\zeta_t(s,a)}{q_t(s)^2\ln(T)}$   in real episodes and  $\phi_t(s,a)=\frac{\I\{(s_t^\dagger,a_t^\dagger)=(s,a)\}}{324\eta_t(s,a)^2H\ln(T)}$   in virtual episodes. Then the update rule of learning rates can be written as
    \begin{align}
        \frac{1}{\eta_{t+1}(s,a)} = \frac{1}{\eta_t(s,a)} + \eta_t(s,a) \phi_t(s,a).
    \end{align}
    To apply \cref{lem:eta_lemma_normal}, it suffices to show that $\phi_t(s,a)\le \frac{1}{\eta_t(s,a)^2}$. 
    This is clear for virtual episodes. For real episodes, 
     \begin{align}
        \phi_t(s,a)\eta_t(s,a)^2 
        = \frac{\eta_t(s,a)^2\zeta_t(s,a)}{q_t(s)^2\ln(T)}
        \leq \frac{H^2}{\ln(T)}\prn*{\frac{\eta_t(s,a)}{q_t(s)}}^2
        \leq \frac{H^2}{\ln(T)}\cdot \frac{1}{50^2H^3S^3A} 
        \leq 1,
    \end{align}
  
    which follows from $\zeta_t(s,a) \leq H^2$ and $\frac{\eta_t(s,a)}{q_t(s)}\leq \frac{1}{50\sqrt{H^3S^3A}}$ in real episodes.

 Then, by \mbox{\cref{lem:eta_lemma_normal}}, we have
    \begin{align}
        \eta_t(s,a) \leq \frac{2}{\sqrt{\sum_{\tau\leq t} \phi_\tau }}
        \leq \frac{2\sqrt{\ln(T)}}{\sqrt{\sum_{\tau\leq t: \tau\in\mathcal{T}_r} \frac{\zeta_\tau(s,a)}{q_\tau(s)^2} }}.
    \end{align}
  
 \end{proof}

 In the following lemma only, we use $T$ to denote the number of real episodes rather than the total number of episodes. 
 \begin{lem}[bandit setting]\label{lem:bound_virtual_episode}
 When the number of real episodes is $T$, the number of virtual episodes $\abs{\mathcal{T}_v}$ is upper bounded by
\begin{align}
        \abs{\mathcal{T}_v} \lesssim  HSA \ln^2(T).
    \end{align}
  
\end{lem}

 \begin{proof}
 By the definition of virtual episodes, whenever $t\in\mathcal{T}_v$, there exists a pair $(s,a)$ such that $\frac{\eta_t(s,a)}{q_t(s)} > \frac{1}{50\sqrt{H^3S^3A}}$.
Moreover, in virtual episodes, the corresponding learning rate will shrink by a factor of $\prn*{1+\frac{1}{324H\ln (\abs{\calT_r}+\abs{\calT_v})}}$ for a state-action   pair  $(s_t^\dagger,a_t^\dagger)$.

Therefore, for each fixed $(s,a)$, a virtual update on this pair can occur only while
\begin{align}
    \eta_t(s,a)>\frac{1}{50\sqrt{H^3S^3A}(\abs{\calT_r}+\abs{\calT_v})}
\end{align}
Thus, the number of virtual updates on this pair is at most the number of multiplicative shrink steps needed to reduce $\eta_t(s,a)$ from $\eta_1$ to $1/(50\sqrt{H^3S^3A}(\abs{\calT_r}+\abs{\calT_v}))$. Hence, \begin{align}
    \abs{\calT_v}
    &\lesssim
    SA \frac{\ln\prn*{\eta_1\sqrt{H^3S^3A}(\abs{\calT_r}+\abs{\calT_v})}}{\ln\prn*{1+\frac{1}{324H\ln(\abs{\calT_r}+\abs{\calT_v})}}} \\
    &\lesssim HSA\ln^2(\abs{\calT_r}+\abs{\calT_v}),
\end{align}
  
where  we used $\ln(1+x)\ge x/2$ for $x\in(0,1]$, $\eta_1=1/(12H^3)$, and $H,S,A\leq \abs{\calT_r}$. This implies
\begin{align}
    \abs{\calT_v}\lesssim HSA\ln^2(\abs{\calT_r}).
\end{align}
Since $\abs{\calT_r}=T$, the desired bound follows.   \end{proof}

\begin{lem}[bandit setting]
\label{lem:bound_rho2}
Suppose that the event $\calE$ holds. Then, it holds that
\begin{align}
    \sum_{t \in \calT_r}\sum_{s} q^{\pi_t}(s)\rho_t(s)^2 &\lesssim S^3A\iota^2.
\end{align}
\end{lem}

\begin{proof}
Recall that $\rho_t(s)=\dfrac{q_t(s) - \underq^{\pi_t}_t(s)}{q_t(s)}$ and $q_t(s)=\overq_t^{\pi_t}(s)+\gamma_t$. Then, we have
\begin{align}
    \sum_{t \in \calT_r}\sum_{s} q^{\pi_t}(s)\rho_t(s)^2
    &\leq \sum_{t \in \calT_r}\sum_{s} q^{\pi_t}(s) \prn*{\frac{\overq^{\pi_t}_t(s) - \underq^{\pi_t}_t(s) + \gamma_t}{q_t(s)}}^2 \\
    &\leq \sum_{t \in \calT_r}\sum_{s} q^{\pi_t}(s) \prn*{2\prn*{\frac{\overq^{\pi_t}_t(s) - \underq^{\pi_t}_t(s)}{q_t(s)}}^2 + 2\prn*{\frac{\gamma_t}{q_t(s)}}^2} \\
    &\leq 2\sum_{t \in \calT_r}\sum_{s} q^{\pi_t}(s) \prn*{\frac{\overq^{\pi_t}_t(s) - \underq^{\pi_t}_t(s)}{q_t(s)}}^2 + 2\sum_{t \in \calT_r}\sum_{s} \gamma_t \\
\end{align}
We bound the first term. By \cref{lem:complicated_lemma}, we obtain
\begin{align}
    &\sum_{t \in \calT_r}\sum_{s} q^{\pi_t}(s) \prn*{\frac{\overq^{\pi_t}_t(s) - \underq^{\pi_t}_t(s)}{q_t(s)}}^2 \\
    &\leq 
    \sum_{t \in \calT_r}\sum_{s} \prn*{\overq^{\pi_t}_t(s) - \underq^{\pi_t}_t(s)} \frac{\overq^{\pi_t}_t(s) - \underq^{\pi_t}_t(s)}{q_t(s)}\\ 
    &\leq c\sqrt{HS^2A  \iota^2 \sum_{t \in \calT_r} \sum_s q^{\pi_t}(s)\prn*{\frac{\overq^{\pi_t}_t(s) - \underq^{\pi_t}_t(s)}{q_t(s)}}^2} + cS^3A\iota^2.
    \tag{for some absolute constant $c$}
\end{align}
Now use the fact that if $x \leq a\sqrt{x}+ b$ with $x,a,b \geq 0$, then $ x\leq 2a^2 + 2b$, we obtain
\begin{align}
    \sum_{t \in \calT_r}\sum_{s} q^{\pi_t}(s) \prn*{\frac{\overq^{\pi_t}_t(s) - \underq^{\pi_t}_t(s)}{q_t(s)}}^2
    &\lesssim HS^2A \iota^2 + S^3A\iota^2 
    \lesssim S^3A\iota^2.
\end{align}
We also have
\begin{align}
    2\sum_{t \in \calT_r}\sum_{s} \gamma_t \leq 2S(1+\ln(T))
\end{align}
Combining the above bounds gives 
\begin{align}
    \sum_{t \in \calT_r}\sum_{s} q^{\pi_t}(s)\rho_t(s)^2 &\lesssim S^3A\iota^2.
\end{align}
\end{proof}

\begin{lem}[bandit setting]
\label{lem:bound_rho2_upper}
Suppose that the event $\calE$ holds. Then, it holds that
\begin{align}
    \sum_{t \in \calT_r}\sum_{s} q_t(s)\rho_t(s)^2 &\lesssim S^3A\iota^2.
\end{align}
\end{lem}
\begin{proof}
We decompose
\begin{align}
    &\sum_{t \in \calT_r}\sum_{s'}q_t(s')\rho_t(s')^2\\
    &=\underbrace{\sum_{t \in \calT_r}\sum_{s'}q^{\pi_t}(s')\rho_t(s')^2}_{\term_a}
    +\underbrace{\sum_{t \in \calT_r}\sum_{s'}\prn*{\overq_t^{\pi_t}(s')-q^{\pi_t}(s')}\rho_t(s')^2}_{\term_b}+\underbrace{\sum_{t \in \calT_r}\sum_{s'}\gamma_t\rho_t(s')^2}_{\term_c}.
    \label{eq:bound_squared_Ctorho}
\end{align}
For $\term_a$, by \cref{lem:bound_rho2},
\begin{align}
    \term_a \lesssim S^3A\iota^2.
    \label{eq:term_a_rho_bound}
\end{align}
For $\term_b$, applying \cref{lem:complicated_lemma} with $g_t(s)=\rho_t(s)^2$ and \cref{lem:bound_rho2}, we obtain
\begin{align}
    \term_b
    &\lesssim \sqrt{HS^2A\iota^2\sum_{t \in \calT_r}\sum_s q^{\pi_t}(s)\rho_t(s)^4}+S^3A\iota^2\\
    &\lesssim \sqrt{HS^2A\iota^2\cdot S^3A\iota^2}+S^3A\iota^2\\
    &\lesssim S^3A\iota^2,
    \label{eq:term_b_rho_bound}
\end{align}
For $\term_c$, since $\rho_t(s')\leq 1$,
\begin{align}
    \term_c
    &\leq \sum_{t \in \calT_r}\sum_{s'}\gamma_t
    \lesssim S\ln(T).
    \label{eq:term_c_rho_bound}
\end{align}
Thus, we have
\begin{align}
    \sum_{t \in \calT_r}\sum_{s} q_t(s)\rho_t(s)^2 &\lesssim S^3A\iota^2.
\end{align}
\end{proof}

\subsection{Common regret analysis}

Now we upper bound the right-hand side of \cref{eq:decompose_bandit}.

\begin{lem}[bandit setting]
\label{lem:regterm_policy}
    For each state $s \in \calS$ and any comparator $\pi$, it holds that 
     \begin{align}
    \mathbb{E} \brk*{\regterm(s) } &\leq \calO\prn*{HS^2A\ln(T)} + \mathbb{E}\brk*{\sum_{t=1}^Tb_t(s)} + \mathbb{E}\brk*{\frac{1}{H}\sum_{t=1}^T\sum_a \pi_t(a \mid s) B_t(s,a)}. 
\end{align} 
\end{lem}

\begin{proof}
We will apply \cref{lem:OFTRL_log_barrier_policy} with $p_t = \pi_t(\cdot \mid s)$ and $\ell_t = \hat{Q}_t(s,a) - B_t(s,a)$ for each $s \in \calS$.
To do so, in what follows, we will check the conditions of \Cref{lem:OFTRL_log_barrier_policy}.
Let 
\begin{align}
x_t &= \inpr*{-\pi_t(\cdot \mid s) , \hatQ_t(s,\cdot) - Q^{\underP^m_t, \pi_t}(s,\cdot; m_t) + C_t(s,\cdot)Y_t}= -\dfrac{\I_t(s)(L_{t,h(s)} - M_{t,h(s)})}{q_t(s)}Y_t
\end{align}
and verify that for all $(s,a)$, $\eta_t(s,a)\pi_t(a\mid s)\prn*{\hat{Q}_t(s,a)-B_t(s,a) - Q^{\underP^m_t, \pi_t}(s,a; m_t) + x_t}\geq -{1}/{2}$.
Recall that in a virtual episode we have $Y_t=0$ and $\ell_t(s,a)=0$ for all state-action pairs $(s,a)$. 

Then, we have
\begin{align}
    &\eta_t(s,a)\pi_t(a\mid s)\prn*{\hat{Q}_t(s,a)-B_t(s,a)-Q^{\underP^m_t, \pi_t}(s,a; m_t) + x_t} \\ 
    &= \eta_t(s,a)\pi_t(a\mid s) \prn[\bigg]{\frac{\I_t(s,a)(L_{t,h(s)}-M_{t,h(s)})}{q_t(s)\pi_t(a\mid s)}Y_t-
    \frac{\I_t(s)(L_{t,h(s)}-M_{t,h(s)})}{q_t(s)}Y_t\\
    &\hspace{42mm} -B_t(s,a) - C_t(s,a)Y_t} \\
    &\geq -\frac{2\eta_t(s,a)H}{q_t(s)}Y_t -\eta_t(s, a)\pi_t(a\mid s) B_t(s,a) -\eta_t(s, a)\pi_t(a\mid s) C_t(s,a)Y_t\\
    &\geq -\frac{1}{3} - \frac{1}{6H} \geq -\frac{1}{2}.
\end{align}
where we used $0\le L_{t,h(s)},M_{t,h(s)}\le H$ and $C_t\ge 0$.
The last inequality uses $\eta_t(s,a)/q_t(s)\leq 1/(50\sqrt{H^3S^3A})$ in real episodes, \cref{lem:qC_size_bound,lem:bound_eta_pi_B}. Indeed,
\begin{align}
    \frac{2\eta_t(s,a)H}{q_t(s)}
    +\eta_t(s,a)\pi_t(a\mid s)C_t(s,a)
    &=\frac{\eta_t(s,a)}{q_t(s)}\prn*{2H+q_t(s)\pi_t(a\mid s)C_t(s,a)}\\
    &\leq \frac{2H+6HS}{50\sqrt{H^3S^3A}}
    \leq \frac{1}{3},
\end{align}
and
\begin{align}
    \eta_t(s,a)\pi_t(a\mid s)B_t(s,a)
    &\leq \frac{1}{6H}.
\end{align}

Hence, by \cref{lem:OFTRL_log_barrier_policy}, we obtain
 \begin{align}
    & \E[\regterm(s)] \\
    &\leq \frac{A\ln(AT^2)}{\eta_1} + \E\brk*{\sum_{t=1}^T \sum_a \prn*{\frac{1}{\eta_{t+1}(s, a)} - \frac{1}{\eta_t(s, a)}}\ln(AT^2)} \\
    &\qquad + \E\brk*{\sum_{t=1}^T \sum_a  \eta_t(s, a)\pi_t(a \mid s)^{2}\prn*{\prn*{\hatQ_t(s, a) - Q^{\underP^m_t, \pi_t}(s, a; m_t)} - B_t(s, a) + x_t}^2} \\
    &\qquad + \E \brk*{\frac{1}{T^2}\sum_{t=1}^T \inpr*{ -\pi(\cdot \mid s) + \frac{1}{A}\one, \hatQ_t(s, \cdot) - B_t(s, \cdot) }}\\
    &\qquad + 2 \E \brk*{\nrm*{Q ^{\underP ^{m} _{ T + 1} , \pi _{ T + 1}} ( s , \cdot ; m _{T+1} )} _{\infty}} \\
    &\leq \calO\prn*{HS^2A\ln(T)} + 3\E\brk*{\sum_{t=1}^T\sum_a \prn*{\frac{1}{\eta_{t+1}(s, a)} - \frac{1}{\eta_t(s, a)}}\ln(T)} \tag{by $\nrm*{Q^{\underP^m_{T+1}, \pi_{T+1}}(s, \cdot; m_{T+1})}_{\infty}\leq H$}\\
    &\qquad + 3\E\brk[\Bigg]{\underbrace{\sum_{t=1}^T \sum_a  \eta_t(s, a)\pi_t(a\mid s)^{2}\prn*{\frac{\I_t(s,a)(L_{t,h(s)}-M_{t,h(s)})}{q_t(s)\pi_t(a\mid s)}Y_t-\frac{\I_t(s)(L_{t,h(s)}-M_{t,h(s)})}{q_t(s)}Y_t}^2}_{\text{stability-term}}} \\
    &\qquad + 3\E\brk*{\sum_{t=1}^T \sum_a  \eta_t(s, a)\pi_t(a\mid s)^{2}\prn*{C_t(s,a)Y_t}^2}+ 3\E\brk*{\sum_{t=1}^T \sum_a  \eta_t(s, a)\pi_t(a\mid s)^{2} B_t(s, a)^2}. \label{eq:regterm_policy_mid}
\end{align} 
Here, the second inequality follows from
\begin{align}
 &\E\brk*{\frac{1}{T^2}\sum_{t=1}^T\inpr*{-\pi(\cdot\mid s)+\frac1A\one,\hatQ_t(s,\cdot)-B_t(s,\cdot)}}\\
 &=\frac{1}{T^2}\inpr*{-\pi(\cdot\mid s)+\frac1A\one,\E\brk*{\sum_{t=1}^T\E_t\brk*{\hatQ_t(s,\cdot)-B_t(s,\cdot)}}}\\
 &\leq\frac{1}{T^2}\nrm*{-\pi(\cdot\mid s)+\frac1A\one}_1
 \nrm*{\E\brk*{\sum_{t=1}^T\E_t\brk*{\hatQ_t(s,\cdot)-B_t(s,\cdot)}}}_\infty\\
 &\leq\frac{2T(6HST+8\sqrt{HS^3A}T)}{T^2}
 =12HS+16\sqrt{HS^3A}
 \leq\calO\prn*{HS^2A\ln(T)},
\end{align}
where we used $\nrm*{-\pi(\cdot\mid s)+\frac1A\one}_1\leq2$,
$\abs*{\E_t[\hatQ_t(s,a)]}\leq6HST$ and
$B_t(s,a)\leq\frac{8\sqrt{HS^3A}}{\gamma_t}$ from
\mbox{\cref{lem:qC_size_bound,lem:bound_eta_pi_B}}, which together imply
$\nrm*{\E\brk*{\sum_{t=1}^T\E_t\brk*{\hatQ_t(s,\cdot)-B_t(s,\cdot)}}}_\infty
\leq T(6HST+8\sqrt{HS^3A}T)$.
The initial regularization term is also absorbed into $\calO(HS^2A\ln T)$ since $A/\eta_1=12H^3A\leq12HS^2A$.

    We further evaluate the \text{stability-term} in the last inequality as
    \begin{align}
        &\eta_t(s, a)\pi_t(a\mid s)^{2}\prn*{\frac{\I_t(s,a)(L_{t,h(s)}-M_{t,h(s)})}{q_t(s)\pi_t(a\mid s)}Y_t-\frac{\I_t(s)(L_{t,h(s)}-M_{t,h(s)})}{q_t(s)}Y_t}^2 \\
        &= \eta_t(s,a)\prn*{\frac{\I_t(s,a) (L_{t,h(s)} - M_{t,h(s)})}{q_t(s)} - \frac{\pi_t(a\mid s) \I_t(s) (L_{t,h(s)} - M_{t,h(s)})}{q_t(s)}}^2Y_t  \\
        &= \frac{\eta_t(s,a)}{q_t(s)^2}(\I_t(s,a) - \pi_t(a\mid s) \I_t(s))^2 (L_{t,h(s)} - M_{t,h(s)})^2Y_t\\
        &= \frac{\eta_t(s,a)\zeta_t(s,a)}{q_t(s)^2} Y_t, \label{eq:stab1_bound}
    \end{align}
    where the last equality follows from the definition of $\zeta_t(s,a) = (\I_t(s,a) - \pi_t(a\mid s) \I_t(s))^2 (L_{t,h(s)} - M_{t,h(s)})^2$.
Thus, we have
 \begin{align}
    \text{stability-term}
    &= \sum_{t=1}^T \sum_a \frac{\eta_t(s,a)\prn*{\zeta_t(s,a)}}{q_t(s)^2}Y_t \\
    &\leq \sum_{t=1}^T \sum_a\prn*{\frac{1}{\eta_{t+1}(s,a)} - \frac{1}{\eta_t(s,a)}}\ln(T), 
\end{align} 
    where the last inequality follows from \cref{eq:eta_update_policy}.

    Then, together with \cref{eq:regterm_policy_mid} and \cref{lem:bound_eta_pi_B}, we obtain
     \begin{align}
    \E[\regterm(s)] &\leq \calO\prn*{HS^2A\ln(T)} + 6\mathbb{E}\brk*{\sum_{t=1}^T\sum_a  \prn*{\frac{1}{\eta_{t+1}(s, a)} - \frac{1}{\eta_t(s, a)}}\ln(T)} \\
    &\qquad + 3\E\brk*{\sum_{t=1}^T \sum_a  \eta_t(s, a)\pi_t(a\mid s)^{2}C_t(s,a)^2Y_t} \\
    &\qquad + 3\E\brk*{\sum_{t=1}^T \sum_a  \eta_t(s, a)\pi_t(a\mid s)^{2} B_t(s, a)^2} \\
    & \leq \calO \prn*{HS^2A\ln(T)} + \mathbb{E} \brk*{\sum_{t=1}^Tb_t(s)} + \mathbb{E} \brk*{\frac{1}{H}\sum_{t=1}^T\sum_a \pi_t(a \mid s) B_t(s,a)} . 
\end{align} 
\end{proof}

\begin{lem}[bandit setting]
\label{lem:biasterm_bandit}
For any comparator policy $\pi$, the following two bounds hold:
    \begin{align}
        &\E \brk*{\sum_{s}q^{\pi}(s)\biasterm(s) }\\ 
        &\lesssim \sqrt{H^2S^2A  \iota^2 \E\brk*{\sum_{t\in\calT_r} \sum_{s,a}q^{\pi_t}(s,a)\max_{\tilP\in\calP_t} Q^{\tilP, \pi_t}(s,a;\prn{\ell_t - m_t}^2)}} + H^{3/2}S^3A\iota^2
    \end{align}
    and 
     \begin{align}
    \E \brk*{\sum_{s}q^{\pi}(s)\biasterm(s) } 
    &\lesssim \sqrt{H^3S^2A  \iota^2 \E\brk*{\sum_{t\in\calT_r} \sum_{s,a}\brk*{q^{\pi_t}(s,a) - q^{\pi}(s,a)}_+}} + H ^{3 / 2} S ^{3} A \iota ^{2} . 
\end{align} 
\end{lem}
\begin{proof}

By the definition of $c_t(s,a)$ \cref{eq:def_new_c},
\begin{align}
    \E_t\brk*{c_t(s,a)}
    &=\rho_t(s)\frac{q^{\pi_t}(s)}{q_t(s)}Q^{\pi_t}(s,a;\abs*{\ell_t - m_t})
    + \rho_t(s)^2H. \label{eq:ct_cond_mean}
\end{align}
  When $P\in\calP_t$ holds , we obtain
\begin{align}
    \E_t\brk*{c_t(s,a)}
    &\geq \rho_t(s) \frac{q^{\pi_t}(s)}{q_t(s)}\abs*{Q^{\pi_t}(s,a;\ell_t-m_t)}
    +\rho_t(s)\prn*{1-\frac{q^{\pi_t}(s)}{q_t(s)}}\abs*{Q^{\pi_t}(s,a;\ell_t-m_t)} \\
    &=\rho_t(s)\abs*{Q^{\pi_t}(s,a;\ell_t-m_t)} \label{eq:ct_dominates_bias0}\\
    &\geq\frac{q_t(s)-q^{\pi_t}(s)}{q_t(s)}\abs*{Q^{\pi_t}(s,a;\ell_t-m_t)} . \label{eq:ct_dominates_bias}
\end{align}

Thus, by the definition of $\hat{Q}_t$, we have
\begin{align}
    &\E_t\brk*{Q^{\pi_t}(s,a;\ell_t)-\hat{Q}_t(s,a)+Q^{\underP^m_t,\pi_t}(s,a;m_t)-Q^{\pi_t}(s,a;m_t)}\\
    &=Q^{\pi_t}(s,a;\ell_t)-\prn*{Q^{\underP^m_t,\pi_t}(s,a;m_t)+\frac{q^{\pi_t}(s)}{q_t(s)}Q^{\pi_t}(s,a;\ell_t-m_t)Y_t-\E_t\brk*{C_t(s,a)}Y_t}\\
    &\qquad +Q^{\underP^m_t,\pi_t}(s,a;m_t)-Q^{\pi_t}(s,a;m_t)\\
    &=\frac{q_t(s)-q^{\pi_t}(s)Y_t}{q_t(s)}Q^{\pi_t}(s,a;\ell_t-m_t)+\E_t\brk*{C_t(s,a)}Y_t\\
    &=
    \begin{cases}
        -Q^{\pi_t}(s,a;m_t) & \text{if } Y_t=0,\\
        \dfrac{q_t(s)-q^{\pi_t}(s)}{q_t(s)}Q^{\pi_t}(s,a;\ell_t-m_t)+\E_t\brk*{C_t(s,a)} & \text{if } Y_t=1.
    \end{cases}
    \label{eq:cond_bias_new_c}
\end{align}

When $Y_t=1$  and $P\in\calP_t$ , by \cref{eq:ct_dominates_bias} and $C_t(s,a)\geq c_t(s,a)$ pathwise,
\begin{align}
    &\frac{q_t(s)-q^{\pi_t}(s)}{q_t(s)}Q^{\pi_t}(s,a;\ell_t-m_t)+\E_t\brk*{C_t(s,a)}\\
    &\geq -\frac{q_t(s)-q^{\pi_t}(s)}{q_t(s)}\abs*{Q^{\pi_t}(s,a;\ell_t-m_t)}+\E_t\brk*{c_t(s,a)}\\
    &\geq 0.
    \label{eq:real_bias_nonnegative_new_c}
\end{align}
When $Y_t=0$, using $Q^{\pi_t}(s,a;m_t)\leq H$ and the fact that the virtual episode has zero loss, we obtain
\begin{align}
    -H
    &\leq \E_t\brk*{Q^{\pi_t}(s,a;\ell_t)-\hat{Q}_t(s,a)+Q^{\underP^m_t,\pi_t}(s,a;m_t)-Q^{\pi_t}(s,a;m_t)}\\
    &\leq 0.
    \label{eq:virtual_bias_upper_new_c}
\end{align}

Therefore, 
 \begin{align}
    &\E\brk*{\sum_s q^{\pi}(s)\biasterm(s)} \\
    &=\E\brk*{\sum_{t=1}^T\sum_{s,a}q^{\pi}(s)(\pi_t(a\mid s)-\pi(a\mid s))\E_t\brk*{Q^{\pi_t}(s,a;\ell_t)-\hat{Q}_t(s,a)+Q^{\underP^m_t,\pi_t}(s,a;m_t)-Q^{\pi_t}(s,a;m_t)}} \\
    &\leq \E\brk*{\sum_{t\in\calT_r}\sum_{s,a}q^{\pi}(s)(\pi_t(a\mid s)-\pi(a\mid s))\prn*{\frac{q_t(s)-q^{\pi_t}(s)}{q_t(s)}Q^{\pi_t}(s,a;\ell_t-m_t)+\E_t\brk*{C_t(s,a)}}} \\
    &\qquad+\E\brk*{\sum_{t\in\calT_v}\sum_{s,a}q^{\pi}(s)\pi(a\mid s)H} \\
    & \leq \E \brk*{\sum_{t\in\calT_r}\sum_{s,a}\prn*{q^{\pi_t}(s,a)-q^{\pi}(s,a)}z_t(s,a)} + H  ^2 \abs{\calT_v} . \label{eq:bias_to_zt_new_c}
\end{align} 
The additive term   $H^2\abs{\calT_v} \lesssim H^3SA\ln^2T$  by \cref{lem:bound_virtual_episode} is lower order and will be absorbed into the final lower-order term. The last   inequality  follows by applying \cref{lem:pdl_with_L}, for each real episode $t\in\calT_r$, with
\begin{align}
    L(s,a)=\frac{q_t(s)-q^{\pi_t}(s)}{q_t(s)}Q^{\pi_t}(s,a;\ell_t-m_t)+\E_t\brk*{C_t(s,a)}.
\end{align}
Thus,
\begin{align}
    z_t(s,a)&=\frac{q_t(s)-q^{\pi_t}(s)}{q_t(s)}Q^{\pi_t}(s,a;\ell_t-m_t)+\E_t\brk*{C_t(s,a)}\\
    &\qquad-\E_{s'\sim P(\cdot\mid s,a),a'\sim\pi_t(\cdot\mid s')}\brk*{\frac{q_t(s')-q^{\pi_t}(s')}{q_t(s')}Q^{\pi_t}(s',a';\ell_t-m_t)+\E_t\brk*{C_t(s',a')}}.
    \label{eq:def_zt_new_c}
\end{align}

We first reduce the bound to the high-probability event $\calE$. By $\E_t\brk*{C_t(s,a)} \leq \calO(HST)$ from \cref{lem:qC_size_bound},
$\abs*{z_t(s,a)}\leq \calO(HST)$ holds.
Therefore, we have
\begin{align}
    &\E\brk*{\sum_{t\in\calT_r}\sum_{s,a}\prn*{q^{\pi_t}(s,a)-q^{\pi}(s,a)}z_t(s,a)}\\
    &=\Pr(\calE)\E\brk*{\sum_{t\in\calT_r}\sum_{s,a}\prn*{q^{\pi_t}(s,a)-q^{\pi}(s,a)}z_t(s,a)\relmiddle|\calE}\\
    &\qquad+\Pr(\bar{\calE})\E\brk*{\sum_{t\in\calT_r}\sum_{s,a}\prn*{q^{\pi_t}(s,a)-q^{\pi}(s,a)}z_t(s,a)\relmiddle|\bar{\calE}}\\
    &\leq \Pr(\calE)\E\brk*{\sum_{t\in\calT_r}\sum_{s,a}\prn*{q^{\pi_t}(s,a)-q^{\pi}(s,a)}z_t(s,a)\relmiddle|\calE}+\calO(\delta) \times \calO(HT) \times \calO(HST)\\
    &\leq \Pr(\calE)\E\brk*{\sum_{t\in\calT_r}\sum_{s,a}\prn*{q^{\pi_t}(s,a)-q^{\pi}(s,a)}z_t(s,a)\relmiddle|\calE}+\calO\prn*{H^2S}.
    \label{eq:high_prob_bias}
\end{align}
In the following, we work on the event $\calE$ and bound the first term on the right-hand side of \cref{eq:high_prob_bias}.

On $\calE$, since $P\in\calP_t$, we can upper bound $\E_t\brk*{c_t(s,a)}$ as
\begin{align}
    \E_t\brk*{c_t(s,a)}
    &\leq \rho_t(s)Q^{\pi_t}(s,a;\abs*{\ell_t-m_t})+\rho_t(s)^2H\\
    &\leq \rho_t(s)\max_{\tilP\in\calP_t}Q^{\tilP,\pi_t}(s,a;\abs{\ell_t-m_t})+\rho_t(s)^2H\\
    &\leq \frac{\overq_t^{\pi_t}(s)-\underq_t^{\pi_t}(s)}{q_t(s)}\max_{\tilP\in\calP_t}Q^{\tilP,\pi_t}(s,a;\abs{\ell_t-m_t})+\frac{\gamma_t H}{q_t(s)}+\rho_t(s)^2H.
    \label{eq:ct_upper_for_term1}
\end{align}

We next derive lower and upper bounds on $z_t(s,a)$. Since $P\in\calP_t$ on $\calE$, the definition of $C_t(s,a)$~\cref{eq:def_C} implies
\begin{align}
    \E_t\brk*{C_t(s,a)}
    &\geq\E_t\brk*{c_t(s,a)}+\E_{s'\sim P(\cdot\mid s,a), a'\sim\pi_t(\cdot\mid s')}\brk*{\E_t\brk*{c_t(s',a') + C_t(s',a')}}. \label{eq:Et_C_lower}
\end{align}

Combining \cref{eq:def_zt_new_c,eq:Et_C_lower,eq:ct_dominates_bias}, we have for any real episode $t \in \calT_r$
\begin{align}
    z_t(s,a)&\geq \frac{q_t(s)-q^{\pi_t}(s)}{q_t(s)}Q^{\pi_t}(s,a;\ell_t-m_t)+\E_t\brk*{c_t(s,a)}\\
    &\qquad +\E_{s'\sim P(\cdot\mid s,a),a'\sim\pi_t(\cdot\mid s')}\brk*{\E_t\brk*{c_t(s',a')} - \frac{q_t(s')-q^{\pi_t}(s')}{q_t(s')}Q^{\pi_t}(s',a';\ell_t-m_t)}\\
    &\geq 0.
    \label{eq:zt_geq0_new_c}
\end{align}

On the other hand, let $\overP^C_t$ be the measurable transition kernel that attains the maximum in \cref{eq:def_C} for the realized table $c_t$. Then, for any real episode $t \in \calT_r$
 \begin{align}
    z_t(s,a)&=\frac{q_t(s)-q^{\pi_t}(s)}{q_t(s)}Q^{\pi_t}(s,a;\ell_t-m_t)+\E_t\brk*{c_t(s,a)} \\
    &\qquad+\E_t\brk*{\E_{s'\sim \overP^C_t(\cdot\mid s,a),a'\sim\pi_t(\cdot\mid s')}\brk*{c_t(s',a')+C_t(s',a')}} \\
    &\qquad-\E_{s'\sim P(\cdot\mid s,a),a'\sim\pi_t(\cdot\mid s')}\brk*{\frac{q_t(s')-q^{\pi_t}(s')}{q_t(s')}Q^{\pi_t}(s',a';\ell_t-m_t)+\E_t\brk*{C_t(s',a')}} \\
    &\leq 2\E_t\brk*{c_t(s,a)} +\E_t\brk*{\E_{s'\sim \overP^C_t(\cdot\mid s,a),a'\sim\pi_t(\cdot\mid s')}\brk*{c_t(s',a')+C_t(s',a')}} \\
    &\qquad - \E _{s'\sim P(\cdot\mid s,a), a'\sim\pi_t(\cdot\mid s')} \brk*{- \E _{t} \brk*{c_t(s',a')} + \E_t\brk*{C_t(s',a')}} \tag{by \cref{eq:ct_dominates_bias0}}\\
    &\leq 2\E_t\brk*{c_t(s,a)}+2\E_{s'\sim P(\cdot\mid s,a),a'\sim\pi_t(\cdot\mid s')}\brk*{\E_t\brk*{c_t(s',a')}} \\
    &\qquad+\E_t\brk*{\sum_{s',a'}\abs*{\overP^C_t(s'\mid s,a)-P(s'\mid s,a)}\pi_t(a'\mid s')\prn*{c_t(s',a')+C_t(s',a')}}. \label{eq:zt_upper_before_unroll_new_c}
\end{align} 

By the definition of $C_t(s,a)$ \cref{eq:def_C}, we have
\begin{align}
    \sum_a\pi_t(a\mid s)\prn*{c_t(s,a)+C_t(s,a)}&=2\sum_{s',a'}q^{\overP^C_t,\pi_t}(s',a'\mid s)c_t(s',a').
    \label{eq:C_unroll_new_c}
\end{align}
Combining \cref{eq:zt_upper_before_unroll_new_c} and \cref{eq:C_unroll_new_c}, we obtain for any real episode $t \in \calT_r$
\begin{align}
    z_t(s,a)&\leq 2\E_t\brk*{c_t(s,a)}+2\E_{s'\sim P(\cdot\mid s,a),a'\sim\pi_t(\cdot\mid s')}\brk*{\E_t\brk*{c_t(s',a')}}\\
    &\qquad+2\E_t\brk*{\sum_{s'}\abs*{\overP^C_t(s'\mid s,a)-P(s'\mid s,a)}\sum_{s'',a''}q^{\overP^C_t,\pi_t}(s'',a''\mid s')c_t(s'',a'')}.
    \label{eq:zt_upper_new_c}
\end{align}

Therefore, by \cref{eq:zt_geq0_new_c} and \cref{eq:zt_upper_new_c},
\begin{align}
    &\sum_{t\in\calT_r}\sum_{s,a}\prn*{q^{\pi_t}(s,a)-q^{\pi}(s,a)}z_t(s,a)\\
    &\leq \sum_{t\in\calT_r}\sum_{s,a}\brk*{q^{\pi_t}(s,a)-q^{\pi}(s,a)}_+z_t(s,a)
    \tag{by \cref{eq:zt_geq0_new_c}}\\
    &\leq 2\underbrace{\sum_{t\in\calT_r}\sum_{s,a}\brk*{q^{\pi_t}(s,a)-q^{\pi}(s,a)}_+\E_t\brk*{c_t(s,a)}}_{\term_1}\\
    &\qquad+2\underbrace{\sum_{t\in\calT_r}\sum_{s,a}\brk*{q^{\pi_t}(s,a)-q^{\pi}(s,a)}_+\E_{s'\sim P(\cdot\mid s,a),a'\sim\pi_t(\cdot\mid s')}\brk*{\E_t\brk*{c_t(s',a')}}}_{\term_2}\\
    &\qquad+2\underbrace{\sum_{t\in\calT_r}\sum_{s,a}\brk*{q^{\pi_t}(s,a)-q^{\pi}(s,a)}_+\E_t\brk*{\sum_{s'}\abs*{\overP^C_t(s'\mid s,a)-P(s'\mid s,a)}\sum_{s'',a''}q^{\overP^C_t,\pi_t}(s'',a''\mid s')c_t(s'',a'')}}_{\term_3}.
\end{align}

Let 
   \begin{align}
    G^{\pi_t} \coloneq \min\set*{\sum_{s,a}q^{\pi_t}(s,a)\max_{\tilP\in\calP_t} Q^{\tilP, \pi_t}(s,a;(\ell_t-m_t)^2),\; H\sum_{s,a}\brk*{q^{\pi_t}(s,a) - q^{\pi}(s,a)}_+}. \label{def:proof_G}
\end{align} 
We first bound $\term_1$. By \cref{eq:ct_upper_for_term1},  
\begin{align}
    \term_1&=\sum_{t\in\calT_r}\sum_{s,a}\brk*{q^{\pi_t}(s,a)-q^{\pi}(s,a)}_+\E_t\brk*{c_t(s,a)}\\
    &\leq \underbrace{\sum_{t\in\calT_r}\sum_{s,a}\prn*{\overq_t^{\pi_t}(s)-\underq_t^{\pi_t}(s)}\frac{\brk*{q^{\pi_t}(s,a)-q^{\pi}(s,a)}_+}{q_t(s)}\max_{\tilP\in\calP_t}Q^{\tilP,\pi_t}(s,a;\abs{\ell_t-m_t})}_{\term_{1a}}\\
    &\qquad+\underbrace{H\sum_{t\in\calT_r}\sum_{s,a}\frac{\gamma_t}{q_t(s)}\brk*{q^{\pi_t}(s,a)-q^{\pi}(s,a)}_+}_{\term_{1b}} +\underbrace{H\sum_{t\in\calT_r}\sum_{s,a}\brk*{q^{\pi_t}(s,a)-q^{\pi}(s,a)}_+\rho_t(s)^2}_{\term_{1c}}.
    \label{eq:term1_split_new_c}
\end{align}

We first control $\term_{1a}$. Define
\begin{align}
    g_t^{(1)}(s)&\coloneq\sum_a\frac{\brk*{q^{\pi_t}(s,a)-q^{\pi}(s,a)}_+}{q_t(s)}\max_{\tilP\in\calP_t}Q^{\tilP,\pi_t}(s,a;\abs{\ell_t-m_t})\\
    &\leq H.
    \label{eq:def_g1_term1}
\end{align}
Thus, by \cref{lem:complicated_lemma},
\begin{align}
    \term_{1a}&=\sum_{t\in\calT_r}\sum_s\prn*{\overq_t^{\pi_t}(s)-\underq_t^{\pi_t}(s)}g_t^{(1)}(s)\\
    &\lesssim \sqrt{HS^2A\iota^2\sum_{t\in\calT_r}\sum_s q^{\pi_t}(s)\prn*{g_t^{(1)}(s)}^2}+HS^3A\iota^2.
    \label{eq:term1a_width_bound}
\end{align}
Here, by the Cauchy--Schwarz inequality,
\begin{align}
    &\sum_s q^{\pi_t}(s)\prn*{g_t^{(1)}(s)}^2\\
    &\leq \sum_s q^{\pi_t}(s)\prn*{\sum_a\frac{\brk*{q^{\pi_t}(s,a)-q^{\pi}(s,a)}_+}{q_t(s)}}\prn*{\sum_a\frac{\brk*{q^{\pi_t}(s,a)-q^{\pi}(s,a)}_+}{q_t(s)}\prn*{\max_{\tilP\in\calP_t}Q^{\tilP,\pi_t}(s,a;\abs{\ell_t-m_t})}^2}\\
    &\leq \sum_{s,a}\brk*{q^{\pi_t}(s,a)-q^{\pi}(s,a)}_+\prn*{\max_{\tilP\in\calP_t}Q^{\tilP,\pi_t}(s,a;\abs{\ell_t-m_t})}^2\\
    &\leq H\sum_{s,a}\brk*{q^{\pi_t}(s,a)-q^{\pi}(s,a)}_+\max_{\tilP\in\calP_t}Q^{\tilP,\pi_t}(s,a;\prn{\ell_t-m_t}^2) \tag{by \cref{eq:max_square_to_H}}\\
    &\leq HG^{\pi_t},
    \label{eq:g1_square_to_G}
\end{align}
Therefore,
\begin{align}
    \term_{1a}&\lesssim \sqrt{H^2S^2A\iota^2\sum_{t\in\calT_r}G^{\pi_t}}+HS^3A\iota^2.
    \label{eq:term1a_final_new_c}
\end{align}
Next, $\term_{1b}$ can be bounded by
 \begin{align}
    \term_{1b}&\leq H\sum_{t\in\calT_r}\sum_s\gamma_t\frac{q^{\pi_t}(s)}{q_t(s)}\\
    &\leq HS\sum_{t\in\calT_r}\gamma_t\\
    &\lesssim HS\ln(T). \label{eq:term1b_final_new_c}
\end{align} 
Finally, $\term_{1c}$ can be bounded by using \cref{lem:bound_rho2},
\begin{align}
    \term_{1c}&\leq H\sum_{t\in\calT_r}\sum_s q^{\pi_t}(s)\rho_t(s)^2\\
    &\lesssim HS^3A\iota^2.
    \label{eq:term1c_final_new_c}
\end{align}
Combining \cref{eq:term1a_final_new_c,eq:term1b_final_new_c,eq:term1c_final_new_c}, we obtain
\begin{align}
    \term_1&\lesssim \sqrt{H^2S^2A\iota^2\sum_{t\in\calT_r}G^{\pi_t}}+HS^3A\iota^2.
    \label{eq:term1_final_new_c}
\end{align}

We next bound $\term_2$. By \cref{eq:ct_upper_for_term1}, we have
\begin{align}
    \term_2&=\sum_{t\in\calT_r}\sum_{s,a}\brk*{q^{\pi_t}(s,a)-q^{\pi}(s,a)}_+\E_{s'\sim P(\cdot\mid s,a),a'\sim\pi_t(\cdot\mid s')}\brk*{\E_t\brk*{c_t(s',a')}}\\
    &\leq \underbrace{\sum_{t\in\calT_r}\sum_{s,a}\brk*{q^{\pi_t}(s,a)-q^{\pi}(s,a)}_+\E_{s'\sim P(\cdot\mid s,a),a'\sim\pi_t(\cdot\mid s')}\brk*{\frac{\overq_t^{\pi_t}(s')-\underq_t^{\pi_t}(s')}{q_t(s')}\max_{\tilP\in\calP_t}Q^{\tilP,\pi_t}(s',a';\abs{\ell_t-m_t})}}_{\term_{2a}}\\
    &\qquad+\underbrace{H\sum_{t\in\calT_r}\sum_{s,a}\brk*{q^{\pi_t}(s,a)-q^{\pi}(s,a)}_+\E_{s'\sim P(\cdot\mid s,a)}\brk*{\frac{\gamma_t}{q_t(s')}}}_{\term_{2b}}\\
    &\qquad+\underbrace{H\sum_{t\in\calT_r}\sum_{s,a}\brk*{q^{\pi_t}(s,a)-q^{\pi}(s,a)}_+\E_{s'\sim P(\cdot\mid s,a)}\brk*{\rho_t(s')^2}}_{\term_{2c}}.
    \label{eq:term2_split_new_c}
\end{align}

We first bound $\term_{2a}$. Define
\begin{align}
    g_t^{(2)}(s')&\coloneq\frac{\sum_{s,a}\brk*{q^{\pi_t}(s,a)-q^{\pi}(s,a)}_+P(s'\mid s,a)}{q_t(s')}\sum_{a'}\pi_t(a'\mid s')\max_{\tilP\in\calP_t}Q^{\tilP,\pi_t}(s',a';\abs{\ell_t-m_t})\\
    &\leq \sum_{a'}\pi_t(a'\mid s')\max_{\tilP\in\calP_t}Q^{\tilP,\pi_t}(s',a';\abs{\ell_t-m_t})\\
    &\leq H.\label{eq:g2_range_term2}
    \label{eq:def_g2_term2}
\end{align}
Therefore, by \cref{lem:complicated_lemma},
\begin{align}
    \term_{2a}&=\sum_{t\in\calT_r}\sum_{s'}\prn*{\overq_t^{\pi_t}(s')-\underq_t^{\pi_t}(s')}g_t^{(2)}(s')\\
    &\lesssim \sqrt{HS^2A\iota^2\sum_{t\in\calT_r}\sum_{s'}q^{\pi_t}(s')\prn*{g_t^{(2)}(s')}^2}+HS^3A\iota^2.
    \label{eq:term2a_width_bound}
\end{align}
Moreover, by using the Cauchy--Schwarz inequality, we obtain
\begin{align}
    &\sum_{s'}q^{\pi_t}(s')\prn*{g_t^{(2)}(s')}^2\\
    &\leq \sum_{s'}q^{\pi_t}(s')\prn*{\sum_{a'}\pi_t(a'\mid s')\max_{\tilP\in\calP_t}Q^{\tilP,\pi_t}(s',a';\abs{\ell_t-m_t})}^2\\
    &\leq \sum_{s',a'}q^{\pi_t}(s',a')\prn*{\max_{\tilP\in\calP_t}Q^{\tilP,\pi_t}(s',a';\abs{\ell_t-m_t})}^2\\
    &\leq H\sum_{s',a'}q^{\pi_t}(s',a')\max_{\tilP\in\calP_t}Q^{\tilP,\pi_t}(s',a';\prn{\ell_t-m_t}^2),\label{eq:g2_square_bound}
\end{align}
where the last inequality follows from \cref{eq:max_square_to_H}.

Simultaneously, it holds that
\begin{align}
    &\sum_{s'}q^{\pi_t}(s')\prn*{g_t^{(2)}(s')}^2\\
    &\leq H\sum_{s'}q^{\pi_t}(s')\sum_{s,a,a'}P(s'\mid s,a)\pi_t(a'\mid s')\frac{\brk*{q^{\pi_t}(s,a)-q^{\pi}(s,a)}_+}{q_t(s')}\max_{\tilP\in\calP_t}Q^{\tilP,\pi_t}(s',a';\abs{\ell_t-m_t})\\
    &\leq H^2\sum_{s,a,s',a'}\brk*{q^{\pi_t}(s,a)-q^{\pi}(s,a)}_+P(s'\mid s,a)\pi_t(a'\mid s')\\
    &\leq H^2\sum_{s,a}\brk*{q^{\pi_t}(s,a)-q^{\pi}(s,a)}_+.
    \label{eq:g2_square_bound2}
\end{align}
From \cref{eq:g2_square_bound,eq:g2_square_bound2}, we have
\begin{align}
    \sum_{s'}q^{\pi_t}(s')\prn*{g_t^{(2)}(s')}^2\leq HG^{\pi_t}.
\end{align}
Hence,
\begin{align}
    \term_{2a}&\lesssim \sqrt{H^2S^2A\iota^2\sum_{t\in\calT_r}G^{\pi_t}}+HS^3A\iota^2.
    \label{eq:term2a_final_new_c}
\end{align}

Next, $\term_{2b}$ can be bounded by 
 \begin{align}
    \term_{2b}&\leq H\sum_{t\in\calT_r}\sum_{s'}\gamma_t\frac{\sum_{s,a}\brk*{q^{\pi_t}(s,a)-q^{\pi}(s,a)}_+P(s'\mid s,a)}{q_t(s')}\\
    &\leq HS\sum_{t\in\calT_r}\gamma_t\\
    &\lesssim HS\ln(T). \label{eq:term2b_final_new_c}
\end{align} 
Finally,   \cref{lem:bound_rho2} gives
\begin{align}
 \term_{2c}
 &\leq H\sum_{t\in\calT_r}\sum_{s,a,s'}[q^{\pi_t}(s,a)-q^\pi(s,a)]_+P(s'\mid s,a)\rho_t(s')^2\\
 &\leq H\sum_{t\in\calT_r}\sum_{s'}q^{\pi_t}(s')\rho_t(s')^2\\
    &\lesssim HS^3A\iota^2.
 \label{eq:term2c_final_new_c}
\end{align} 
Combining \cref{eq:term2a_final_new_c,eq:term2b_final_new_c,eq:term2c_final_new_c}, we obtain
\begin{align}
    \term_2&\lesssim \sqrt{H^2S^2A\iota^2\sum_{t\in\calT_r}G^{\pi_t}}+HS^3A\iota^2.
    \label{eq:term2_final_new_c}
\end{align}

We finally bound $\term_3$.
By \cref{lem:tildeP-P,eq:ct_cond_mean}, we have
{\small
\begin{align}
 &\term_3\\
 &=\sum_{t\in\calT_r}\sum_{s,a}[q^{\pi_t}(s,a)-q^\pi(s,a)]_+
 \E_t\brk*{\sum_{s'}\abs*{\overP_t^C(s'\mid s,a)-P(s'\mid s,a)}
 \sum_{s'',a''}q^{\overP_t^C,\pi_t}(s'',a''\mid s')c_t(s'',a'')}\\
 &\lesssim\sum_{t\in\calT_r}\sum_{s,a,s'}[q^{\pi_t}(s,a)-q^\pi(s,a)]_+
 \prn*{\sqrt{\frac{P(s'\mid s,a)\iota}{n_t(s,a)}}+\frac{\iota}{n_t(s,a)}}\E_t\brk*{\sum_{s'',a''}q^{\overP_t^C,\pi_t}(s'',a''\mid s')c_t(s'',a'')}\\
 &\leq\sum_{t\in\calT_r}\sum_{s,a,s'}[q^{\pi_t}(s,a)-q^\pi(s,a)]_+
 \prn*{\sqrt{\frac{P(s'\mid s,a)\iota}{n_t(s,a)}}+\frac{\iota}{n_t(s,a)}}\\
 &\qquad\times\sum_{s''}\max_{\tilP\in\calP_t}q^{\tilP,\pi_t}(s''\mid s')
 \sum_{a''}\pi_t(a''\mid s'')\E_t[c_t(s'',a'')]\\
 &\lesssim H\sum_{t\in\calT_r}\sum_{s,a,s'}[q^{\pi_t}(s,a)-q^\pi(s,a)]_+
 \prn*{\sqrt{\frac{P(s'\mid s,a)\iota}{n_t(s,a)}}+\frac{\iota}{n_t(s,a)}}\sum_{s''}\max_{\tilP\in\calP_t}q^{\tilP,\pi_t}(s''\mid s')\rho_t(s'').
 \label{eq:term3_envelope}
\end{align}
}
Applying the Cauchy--Schwarz inequality to \cref{eq:term3_envelope} gives
\begin{align}
 \term_3
 &\lesssim H\sqrt{\sum_{t\in\calT_r}\sum_{s,a,s'}[q^{\pi_t}(s,a)-q^\pi(s,a)]_+\frac{\iota}{n_t(s,a)}}\\
 &\qquad\times\sqrt{\sum_{t\in\calT_r}\sum_{s,a,s'}[q^{\pi_t}(s,a)-q^\pi(s,a)]_+P(s'\mid s,a)
 \prn*{\sum_{s''}\max_{\tilP\in\calP_t}q^{\tilP,\pi_t}(s''\mid s')\rho_t(s'')}^2}\\
 &\qquad+HS\sum_{t\in\calT_r}\sum_{s,a,s'}\frac{q^{\pi_t}(s,a)\iota}{n_t(s,a)}.
\end{align}
To bound the second factor, for each layer $h$, we have
\begin{align}
 &\sum_{(s,a,s')\in\calS_h\times\calA\times\calS_{h+1}}q^{\pi_t}(s,a)P(s'\mid s,a)
 \max_{\tilP\in\calP_t}q^{\tilP,\pi_t}(s''\mid s')\\
 &=\max_{\tilP\in\calP_t}\sum_{(s,a,s')\in\calS_h\times\calA\times\calS_{h+1}}
 q^{\pi_t}(s,a)P(s'\mid s,a)q^{\tilP,\pi_t}(s''\mid s')\leq\overq_t^{\pi_t}(s'').
 \label{eq:term3_splicing_bound}
\end{align}
Thus, we have
\begin{align}
 &\sum_{t\in\calT_r}\sum_{s,a,s'}[q^{\pi_t}(s,a)-q^\pi(s,a)]_+P(s'\mid s,a)
 \prn*{\sum_{s''}\max_{\tilP\in\calP_t}q^{\tilP,\pi_t}(s''\mid s')\rho_t(s'')}^2\\
 &\leq S\sum_{t\in\calT_r}
 \sum_{h=0}^{H-1}\sum_{(s,a,s')\in\calS_h\times\calA\times\calS_{h+1}}
 q^{\pi_t}(s,a)P(s'\mid s,a)\sum_{s''}\rho_t(s'')^2\max_{\tilP\in\calP_t}q^{\tilP,\pi_t}(s''\mid s') \tag{by the Cauchy--Schwarz inequality}\\
 &\leq SH\sum_{t\in\calT_r}\sum_{s''}\overq_t^{\pi_t}(s'')\rho_t(s'')^2 \tag{by \cref{eq:term3_splicing_bound}}\\
    &\lesssim HS^4A\iota^2. \tag{by \cref{lem:bound_rho2_upper}}
\end{align}
Substituting these bounds and applying \cref{lem:occup_bound}, we obtain
\begin{align}
 \term_3
 &\lesssim H\sqrt{S^2A\iota^2\cdot HS^4A\iota^2}+HS^3A\iota^2\\
    &\lesssim H^{3/2}S^3A\iota^2.
 \label{eq:term3_maxQ_bound}
\end{align}
Combining \cref{eq:high_prob_bias,eq:term1_final_new_c,eq:term2_final_new_c,eq:term3_maxQ_bound}
with the virtual-episode contribution and applying Jensen's inequality gives
\begin{align}
 \E\brk*{\sum_sq^\pi(s)\biasterm(s)}
 &\lesssim\sqrt{H^2S^2A\iota^2\E\brk*{\sum_{t\in\calT_r}G^{\pi_t}}}
 +H^{3/2}S^3A\iota^2.
\end{align}
Using either entry of the minimum in \cref{def:proof_G} proves the two asserted bounds.

\end{proof}

\begin{lem}[bandit setting]
\label{lem:errorterm_bandit}
    For any comparator policy $\pi$, it holds that
    \begin{align}
        \E\brk*{\sum_s q^{\pi}(s)\errorterm(s)}
        &\lesssim 
        \sqrt{S^2A\iota^2\E\brk*{\Qtrans^{\pi_{1:T}}(m)}} + H^2S^{2}A\iota^2. \label{eq:bound_error_1}
    \end{align}
    Also,
    \begin{align}
        &\E\brk*{\sum_s q^{\pi}(s)\errorterm(s)}\\
        &\lesssim 
        \sum_{h = 0}^{H-1}\sqrt{H^2\abs*{S_{h}}\abs*{S_{h + 1}}A\iota^2 \E\brk*{\sum_{t\in\calT_r}\sum_{(s,a) \in \calS_h \times \calA}\brk*{q^{\pi_t}(s,a) - q^{\pi}(s,a)}_{+}}} + H^2S^2A\iota^2. \label{eq:bound_error_2}
    \end{align}
\end{lem}
\begin{proof}
The proof is identical to that of \cref{lem:biasterm_full}, with the loss sequence $\ell_t$ replaced by the prediction sequence $m_t$. 
The only additional point is that we first separate the contribution of virtual episodes.

\begin{align}
&\sum_s q^{\pi}(s)\errorterm(s) \\
&= \sum_s q^{\pi}(s) \sum_{t=1}^T\sum_{a} \prn*{\pi_t(a\mid s) -  \pi(a\mid s)} \prn*{Q^{\pi_t}(s, a; m_t) - Q^{\underP^m_t, \pi_t}(s, a; m_t)}\\
&\leq \sum_s q^{\pi}(s) \sum_{t\in\calT_r}\sum_{a} \prn*{\pi_t(a\mid s) -  \pi(a\mid s)} \prn*{Q^{\pi_t}(s, a; m_t) - Q^{\underP^m_t, \pi_t}(s, a; m_t)} + 2H^2\abs{\calT_v}\\
&\leq \sum_s q^{\pi}(s) \sum_{t\in\calT_r}\sum_{a} \prn*{\pi_t(a\mid s) -  \pi(a\mid s)} \prn*{Q^{\pi_t}(s, a; m_t) - Q^{\underP^m_t, \pi_t}(s, a; m_t)}\\*
&\qquad + \calO(H^3SA\ln^2(T)),
\end{align}
where we used \cref{lem:bound_virtual_episode}. The resulting virtual-episode contribution is absorbed into the lower-order term. 
It remains to bound the first term in expectation, which follows from \cref{lem:biasterm_full} applied to the real episodes.

Applying the first bound of \cref{lem:biasterm_full} with $\ell_t$ replaced by $m_t$ gives, on $\calE$,
\begin{align}
    \sum_s q^\pi(s)\errorterm(s)
    &\lesssim
    \sqrt{S^2A\iota^2\Qtrans^{\pi_{1:T}}(m)}
    +H^2S^{2}A\iota^2.
\end{align}
Taking expectations, using Jensen's inequality, and combining with the trivial bound $\calO(H^2T)$ on $\bar{\calE}$ with $\delta=\calO(1/T^2)$, we obtain
\begin{align}
    \E\brk*{\sum_s q^{\pi}(s)\errorterm(s)}
    &\lesssim
    \sqrt{S^2A\iota^2\E\brk*{\Qtrans^{\pi_{1:T}}(m)}}
    +H^2S^{2}A\iota^2
    +\delta H^2T\\
    &\lesssim
    \sqrt{S^2A\iota^2\E\brk*{\Qtrans^{\pi_{1:T}}(m)}}
    +H^2S^{2}A\iota^2,
\end{align}
which proves \cref{eq:bound_error_1}.

Similarly, applying the second bound of \cref{lem:biasterm_full} with $\ell_t$ replaced by $m_t$ gives, on $\calE$,
\begin{align}
    \sum_s q^\pi(s)\errorterm(s)
    &\lesssim \sum_{h=0}^{H-1}\sqrt{H^2\abs*{S_h}\abs*{S_{h+1}}A\iota^2\sum_{t\in\calT_r}\sum_{(s,a)\in\calS_h\times\calA}\brk*{q^{\pi_t}(s,a)-q^\pi(s,a)}_+}\\
    &\qquad+H^2S^2A\iota^2.
\end{align}
Taking expectations, applying Jensen's inequality, and combining with the trivial bound $\calO(H^2T)$ on $\bar{\calE}$ with $\delta=\calO(1/T^2)$, we obtain
\begin{align}
    &\E\brk*{\sum_s q^\pi(s)\errorterm(s)}\\
    &\lesssim \sum_{h=0}^{H-1}\sqrt{H^2\abs*{S_h}\abs*{S_{h+1}}A\iota^2\E\brk*{\sum_{t\in\calT_r}\sum_{(s,a)\in\calS_h\times\calA}\brk*{q^{\pi_t}(s,a)-q^\pi(s,a)}_+}}
    +H^2S^2A\iota^2+\delta H^2T\\
    &\lesssim \sum_{h=0}^{H-1}\sqrt{H^2\abs*{S_h}\abs*{S_{h+1}}A\iota^2\E\brk*{\sum_{t\in\calT_r}\sum_{(s,a)\in\calS_h\times\calA}\brk*{q^{\pi_t}(s,a)-q^\pi(s,a)}_+}}
    +H^2S^2A\iota^2.
\end{align}
This proves \cref{eq:bound_error_2}.
\end{proof}

\begin{lem}[bandit setting]\label{lem:squared_C_bound}
It holds that
\begin{align}
&\E\brk*{\sum_{t\in\calT_r}\sum_{s,a}q^{\overP_t^B,\pi_t}(s)\eta_t(s,a)\pi_t(a\mid s)^2C_t(s,a)^2}
\lesssim H^{\frac32}S^{\frac{5}2}A^{\frac{1}2}\iota^2.
\end{align}
\end{lem}
\begin{proof}
By the decomposition of $q_t(s)\pi_t(a\mid s)C_t(s,a)$ in \cref{eq:qC_bound_middle}, we have
\begin{align}
 &q_t(s)\pi_t(a\mid s)C_t(s,a)\\
 &\leq2\sum_{s',a'}\frac{q_t(s)\pi_t(a\mid s)q^{\overP_t^C,\pi_t}(s',a'\mid s,a)}{q_t(s')\pi_t(a'\mid s')}
 \rho_t(s')\I_t(s',a')D_{t,h(s')}\\
 &\qquad+2Hq_t(s)\pi_t(a\mid s)\sum_{s',a'}
 q^{\overP_t^C,\pi_t}(s',a'\mid s,a)\rho_t(s')^2.
 \label{eq:squared_C_decomp}
\end{align}
Thus, by the Cauchy--Schwarz inequality, we obtain
\begin{align}
 &\E_t\brk*{\sum_{s,a}q_t(s)^2\pi_t(a\mid s)^2C_t(s,a)^2}\\
 &\lesssim H^3\E_t\brk*{\sum_{s',a'}\rho_t(s')^2\I_t(s',a')
 \sum_{s,a}\prn*{\frac{q_t(s)\pi_t(a\mid s)q^{\overP_t^C,\pi_t}(s',a'\mid s,a)}{q_t(s')\pi_t(a'\mid s')}}^2}\\
 &\qquad+H^3\E_t\brk*{\sum_{s',a'}\rho_t(s')^2
 \sum_{s,a}q_t(s)^2\pi_t(a\mid s)^2q^{\overP_t^C,\pi_t}(s',a'\mid s,a)}.
 \label{eq:C-square-decomp}
\end{align}
For the first term of \cref{eq:C-square-decomp}, we have
\begin{align}
 \sum_{s,a}\prn*{\frac{q_t(s)\pi_t(a\mid s)q^{\overP_t^C,\pi_t}(s',a'\mid s,a)}{q_t(s')\pi_t(a'\mid s')}}^2
 \leq\sum_s\prn*{\sum_a\frac{q_t(s)\pi_t(a\mid s)q^{\overP_t^C,\pi_t}(s',a'\mid s,a)}{q_t(s')\pi_t(a'\mid s')}}^2\leq S,
 \label{eq:C-action-splicing}
\end{align}
where the same argument as for \cref{eq:conditional_occupancy_domination} gives
\begin{align}
 &\sum_a q_t(s)\pi_t(a\mid s)q^{\overP_t^C,\pi_t}(s',a'\mid s,a)\\
 &=\prn*{\overq_t^{\pi_t}(s)+\gamma_t}\sum_a\pi_t(a\mid s)q^{\overP_t^C,\pi_t}(s'\mid s,a)\pi_t(a'\mid s')\\
 &\leq\prn*{\overq_t^{\pi_t}(s')+\gamma_t}\pi_t(a'\mid s')
 =q_t(s')\pi_t(a'\mid s').
\end{align}
For the second term of \cref{eq:C-square-decomp}, since $q_t(s)\pi_t(a\mid s)\leq2$, we have
\begin{align}
 \sum_{s,a}q_t(s)^2\pi_t(a\mid s)^2q^{\overP_t^C,\pi_t}(s',a'\mid s,a)
 \leq2Sq_t(s')\pi_t(a'\mid s').
\end{align}
Hence, using $\E_t[\I_t(s',a')]=q^{\pi_t}(s',a')$,
\begin{align}
 &\E_t\brk*{\sum_{s,a}q_t(s)^2\pi_t(a\mid s)^2C_t(s,a)^2}
 \lesssim H^3S\sum_{s'}\prn*{q^{\pi_t}(s')+q_t(s')}\rho_t(s')^2.
 \label{eq:squared_C_to_rho}
\end{align}
Combining \cref{lem:bound_rho2_upper} and \cref{lem:bound_rho2} with the trivial bound
$\sum_{s'}(q^{\pi_t}(s')+q_t(s'))\rho_t(s')^2\leq3S$, and \cref{eq:squared_C_to_rho}, we get
\begin{align}
 \E\brk*{\sum_{t\in\calT_r}\sum_{s,a}q_t(s)^2\pi_t(a\mid s)^2C_t(s,a)^2}\lesssim H^3S\prn*{S^3A\iota^2+\delta ST}
 &\lesssim H^3S^{4}A\iota^2.
 \label{eq:squared_C_final_bound}
\end{align}
Therefore,
\begin{align}
 &\E\brk*{\sum_{t\in\mathcal{T}_r}\sum_{s,a}q^{\overP_t^B,\pi_t}(s)\eta_t(s,a)\pi_t(a\mid s)^2C_t(s,a)^2}\\
 &=\E\brk*{\sum_{t\in\mathcal{T}_r}\sum_{s,a}q^{\overP_t^B,\pi_t}(s)
 \frac{\eta_t(s,a)q_t(s)^2\pi_t(a\mid s)^2C_t(s,a)^2}{q_t(s)^2}}\\
 &\leq\E\brk*{\sum_{t\in\mathcal{T}_r}\sum_{s,a}q_t(s)^2\pi_t(a\mid s)^2C_t(s,a)^2
 \max_{s,a}\frac{\eta_t(s,a)}{q_t(s)}}\\
 &\leq\E\brk*{\sum_{t\in\mathcal{T}_r}\sum_{s,a}q_t(s)^2\pi_t(a\mid s)^2C_t(s,a)^2}
 \cdot\frac{1}{50\sqrt{H^3S^3A}} \tag{by $t$ is a real episode}\\
 &\lesssim H^3S^{4}A\iota^2
 \cdot\frac{1}{\sqrt{H^3S^3A}} \tag{by \cref{eq:squared_C_final_bound}}\\
 &=H^{\frac32}S^{\frac{5}2}A^{\frac{1}2}\iota^2.
\end{align}
\end{proof}

\begin{lem}[bandit setting]\label{lem:log_barrier_bt_bandit}
It holds that
\begin{align}
&\E\brk*{\sum_{t=1}^TV^{\overP_t^B,\pi_t}(s_0;b_t)}
\lesssim\ln(T)\sum_{s,a}\sqrt{\E\brk*{\sum_{t\in\calT_r}\zeta_t(s,a)}}
+H^{\frac32}S^{\frac{5}2}A^{\frac32}\iota^2.
\end{align}
\end{lem}
\begin{proof}
By \cref{lem:learning_rate_policy} and $\eta_t(s,a)\leq\eta_1\leq\sqrt{\ln(T)}$,
\begin{align}
 \eta_t(s,a)\lesssim
 \frac{\sqrt{\ln(T)}}{\sqrt{1+\sum_{\tau\leq t:\tau\in\calT_r}\zeta_\tau(s,a)/q_\tau(s)^2}}.\label{eq:bandit_learning_rate_shifted}
\end{align}
 We have
    \allowdisplaybreaks
     \begin{align}
    &\E\brk*{\sum_{t=1}^T  V^{\overP^B_t, \pi_t}(s_0;b_t)} \\
    &= \E\brk*{\sum_{t=1}^T\sum_s q^{\overP^B_t,\pi_t}(s) b_t(s)} \\
    &\lesssim \E\brk*{\sum_{t\in\calT_r}\sum_{s,a} q^{\overP^B_t,\pi_t}(s) \prn*{\frac{1}{\eta_{t+1}(s,a)}-\frac{1}{\eta_t(s,a)}}\ln(T)} \\
    &\qquad + \E\brk*{\sum_{t\in\calT_v}\sum_{s,a} q^{\overP^B_t,\pi_t}(s) \prn*{\frac{1}{\eta_{t+1}(s,a)}-\frac{1}{\eta_t(s,a)}}\ln(T)} \\
    &\qquad + \E\brk*{\sum_{t\in\calT_r}\sum_{s,a} q^{\overP^B_t, \pi_t}(s) \eta_t(s, a)\pi_t(a\mid s)^{2}C_t(s,a)^2} \\
    &\lesssim \E \brk*{\sum_{t\in\calT_r}\sum_{s,a} q^{\overP^B_t, \pi_t}(s) \frac{\eta_t(s,a)\zeta_t(s,a)}{q_t(s)^2} } + \E \brk*{\sum_{t\in\calT_v}\sum_{s,a} q^{\overP^B_t, \pi_t}(s) \frac{\ind\{(s^\dagger_t, a^\dagger_t) = (s,a)\}}{\eta_t(s,a)H}}\\
    &\qquad + H ^{\frac{3}{2}} S ^{\frac{ 5}{2}} A ^{\frac{1}{2}} \iota ^{2} \tag{by \cref{lem:squared_C_bound}}\\
    &\leq \E \brk*{\sum_{t\in\calT_r}\sum_{s,a} \frac{\eta_t(s,a)\zeta_t(s,a)}{q_t(s)}} + \E \brk*{\sum_{t\in\calT_v} \frac{q^{\overP^B_t, \pi_t} (s_t^\dagger)}{\eta_t(s_t^\dagger,a_t^\dagger)H}} + H ^{\frac{3}{2}} S ^{\frac{ 5}{2}} A ^{\frac{1}{2}} \iota ^{2} \\
    &\lesssim \E \brk*{\sum_{t\in\calT_r}\sum_{s,a}\frac{\eta_t(s,a)\zeta_t(s,a)}{q_t(s)}} + H^{\frac32}S^{\frac52}A^{\frac32}\iota^2 \tag{by the virtual-episode condition and \cref{lem:bound_virtual_episode}}\\
    &\lesssim \sqrt{\ln(T)}\E\brk*{\sum_{t\in\calT_r}\sum_{s,a}\frac{\frac{\sqrt{\zeta_t(s,a)}}{q_t(s)}\sqrt{\zeta_t(s,a)}}{\sqrt{1+\sum_{\tau\leq t:\tau\in\calT_r}\frac{\zeta_\tau(s,a)}{q_\tau(s)^2}}}}+H^{\frac32}S^{\frac52}A^{\frac32}\iota^2\tag{by \cref{eq:bandit_learning_rate_shifted}}\\
    &\lesssim \sqrt{\ln(T)}\E\brk*{\sum_{s,a}\sqrt{\sum_{t\in\calT_r}\frac{\zeta_t(s,a)/q_t(s)^2}{1+\sum_{\tau\leq t:\tau\in\calT_r}\zeta_\tau(s,a)/q_\tau(s)^2}}\sqrt{\sum_{t\in\calT_r}\zeta_t(s,a)}}+H^{\frac32}S^{\frac{5}2}A^{\frac32}\iota^2
    \tag{by the Cauchy--Schwarz inequality}\\
    &\lesssim \ln(T)\E\brk*{\sum_{s,a}\sqrt{\sum_{t\in\calT_r}\zeta_t(s,a)}}+H^{\frac32}S^{\frac{5}2}A^{\frac32}\iota^2,
\end{align}
where the last inequality follows from
\begin{align}
 \sqrt{\sum_{t\in\calT_r}\frac{\zeta_t(s,a)/q_t(s)^2}{1+\sum_{\tau\leq t:\tau\in\calT_r}\zeta_\tau(s,a)/q_\tau(s)^2}}
 &\leq\sqrt{\ln\prn*{1+\sum_{t\in\calT_r}\frac{\zeta_t(s,a)}{q_t(s)^2}}}
 \lesssim\sqrt{\ln(T)},
\end{align}
using $q_t(s)\geq1/T$ and $\zeta_t(s,a)\leq H^2$.
    Finally, the desired bound follows from Jensen's inequality.
\end{proof}

We now summarize the bounds in the bandit setting proved above and plug them into the regret decomposition.
By the definition of the regret decomposition in \cref{eq:decompose_bandit},
\begin{align}
    &\E\brk*{\sum_s q^{\pi}(s)\sum_{t,a} \prn*{\pi_t(a\mid s) -  \pi(a\mid s) } \prn*{Q^{\pi_t}(s,a;\ell_t) -  B_t(s,a)}}\\
    &= \E\brk*{\sum_s q^{\pi}(s) \cdot \regterm(s)} + \E\brk*{\sum_s q^{\pi}(s) \cdot \biasterm(s)} + \E\brk*{\sum_s q^{\pi}(s) \cdot \errorterm(s)}\\
    &\leq \calO\prn*{H^2S^2A\ln(T)} + \mathbb{E}\brk*{\sum_{t=1}^T\sum_sq^{\pi}(s)b_t(s)} + \mathbb{E}\brk*{\frac{1}{H}\sum_{t=1}^T\sum_{s,a} q^{\pi}(s)\pi_t(a \mid s) B_t(s,a)}\\
    &\qquad  + \E\brk*{\sum_s q^{\pi}(s) \cdot \biasterm(s)} + \E\brk*{\sum_s q^{\pi}(s) \cdot \errorterm(s)}
\end{align}
where we used \cref{lem:regterm_policy}.
Combining this with \mbox{\cref{lem:dilated_bonus}} and $\delta =\frac{1}{T^2}$, for any comparator policy $\pi$, \cref{alg:OFTRL_unknown} guarantees
\begin{align} 
\Reg_T(\pi)
&\leq \calO\prn*{H^2S^2A\ln(T)} + 3\E\brk*{\sum_{t = 1}^T V^{\overP^B_t, \pi_t}(s_0;b_t)}\\
&\qquad  + \E\brk*{\sum_s q^{\pi}(s) \cdot \biasterm(s)} + \E\brk*{\sum_s q^{\pi}(s) \cdot \errorterm(s)}\\
&\lesssim H^{\frac32}S^{\frac{5}2}A^{\frac32}\iota^2 + \ln(T)\sum_{s,a}\sqrt{\E\brk*{\sum_{t\in\mathcal{T}_r} \zeta_t(s,a)}}\\
&\qquad  + \E\brk*{\sum_s q^{\pi}(s) \cdot \biasterm(s)} + \E\brk*{\sum_s q^{\pi}(s) \cdot \errorterm(s)}, \label{eq:regret_sum_bandit}
\end{align}
The contribution of $\bar{\calE}_1$ is $\calO(H^{3/2}S^{3/2}\sqrt A)$ by \cref{lem:bound_eta_pi_B} and $\delta=T^{-2}$, and is absorbed into $\calO(H^2S^2A\ln(T))$. The last inequality follows from \cref{lem:log_barrier_bt_bandit}.

\subsection{Proof for the adversarial regime}
\begin{thm}[bandit setting]
\label{thm:app_bandit_adv}
For any comparator policy $\pi$, \cref{alg:OFTRL_unknown} guarantees
\begin{align}
    \Reg_T(\pi)
    &\lesssim \sqrt{H^3S^2A\ln^2(T) \min\set*{L(\pi), HT - L(\pi), Q_{\infty}, V_1}}\\
    &\qquad+\sqrt{S^2A\ln^2(T)\E\brk*{\Qtrans^{\pi_{1:T}}(\ell)}}\\
    &\qquad + \bigl(H^3S^2A+H^{\frac32}S^3A+H^{\frac32}S^{\frac{5}2}A^{\frac32}\bigr)\ln^2(T).
\end{align}
\end{thm}
\begin{proof}
By \cref{eq:regret_sum_bandit,lem:biasterm_bandit,lem:errorterm_bandit}, for any comparator policy $\pi$, we have
\begin{align}
    \Reg_T(\pi)
    &\lesssim \bigl(H^{\frac32}S^3A+H^{\frac32}S^{\frac{5}2}A^{\frac32}\bigr)\iota^2+\ln(T)\sum_{s,a}\sqrt{\E\brk*{\sum_{t \in \calT_r}\zeta_t(s,a)}}\\
    &\qquad+\sqrt{H^2S^2A\ln^2(T)\E\brk*{\sum_{t \in \calT_r}\sum_{s,a}q^{\pi_t}(s,a)\max_{\tilP\in\calP_t}Q^{\tilP,\pi_t}(s,a;\prn*{\ell_t-m_t}^2)}}\\
    &\qquad+\sqrt{S^2A\iota^2\E\brk*{\Qtrans^{\pi_{1:T}}(m)}}.
\end{align}
Here, since $\delta=\calO(1/T^2)$, we have $\iota\lesssim\ln(T)$. Hence,
\begin{align}
    \Reg_T(\pi)
    &\lesssim \bigl(H^{\frac32}S^3A+H^{\frac32}S^{\frac{5}2}A^{\frac32}\bigr)\ln^2(T)+\underbrace{\ln(T)\sum_{s,a}\sqrt{\E\brk*{\sum_{t \in \calT_r}\zeta_t(s,a)}}}_{\term_1}\\
    &\qquad+\underbrace{\sqrt{H^3S^2A\ln^2(T)\E\brk*{\sum_{t \in \calT_r}V^{\pi_t}(s_0;\prn*{\ell_t-m_t}^2)}}}_{\term_2}\\
    &\qquad+\underbrace{\sqrt{S^2A\ln^2(T)\E\brk*{\Qtrans^{\pi_{1:T}}(m)}}}_{\term_3}, \label{eq:bandit_regret_decomp_adv}
\end{align}
where the last inequality uses \cref{lem:max_to_value} on $\calE$ together with the trivial bound $\calO(H^2T)$ on $\bar{\calE}$, whose contribution is absorbed since $\delta=\calO(1/T^2)$.

For $\term_1$, by the definition of $\zeta_t(s,a)$, we have
\begin{align}   
        \term_1
        &= \sum_{s,a}\sqrt{\ln^2(T)\E\brk*{\sum_{t \in \calT_r} (\I_t(s,a) - \pi_t(a\mid s) \I_t(s))^2(L_{t,h(s)} - M_{t,h(s)})^2}}\\
        &\leq \sqrt{SA\ln^2(T)\E\brk*{\sum_{t \in \calT_r}\sum_{s,a} (\I_t(s,a) - \pi_t(a\mid s) \I_t(s))^2(L_{t,h(s)} - M_{t,h(s)})^2}} \tag{by the Cauchy–Schwarz inequality} \\
        &\leq \sqrt{SA\ln^2(T)\E\brk*{\sum_{t \in \calT_r}\sum_{s,a} 2\I_t(s,a)(L_{t,h(s)} - M_{t,h(s)})^2}} \\
        &\lesssim \sqrt{H^2SA\ln^2(T)\E\brk*{\sum_{t \in \calT_r}\sum_{s,a}\I_t(s,a)(\ell_t(s,a) - m_t(s,a))^2}}, \label{eq:common_bound_adv}
\end{align}
where the third line uses 
$\sum_a(\I_t(s,a)-\pi_t(a\mid s)\I_t(s))^2\le 2\I_t(s)$
for each fixed state $s$, and the last inequality follows from \cref{lem:L-M_to_l-m}.

Applying \cref{lem:general_predict_result} to \cref{eq:common_bound_adv}, we obtain, for any policy $\pi$,
\begin{align}
    \term_1
    &\lesssim \sqrt{H^2SA\ln^2(T)\min\set*{L(\pi)  + \Reg_T(\pi), HT - L(\pi)  - \Reg_T(\pi), Q_{\infty}, V_1}}\\
    &\qquad + HSA\ln(T).\label{eq:adv_bandit_term1}
\end{align}

For $\term_2$, \cref{lem:general_predict_result} similarly yields, for any policy $\pi$,
\begin{align}
    \term_2
    &\lesssim \sqrt{H^3S^2A\ln^2(T)\min\set*{L(\pi)  + \Reg_T(\pi), HT - L(\pi)  - \Reg_T(\pi), Q_{\infty}, V_1}}\\
    &\qquad + \sqrt{H^3S^3}A\ln(T).\label{eq:adv_bandit_term2}
\end{align}
For $\term_3$, \cref{lem:qtrans_m_to_ell} applied to the real episodes yields
\begin{align}
    \term_3
    &\lesssim \sqrt{S^2A\ln^2(T)\E\brk*{\Qtrans^{\pi_{1:T}}(\ell)}}+\sqrt{HS^2A\ln^2(T)\E\brk*{\sum_{t \in \calT_r}V^{\pi_t}(s_0;\prn*{\ell_t-m_t}^2)}}.
\end{align}
Applying \cref{lem:general_predict_result} to the second term, we obtain, for any policy $\pi$,
\begin{align}
    \term_3
    &\lesssim \sqrt{S^2A\ln^2(T)\E\brk*{\Qtrans^{\pi_{1:T}}(\ell)}}\\
    &\qquad +\sqrt{HS^2A\ln^2(T)\min\set*{L(\pi)+\Reg_T(\pi),HT-L(\pi)-\Reg_T(\pi),Q_\infty,V_1}}\\
    &\qquad+\sqrt{HS^3}A\ln(T). \label{eq:adv_bandit_term3}
\end{align}
Combining \cref{eq:adv_bandit_term1,eq:adv_bandit_term2,eq:adv_bandit_term3} with \cref{eq:bandit_regret_decomp_adv}, we obtain
\begin{align}
    \Reg_T(\pi)
    &\lesssim \sqrt{H^3S^2A\ln^2(T)\min\set*{L(\pi)+\Reg_T(\pi),HT-L(\pi)-\Reg_T(\pi),Q_\infty,V_1}}\\
    &\qquad+\sqrt{S^2A\ln^2(T)\E\brk*{\Qtrans^{\pi_{1:T}}(\ell)}}+\bigl(H^{\frac32}S^3A+H^{\frac32}S^{\frac{5}2}A^{\frac32}\bigr)\ln^2(T). \label{eq:adv_bandit_mid}
\end{align}
Below, we derive the first-order bound. By \cref{eq:adv_bandit_mid}, there exists an absolute constant $c>0$ such that
\begin{align}
    \Reg_T(\pi)
    &\leq c\sqrt{H^3S^2A  \ln^2(T) L(\pi)}+c\sqrt{H^3S^2A  \ln^2(T) \Reg_T(\pi)}\\
    &\qquad + c\sqrt{S^2A\ln^2(T) \E\brk*{\Qtrans^{\pi_{1:T}}(\ell)}} + c\bigl(H^{\frac32}S^3A+H^{\frac32}S^{\frac{5}2}A^{\frac32}\bigr)\ln^2(T)\\
    &\leq c\sqrt{H^3S^2A  \ln^2(T)  L(\pi)}+\frac{c^2}{2}H^3S^2A  \ln^2(T) + \frac{1}{2} \Reg_T(\pi)\\
    &\qquad + c\sqrt{S^2A\ln^2(T) \E\brk*{\Qtrans^{\pi_{1:T}}(\ell)}} + c\bigl(H^{\frac32}S^3A+H^{\frac32}S^{\frac{5}2}A^{\frac32}\bigr)\ln^2(T),
\end{align}
where the last inequality follows from the AM--GM inequality.
Thus, we obtain
\begin{align}
    \Reg_T(\pi) 
    &\lesssim \sqrt{H^3S^2A  \ln^2(T)  L(\pi)} + \sqrt{S^2A\ln^2(T) \E\brk*{\Qtrans^{\pi_{1:T}}(\ell)}}\\
    &\qquad + \bigl(H^3S^2A+H^{\frac32}S^3A+H^{\frac32}S^{\frac{5}2}A^{\frac32}\bigr)\ln^2(T). \label{eq:first_order_bandit1}
\end{align}
Furthermore, we derive the bound in terms of $HT-L(\pi)$. \cref{eq:adv_bandit_mid} gives
\begin{align}
    \Reg_T(\pi)
    &\lesssim
    \sqrt{H^3S^2A\ln^2(T)\prn*{HT-L(\pi)-\Reg_T(\pi)}}
    + \sqrt{S^2A\ln^2(T)\E\brk*{\Qtrans^{\pi_{1:T}}(\ell)}}\\
    &\qquad + \bigl(H^{\frac32}S^3A+H^{\frac32}S^{\frac{5}2}A^{\frac32}\bigr)\ln^2(T)\\
    &\leq
    \sqrt{H^3S^2A\ln^2(T)\prn*{HT-L(\pi)}} 
    + \sqrt{S^2A\ln^2(T)\E\brk*{\Qtrans^{\pi_{1:T}}(\ell)}}\\
    &\qquad + \bigl(H^{\frac32}S^3A+H^{\frac32}S^{\frac{5}2}A^{\frac32}\bigr)\ln^2(T).
    \label{eq:first_order_bandit2}
\end{align}

Combining \cref{eq:adv_bandit_mid,eq:first_order_bandit1,eq:first_order_bandit2} completes the proof.
\end{proof}

\subsection{Proof for the stochastic regime with adversarial corruption}
\begin{thm}[bandit setting]
\label{thm:app_bandit_corruption}
Under the stochastic regime with adversarial corruption, \cref{alg:OFTRL_unknown} with $\delta=\frac{1}{T^2}$ guarantees that, for any policy $\pi$, 
\begin{align}
    \Reg_T(\pi) \lesssim U + \sqrt{U\calC} + \bigl(H^{\frac32}S^3A+H^{\frac32}S^{\frac{5}2}A^{\frac32}\bigr)\ln^2(T),
\end{align}
where $U = \frac{H^{4}S^{2}A\ln^2(T)}{\Delta_{\min}}$.    
\end{thm}
\begin{proof}
Since $\Reg_T(\pi) \leq \Reg_T(\pio)$ for all $\pi$, it suffices to bound $\Reg_T(\pio)$. 
By \cref{eq:regret_sum_bandit,lem:biasterm_bandit}, \cref{eq:bound_error_2} and $\delta=\calO(1/T^2)$, we have $\iota\lesssim\ln(T)$. Hence, there exists an absolute constant $c>0$ such that
\begin{align}
    \Reg_T(\pio)
    &\leq c\bigl(H^{\frac32}S^3A+H^{\frac32}S^{\frac{5}2}A^{\frac32}\bigr)\ln^2(T)+\underbrace{c\ln(T)\sum_{s,a}\sqrt{\E\brk*{\sum_{t \in \calT_r}\zeta_t(s,a)}}}_{\term_1}\\
    &\qquad+\underbrace{c\sqrt{H^3S^2A\ln^2(T)\E\brk*{\sum_{t \in \calT_r}\sum_{s,a}\brk*{q^{\pi_t}(s,a)-q^{\pio}(s,a)}_+}}}_{\term_2}\\
    &\qquad+\underbrace{c\sum_{h=0}^{H-1}\sqrt{H^2\abs*{\calS_h}\abs*{\calS_{h+1}}A\ln^2(T)\E\brk*{\sum_{t \in \calT_r}\sum_{(s,a)\in\calS_h\times\calA}\brk*{q^{\pi_t}(s,a)-q^{\pio}(s,a)}_+}}}_{\term_3}. \label{eq:bandit_regret_decomp_sto}
\end{align}
For $\term_1$, by the definition of $\zeta_t(s,a)$, we have for any $\alpha>0$,
\begin{align}
    \term_1
    &= c\sum_{s,a}\sqrt{\ln^2(T)\E\brk*{\sum_{t \in \calT_r}(\I_t(s,a)-\pi_t(a\mid s)\I_t(s))^2(L_{t,h(s)}-M_{t,h(s)})^2}}\\
    &\leq cH\sum_{s,a}\sqrt{\ln^2(T)\E\brk*{\sum_{t \in \calT_r}q^{\pi_t}(s,a)(1-\pi_t(a\mid s))}}\\
    &\leq 2cH\sum_s\sum_{a\neq\pist(s)}\sqrt{\ln^2(T)\E\brk*{\sum_{t \in \calT_r}q^{\pi_t}(s,a)}}\\
    &\leq \alpha\prn*{\Reg_T(\pio)+2\calC}+\sum_s\sum_{a\neq\pist(s)}\frac{c^2H^2\ln^2(T)}{\alpha\Delta(s,a)}, \label{eq:sto_bandit_term1}
\end{align}
where the last line follows from \cref{lem:general_self_bounding}.

For $\term_2$, \cref{lem:self_bound_occ_positive} gives
\begin{align}
    \term_2
    &\leq \alpha\prn*{\Reg_T(\pio)+4\calC}+\frac{c^2H^4S^2A\ln^2(T)}{4\alpha\Delta_{\min}}. \label{eq:sto_bandit_term2}
\end{align}

For $\term_3$, \cref{lem:self_bound_layer_positive} gives
\begin{align}
    \term_3
    &\leq \alpha\prn*{\Reg_T(\pio)+4\calC}+\frac{c^2\prn*{\sum_{h=0}^{H-1}\sqrt{H^2\abs*{\calS_h}\abs*{\calS_{h+1}}A\ln^2(T)}}^2}{4\alpha\Delta_{\min}}.
\end{align}
Moreover, by the AM--GM inequality,
\begin{align}
    \prn*{\sum_{h=0}^{H-1}\sqrt{H^2\abs*{\calS_h}\abs*{\calS_{h+1}}A\ln^2(T)}}^2
    &\leq \prn*{\sum_{h=0}^{H-1}\frac{1}{2}\prn*{\abs*{\calS_h}+\abs*{\calS_{h+1}}}\sqrt{H^2A\ln^2(T)}}^2\\
    &\leq \prn*{\sqrt{H^2S^{2}A\ln^2(T)}}^2 = H^2S^{2}A\ln^2(T).
\end{align}
Thus, we obtain 
\begin{align}
    \term_3
    &\leq \alpha\prn*{\Reg_T(\pio)+4\calC}+\frac{c^2H^2S^{2}A\ln^2(T)}{4\alpha\Delta_{\min}}. \label{eq:sto_bandit_term3}
\end{align}
Combining \cref{eq:sto_bandit_term1,eq:sto_bandit_term2,eq:sto_bandit_term3} with \cref{eq:bandit_regret_decomp_sto}, we obtain
\begin{align}
    \Reg_T(\pio)
    &\leq \alpha(3\Reg_T(\pio) + 10\calC) + \calO\prn*{\frac{H^{4}S^{2}A\ln^2(T)}{4\alpha\Delta_{\min}} + \bigl(H^{\frac32}S^3A+H^{\frac32}S^{\frac{5}2}A^{\frac32}\bigr)\ln^2(T)}.
\end{align}
Finally, choosing $\alpha = \min\set*{\frac{1}{6}, \sqrt{\frac{U}{\calC}}}$ with $U = \frac{H^{4}S^{2}A\ln^2(T)}{\Delta_{\min}}$, we get
\begin{align}
    \Reg_T(\pio) \lesssim U + \sqrt{U\calC} + \bigl(H^{\frac32}S^3A+H^{\frac32}S^{\frac{5}2}A^{\frac32}\bigr)\ln^2(T).
\end{align}
Therefore, since $\Reg_T(\pi)\leq \Reg_T(\pio)$ for all $\pi$ by the definition of $\pio$, the desired bound follows.
\end{proof}

\newpage
\section*{NeurIPS Paper Checklist}

\begin{enumerate}

\item {\bf Claims}
    \item[] Question: Do the main claims made in the abstract and introduction accurately reflect the paper's contributions and scope?
    \item[] Answer: \answerYes{} %
    \item[] Justification: The abstract and introduction state the gap in prior work addressed by this paper and summarize the main regret guarantees.

\item {\bf Limitations}
    \item[] Question: Does the paper discuss the limitations of the work performed by the authors?
    \item[] Answer: \answerYes{} %
    \item[] Justification: The limitations and future directions are discussed in \cref{sec:main_result,sec:conclusion}, including lower-bound gaps, variance-aware bounds, and aggregate-feedback extensions.

\item {\bf Theory assumptions and proofs}
    \item[] Question: For each theoretical result, does the paper provide the full set of assumptions and a complete (and correct) proof?
    \item[] Answer: \answerYes{} %
    \item[] Justification: The assumptions and theoretical results are stated in \cref{sec:pre,sec:main_result}, with a proof sketch in \cref{sec:proof}. Complete proofs are provided in the appendix.

    \item {\bf Experimental result reproducibility}
    \item[] Question: Does the paper fully disclose all the information needed to reproduce the main experimental results of the paper to the extent that it affects the main claims and/or conclusions of the paper (regardless of whether the code and data are provided or not)?
    \item[] Answer: \answerNA{} %
    \item[] Justification: This paper does not include experiments.

\item {\bf Open access to data and code}
    \item[] Question: Does the paper provide open access to the data and code, with sufficient instructions to faithfully reproduce the main experimental results, as described in supplemental material?
    \item[] Answer: \answerNA{} %
    \item[] Justification: This paper does not include experiments.

\item {\bf Experimental setting/details}
    \item[] Question: Does the paper specify all the training and test details (e.g., data splits, hyperparameters, how they were chosen, type of optimizer) necessary to understand the results?
    \item[] Answer: \answerNA{} %
    \item[] Justification: This paper does not include experiments.

\item {\bf Experiment statistical significance}
    \item[] Question: Does the paper report error bars suitably and correctly defined or other appropriate information about the statistical significance of the experiments?
    \item[] Answer: \answerNA{} %
    \item[] Justification: This paper does not include experiments.

\item {\bf Experiments compute resources}
    \item[] Question: For each experiment, does the paper provide sufficient information on the computer resources (type of compute workers, memory, time of execution) needed to reproduce the experiments?
    \item[] Answer: \answerNA{} %
    \item[] Justification: This paper does not include experiments.
    
\item {\bf Code of ethics}
    \item[] Question: Does the research conducted in the paper conform, in every respect, with the NeurIPS Code of Ethics \url{https://neurips.cc/public/EthicsGuidelines}?
    \item[] Answer: \answerYes{}{} %
    \item[] Justification: We have reviewed the NeurIPS Code of Ethics. The research is theoretical and uses no human subjects or sensitive data.

\item {\bf Broader impacts}
    \item[] Question: Does the paper discuss both potential positive societal impacts and negative societal impacts of the work performed?
    \item[] Answer: \answerNA{} %
    \item[] Justification: This is a theoretical paper and does not have direct societal impacts.
    
\item {\bf Safeguards}
    \item[] Question: Does the paper describe safeguards that have been put in place for responsible release of data or models that have a high risk for misuse (e.g., pre-trained language models, image generators, or scraped datasets)?
    \item[] Answer: \answerNA{} %
    \item[] Justification: The paper poses no such risks.

\item {\bf Licenses for existing assets}
    \item[] Question: Are the creators or original owners of assets (e.g., code, data, models), used in the paper, properly credited and are the license and terms of use explicitly mentioned and properly respected?
    \item[] Answer: \answerNA{} %
    \item[] Justification: This paper does not use existing assets.

\item {\bf New assets}
    \item[] Question: Are new assets introduced in the paper well documented and is the documentation provided alongside the assets?
    \item[] Answer: \answerNA{} %
    \item[] Justification: This paper does not release new research assets.

\item {\bf Crowdsourcing and research with human subjects}
    \item[] Question: For crowdsourcing experiments and research with human subjects, does the paper include the full text of instructions given to participants and screenshots, if applicable, as well as details about compensation (if any)? 
    \item[] Answer: \answerNA{} %
    \item[] Justification: This paper does not involve crowdsourcing nor research with human subjects.

\item {\bf Institutional review board (IRB) approvals or equivalent for research with human subjects}
    \item[] Question: Does the paper describe potential risks incurred by study participants, whether such risks were disclosed to the subjects, and whether Institutional Review Board (IRB) approvals (or an equivalent approval/review based on the requirements of your country or institution) were obtained?
    \item[] Answer: \answerNA{} %
    \item[] Justification: This paper does not involve crowdsourcing nor research with human subjects.

\item {\bf Declaration of LLM usage}
    \item[] Question: Does the paper describe the usage of LLMs if it is an important, original, or non-standard component of the core methods in this research? Note that if the LLM is used only for writing, editing, or formatting purposes and does \emph{not} impact the core methodology, scientific rigor, or originality of the research, declaration is not required.
    \item[] Answer: \answerNA{} %
    \item[] Justification: The core method and theoretical results do not involve LLMs.

\end{enumerate}

\end{document}